\documentclass{article} 

\usepackage{microtype}
\usepackage{hyperref}
\usepackage{url}
\usepackage{booktabs}
\usepackage{longtable}
\usepackage{float}
\usepackage{inconsolata}
\usepackage{pgfplots}
\usepgfplotslibrary{groupplots}
\pgfplotsset{compat=1.18}

\newcommand{\mymacro}[1]{{#1}}

\newcommand{\defn}[1]{\textbf{#1}}

\newcommand{\paroutline}[3][false]{%
    \ifnum\pdfstrcmp{#1}{true}=0
        #3
    \else
        [\textit{\textcolor{DiverseMagenta}{#2}}] \textcolor{AccentBlue}{#3}
    \fi
}

\newcommand{\rope}{\mymacro{\text{\textnormal{{RoPE}}}}}
\newcommand{\nope}{\mymacro{\text{\textnormal{{NoPE}}}}}
\newcommand{\sinpe}{\mymacro{\text{\textnormal{{SiPE}}}}}

\newcommand{\ropeperiodic}{\mymacro{\text{\textnormal{{RoPE\textsubscript{P}}}}}}
\newcommand{\ropenonperiodic}
{\mymacro{\text{\textnormal{{RoPE\textsubscript{NP}}}}}}

\newcommand{\sinperiodic}{\mymacro{\text{\textnormal{{SiPE\textsubscript{P}}}}}}
\newcommand{\sinnonperiodic}
{\mymacro{\text{\textnormal{{SiPE\textsubscript{NP}}}}}}

\newcommand{\rotAngles}{{\mymacro{\Theta}}}

\newcommand{\valpha}{{\mymacro{\boldsymbol{\alpha}}}}

\newcommand{\biasVector}{{\mymacro{ \vb}}}

\newcommand{\transclosure}{{\mymacro{ \delta^*}}}

\newcommand{\Q}{{\mymacro{ \mathbb{Q}}}}

\newcommand{\Z}{{\mymacro{ \mathbb{Z}}}}

\newcommand{\boundyesterday}[1][]{{\mymacro{N^{#1}_{\textnormal{max}}}}}

\newcommand{\offsetset}{{\mymacro{K}}}

\newcommand{\ropeconstant}{{\mymacro{C}}}

\newcommand{\alphabet}{{\mymacro{ \Sigma}}}

\newcommand{\eosalphabet}{{\mymacro{ \overline{\alphabet}}}}

\newcommand{\kleene}[1]{{\mymacro{#1^*}}}
\newcommand{\reComplement}[1]{{\mymacro{#1^\textsc{c}}}}

\newcommand{\str}{{\mymacro{\boldsymbol{w}}}}

\newcommand{\strlen}{{\mymacro{N}}}
\newcommand{\strs}{{\mymacro{\boldsymbol{s}}}}
\newcommand{\strx}{{\mymacro{\boldsymbol{x}}}}
\newcommand{\stry}{{\mymacro{\boldsymbol{y}}}}
\newcommand{\strz}{{\mymacro{\boldsymbol{z}}}}

\newcommand{\syma}{{\mymacro{a}}}
\newcommand{\symb}{{\mymacro{b}}}
\newcommand{\symc}{{\mymacro{c}}}

\newcommand{\defeq}{\mathrel{\stackrel{\textnormal{\tiny def}}{=}}}

\newcommand{\NTo}[1]{{\mymacro{\left[ #1 \right]}}}

\newcommand{\set}[1]{{\mymacro{\{ #1 \}}}}

\newcommand{\idxn}{{\mymacro{ n}}}

\newcommand{\idxd}{{\mymacro{ d}}}
\newcommand{\idxi}{{\mymacro{ i}}}
\newcommand{\idxj}{{\mymacro{ j}}}
\newcommand{\idxk}{{\mymacro{ k}}}
\newcommand{\idxl}{{\mymacro{ l}}}

\newcommand{\bos}{{\mymacro{\textsc{bos}}}}
\newcommand{\eos}{{\mymacro{\textsc{eos}}}}

\newcommand{\automaton}{{\mymacro{ \mathcal{A}}}}
\newcommand{\wfsa}{{\mymacro{ \automaton}}}

\newcommand{\stateq}{{\mymacro{ q}}}
\newcommand{\statep}{{\mymacro{ p}}}

\newcommand{\states}{{\mymacro{ Q}}}

\newcommand{\trans}{{\mymacro{ \delta}}}

\newcommand{\final}{{\mymacro{ F}}}

\newcommand{\qinit}{{\mymacro{ q_{\iota}}}}

\newcommand{\fsatuple}{{\mymacro{ \mleft( \alphabet, \states, \qinit, \final, \trans \mright)}}}
\newcommand{\satuple}{{\mymacro{ \mleft( \alphabet, \states, \trans \mright)}}}

\newcommand{\dfatuple}{{\mymacro{ \mleft( \alphabet, \states, \qinit, \final, \trans \mright)}}}

\newcommand{\hiddDim}{{\mymacro{ D}}}

\newcommand{\proj}{{\mymacro{\texttt{proj}}}}
\newcommand{\projfunc}{{\mymacro{\textnormal{\proj}}}}

\newcommand{\negterm}[1]{{\mymacro{ {\raise.17ex\hbox{$\scriptstyle\sim$}} #1}}}

\newcommand{\ignore}[1]{}
\newcommand{\expandLater}[1]{}

\newcommand{\tf}{\mymacro{\relsize{-0.25}{\textsf{T}}}\xspace}

\newcommand{\outputlayer}{\mymacro{\mathrm{\mathcal{F}}}}

\newcommand{\layerIdx}{\mymacro{\ell}}
\newcommand{\nLayers}{\mymacro{L}}

\def\1{\mathbf{1}}

\def\eps{{\mymacro{ \varepsilon}}}

\def\rmH{{{\mymacro{ \mathbf{H}}}}}

\def\rmK{{{\mymacro{ \mathbf{K}}}}}

\def\rmQ{{{\mymacro{ \mathbf{Q}}}}}

\def\rmV{{{\mymacro{ \mathbf{V}}}}}

\def\vb{{{\mymacro{ \bm{b}}}}}

\def\vk{{{\mymacro{ \bm{k}}}}}

\def\vq{{{\mymacro{ \bm{q}}}}}

\def\vv{{{\mymacro{ \bm{v}}}}}

\def\vx{{{\mymacro{ \bm{x}}}}}
\def\valpha{{{\mymacro{ \bm{\alpha}}}}}

\def\mK{{{\mymacro{ \bm{K}}}}}

\def\mQ{{{\mymacro{ \bm{Q}}}}}
\def\mR{{{\mymacro{ \bm{R}}}}}
\def\mS{{{\mymacro{ \bm{S}}}}}

\newcommand{\N}{{\mymacro{ \mathbb{N}}}}

\newcommand{\bb}[1][]{\ifthenelse{\isempty{#1}}{\mymacro{\mathbf{b}}}{\mymacro{\mathbf{b}^{\text{#1}}}}}
\newcommand{\ff}[1][]{\ifthenelse{\isempty{#1}}{\mymacro{f}}{\mymacro{f_{\text{#1}}}}}

\newcommand{\W}[1][]{\ifthenelse{\isempty{#1}}{\mymacro{\mathbf{W}}}{\mymacro{\mathbf{W}^{\text{#1}}}}}

\newcommand{\stacktop}[1][]{\ifthenelse{\isempty{#1}}{\mymacro{\gamma^{\text{top}}}}{\mymacro{\gamma^{\text{top}}_{#1}}}}
\newcommand{\sym}{\mymacro{w}}

\newcommand{\lang}{\mymacro{\mathbb{L}}}

\newcommand{\until}{\mymacro{\mathrel{\textbf{U}}}}

\newcommand{\since}{\mymacro{\mathrel{\textbf{S}}}}
\newcommand{\past}{\mymacro{\mathrel{\textnormal{\textbf{P}}}}}
\newcommand{\future}{\mymacro{\mathrel{\textnormal{\textbf{F}}}}}
\newcommand{\yesterday}{\mymacro{\mathrel{\textbf{Y}}}}

\newcommand{\true}{\mymacro{\top}}

\newcommand{\false}{\mymacro{\bot}}

\newcommand{\ltlAcr}{\mymacro{\textnormal{\textbf{LTL}}}}

\NewDocumentCommand{\stl}{o}{\IfNoValueTF{#1}{\mymacro{\ltlAcr[\since]}}{\mymacro{\ltlAcr[\since,#1]}}}

\NewDocumentCommand{\transformer}{O{} o o o }{
  \IfBlankTF{#1}{
    \IfNoValueTF{#2}{
      \IfNoValueTF{#4}{\rmH}{\rmH(#4)}
    }{
      \IfNoValueTF{#4}{\mymacro{\rmH_{#2,#3}}}{\mymacro{\rmH(#4)_{#2,#3}}}
    }
  }{
    \IfNoValueTF{#2}{
      \IfNoValueTF{#4}{\mymacro{\rmH^{(#1)}}}{\mymacro{\rmH^{(#1)}(#4)}}
    }{
      \IfNoValueTF{#4}{\mymacro{\rmH^{(#1)}_{#2,#3}}}{\mymacro{\rmH^{(#1)}(#4)_{#2,#3}}}
    }
  }
}
\NewDocumentCommand{\attention}{o}{\IfNoValueTF{#1}{\mymacro{\mathbf{A}}}{\mymacro{\mathbf{A}^{(#1)}}}}
\NewDocumentCommand{\ffn}{o}{\IfNoValueTF{#1}{\mymacro{\mathbf{F}}}{\mymacro{\mathbf{F}^{(#1)}}}}

\NewDocumentCommand{\fo}{o}{\IfNoValueTF{#1}{\mymacro{\text{\textnormal{\textbf{FO}}}[\mathord<]}}{\mymacro{\text{\textnormal{\textbf{FO}}}[\mathord<,#1]}}}

\NewDocumentCommand{\fotwo}{o}{\IfNoValueTF{#1}{\mymacro{\text{\textnormal{\textbf{FO}}}^2[\mathord<]}}{\mymacro{\text{\textnormal{\textbf{FO}}}^2[\mathord<,#1]}}}

\NewDocumentCommand{\pfo}{o}{\IfNoValueTF{#1}{\mymacro{\text{\textnormal{\textbf{PFO}}}^2[\mathord<]}}{\mymacro{\text{\textnormal{\textbf{PFO}}}^2[\mathord<,#1]}}}

\NewDocumentCommand{\ptl}{o}{\IfNoValueTF{#1}{\mymacro{\ltlAcr[\past]}}{\mymacro{\ltlAcr[\past,#1]}}}

\NewDocumentCommand{\pftl}{o}{\IfNoValueTF{#1}{\mymacro{\ltlAcr[\past, \future]}}{\mymacro{\ltlAcr[\past, \future, #1]}}}

\NewDocumentCommand{\ytl}{o}{\IfNoValueTF{#1}{\mymacro{\ltlAcr[\past]}}{\mymacro{\ltlAcr[\yesterday,#1]}}}

\newcommand{\operators}{\mymacro{\mathcal{O}}}

\NewDocumentCommand{\tl}{o}{\IfNoValueTF{#1}{\mymacro{\ltlAcr[\past,\future]}}{\mymacro{\ltlAcr[\past,\future, #1]}}}

\NewDocumentCommand{\ffo}{o}{\IfNoValueTF{#1}{\mymacro{\textbf{FFO}^2[\mathord<]}}{\mymacro{\textbf{FFO}^2[\mathord<,#1]}}}

\NewDocumentCommand{\ahat}{o}{%
  \IfNoValueTF{#1}
    {\mymacro{\text{\textnormal{\textbf{AHAT}}}}\xspace}
    {\mymacro{\text{\textnormal{\textbf{AHAT[#1]}}}}\xspace}%
}

\NewDocumentCommand{\smat}{o}{%
  \IfNoValueTF{#1}
    {\mymacro{\text{\textnormal{\textbf{SMAT}}}}\xspace}
    {\mymacro{\text{\textnormal{\textbf{SMAT[#1]}}}}\xspace}%
}

\newcommand{\modpred}{\mymacro{\text{\textnormal{\texttt{MOD}}}}}

\newcommand{\tlf}{\mymacro{\psi}}

\newcommand{\atom}{\mymacro{\pi}}

\NewDocumentCommand{\query}{o o }{\IfNoValueTF{#1}{\rmQ}{\rmQ_{#1,#2}}}
\NewDocumentCommand{\key}{o o }{\IfNoValueTF{#1}{\rmK}{\rmK_{#1,#2}}}
\NewDocumentCommand{\val}{o o }{\IfNoValueTF{#1}{\rmV}{\rmV_{#1,#2}}}

\newcommand{\tffunc}{\mymacro{\boldsymbol{\lambda}}}

\newcommand{\monoid}{\mymacro{\mathbb{M}}}
\newcommand{\relation}{\mymacro{\preceq}}

\newcommand{\podfaAcr}{\mymacro{\textnormal{PODFA}}}

\newcommand{\QRMonoid}{\mymacro{{\text{\textnormal{\textbf{QR}}}}}}

\newcommand{\LMonoid}{\mymacro{{\text{\textnormal{\textbf{L}}}}}}
\newcommand{\RMonoid}{\mymacro{{\text{\textnormal{\textbf{R}}}}}}

\newcommand{\QDAMonoid}{\mymacro{{\text{\textnormal{\textbf{QDA}}}}}}

\newcommand{\aformula}{\mymacro{\phi}}

\newcommand{\rotationmat}[2]{\mymacro{\mathbf{Rot}}_{#1, #2}}
\newcommand{\transpose}[1]{{\mymacro{#1^{\intercal}}}}

\newcommand{\floatvalue}{\mymacro{f}}
\newcommand{\floatlarge}{\mymacro{f_\text{large}}}

\newcommand{\finiteset}{{\mymacro{\mathbb{F}}}}

\newcommand{\congruenceclass}{{\mymacro{\mathcal{C}}}}
\newcommand{\WFSet}{{\mymacro{K}}}
\newcommand{\monomial}{\mymacro{M}}

\usepackage[final]{colm2026_conference}
\usepackage{colm_custom}

\usepackage{lineno}

\definecolor{darkblue}{rgb}{0, 0, 0.5}
\definecolor{fuchsia}{rgb}{1, 0, 1}
\hypersetup{colorlinks=true, citecolor=darkblue, linkcolor=darkblue, urlcolor=darkblue}

\title{Disentangling the Expressivity of $\rope$}

\author{Selim Jerad$^{1}$ \quad Anej Svete$^{2}$ \quad Jiaoda Li$^{2}$ \quad Ryan Cotterell$^{2}$ \\
Toyota Technological Institute at Chicago$^{1}$ \\
ETH Zürich$^{2}$ \\
\texttt{\href{mailto:sjerad@ttic.edu}{sjerad@ttic.edu}} \\
\texttt{\{\href{mailto:anej.svete@inf.ethz.ch}{anej.svete},
\href{mailto:jiaoda.li@inf.ethz.ch}{jiaoda.li}, \href{mailto:ryan.cotterell@inf.ethz.ch}{ryan.cotterell}\}@ethz.ch}}

\begin{document}

\ifcolmsubmission
\linenumbers
\fi

\maketitle

\begin{abstract}
    Two accounts recur in explanations of the success of rotary position embeddings ($\rope$). 
    Expressivity studies associate periodic position information with \emph{modular predicates}, whereas mechanistic and long-context studies emphasize positional \emph{anchors} and \emph{local offsets}. 
    We formalize both accounts for fully uniform, finite-precision soft-attention transformers. 
    We find that, if every rotary component is \emph{periodic}, $\rope$ transformers recognize exactly the languages definable in past temporal logic with modular predicates. 
    Conventional $\rope$ is \emph{different}: The rotations it computes never repeat. 
    This yields a precision-dependent bounded simulation of fixed-offset \emph{look-back} operators, rather than an all-length modular characterization.
    Controlled experiments match this separation: Constructed periodic schedules length-generalize on modular languages, while conventional $\rope$ behaves more like a bounded locality bias and can impair tasks requiring position-invariant access to distant context.
    Altogether, our findings shed light on \rope{} transformers, bringing theoretical expressivity characterizations closer to models used in practice.
\end{abstract}

\section{Introduction}
In the absence of positional information, self-attention---the core operation of transformers \citep{vaswani2023attentionneed}---is permutation-equivariant. 
Positional encodings (PEs) help distinguish symbol positions. 
The original transformer employs sinusoidal PEs ($\sinpe$), which add position-dependent sine and cosine features to symbol representations. 
However, models using $\sinpe$ often generalize poorly to sequences longer than those seen during training \citep{rosendahl-etal-2019-analysis,neishi-yoshinaga-2019-relation,press2022train}. 
Motivated by this limitation, $\rope$ was proposed as an alternative that encodes relative positions within attention: It rotates query and key vectors as a function of position, making their dot product depend on relative displacement \citep{su2023roformerenhancedtransformerrotary}. 
$\rope$ has since become the de facto PE scheme in transformer-based LMs \citep[][\textit{inter alia}]{bai2023qwentechnicalreport,grattafiori2024llama3herdmodels,gemmateam2024gemmaopenmodelsbased,olmo2025olmo3,yang2025qwen3technicalreport,deepseekai2026deepseekv4highlyefficientmilliontoken}.

Despite its widespread use, the theoretical contributions of $\rope$ remain unclear. 
Existing explanations fall broadly into two themes.
Theoretical expressivity work links the periodicity of sinusoidal functions to modular predicates in formal logic \citep{yang2024maskedhardattentiontransformersrecognize, yang2025kneedeepcrasptransformerdepth}.
More practically oriented analyses find that the apparent length generalization ability of \rope{} can be superficial: $\rope$ may extrapolate to longer sequences without enabling effective use of distant context \citep{kazemnejad2023impactpositionalencodinglength,men2024baseropeboundscontext,gelberg2025extendingcontextpretrainedllms,du2026ropedistinguishespositionstokens}.
A complementary, more mechanistic line of work instead emphasizes its role as a positional anchor and its bias toward fixed local offsets \citep{barbero2025roundroundgomakes,barbero2025llmsattendtoken}.

These perspectives \textit{prima facie} appear incompatible. 
We reconcile them within the fully uniform fixed-precision transformer framework. 
We show that component-periodic $\rope$ ($\ropeperiodic$) is equivalent to past temporal logic with modular predicates, $\ptl[\modpred]$. 
Conventional $\rope$, in contrast, realizes a non-periodic schedule. The two variants differ only in the angles they rotate by.
$\ropeperiodic$ rotates the key and query subspaces by rational multiples of $2\pi$, so each rotary component cycles through finitely many phases and repeats with a fixed period, which a fixed-precision model can realize by a lookup table indexed by the position modulo that period.
Conventional $\rope$ instead rotates by frequencies whose phases are incommensurate with $2\pi$ and therefore never repeat (\cref{subsec:componentperiodic}).
Its relative scores can simulate the look-back of $k$ steps, only while a precision-dependent score gap is maintained. 
This connects $\rope$ to local attention \citep{li2026characterizingexpressivitylocalattention} and gives a formal account of fixed-offset heads observed in practice \citep{barbero2025roundroundgomakes}.
\Cref{fig:landscape} summarizes the landscape of fully uniform finite-precision transformers.

We vet our results experimentally by training transformers on controlled formal languages. 
We find that constructed periodic $\rope$ schedules length-generalize on languages requiring modular predicates. 
In contrast, no conventional base value generalizes on $\kleene{(\syma\syma)}$ or $\kleene{(\syma\symb)}$, suggesting that changing the base does not recover the modular construction. 
Across the same sweep, the mean longest-perfect length of conventional $\rope$ outperforms $\nope$ at six of seven bases on an $\ltlAcr[\yesterday]$ language, but underperforms at every base on both an $\ltlAcr[\past]$ language and a language requiring both $\past$ and $\yesterday$. 
These results support the locality account while suggesting that the resulting local bias can interfere with long-distance conditioning.

\begin{figure}[t!]
    \centering
    \begin{tikzpicture}[
        example/.style={font=\itshape, text=ETHGray!80!black},
        label node/.style={align=right, inner sep=2pt},
        fomod box/.style={fill=ETHBlue!15, draw=ETHBlue!50, thick, rounded corners=3pt},
        fo box/.style={fill=ETHGreen!15, draw=ETHGreen!50, thick, rounded corners=3pt},
        pfomod box/.style={fill=ETHRed!15, draw=ETHRed!50, thick, rounded corners=3pt},
        pfo box/.style={fill=ETHPetrol!15, draw=ETHPetrol!50, thick, rounded corners=3pt},
        ]


        \filldraw[fomod box] (-7,-3.2) rectangle (5.45,1.75);

        \filldraw[fo box] (-6.75,-1.85) rectangle (-0.2,1.5);

        \filldraw[pfomod box] (-4.8,-3) rectangle (5.15,0.25);

        \begin{scope}
            \clip[rounded corners=5pt] (-4.8,-1.85) rectangle (-0.2,0.25);
            \fill[ETHRed, opacity=0.15, rounded corners=3pt]
            (-4.8,-1.85) rectangle (-0.2,0.25);
        \end{scope}

        \filldraw[pfo box] (-4.65,-1.6) rectangle (-0.3,-0.4);


        \node[label node, ETHBlue!80!black, anchor=north west] at (-0.1, 1.35) {%
        \small\bfseries$\stl[\modpred]$: Rightmost hard-attention\\[1pt] + \small\bfseries \sinperiodic{}};
        \node[label node, anchor=north west] at (-0.1, 0.85) {%
        \small\citet{yang2024maskedhardattentiontransformersrecognize}};

        \node[label node, ETHGreen!70!black, anchor=north west] at (-6.65, 1.35) {%
        \small\bfseries$\stl$: Rightmost hard attention};
        \node[label node, anchor=north west] at (-6.65, 0.8) {%
        \small\citet{yang2024maskedhardattentiontransformersrecognize}};

        \node[label node, ETHRed!80!black, anchor=north] at (-1.2, -2) {%
        \small\bfseries$\ptl[\modpred]$: \textbf{\ropeperiodic{} / \sinperiodic{} attention}};
        \node[label node, anchor=north] at (-1.2, -2.5) {%
        \small\cref{thm:SmatEquivPfo}, \cref{thm:smatsin}};

        \node[label node, ETHPetrol!70!black, anchor=north west] at (-4.55, -0.55) {%
        \small\bfseries$\ptl$: Soft attention};
        \node[label node, anchor=north west] at (-4.55, -1.05) {%
        \small\citet{li2025characterizingexpressivityfixedprecisiontransformer}};

        \node[example] at (-5.7, -0.75)   {$\kleene{\alphabet}\syma\symb\kleene{\alphabet}$};
        \node[example] at (-0.85, -0.12)  {$\kleene{(\syma\symb)}$};
        \node[example] at (3.8, -2.6)  {$\kleene{(\syma\syma)}$};
        \node[example] at (-0.65, -0.95)  {$\kleene{\syma}$};
        \node[example] at (6.15, 0)     {\textsc{Parity}};

    \end{tikzpicture}
    \caption{The expressivity landscape of fully uniform finite-precision transformers. The \textcolor{ETHRed}{red box} shows our contributions: The connection of \ropeperiodic{} and \sinperiodic{} transformers to $\ptl[\modpred]$, which is incomparable with $\stl$ (capturing rightmost hard-attention $\nope$ transformers).
    Their intersection contains $\ptl$, corresponding to soft-attention $\nope$ transformers.}
    \label{fig:landscape}
\end{figure}
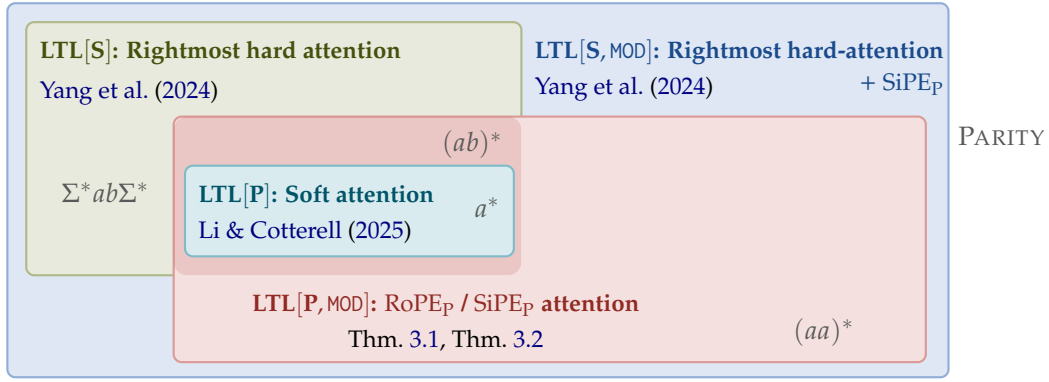

\section{Preliminaries}\label{sec:prelim}
An \defn{alphabet} $\alphabet$ is a finite, non-empty set of \defn{symbols} and $\kleene{\alphabet}$ denotes the set of all strings over $\alphabet$.
With $|\cdot|$, we denote the length of a string---in particular, we write $|w_1 \cdots w_N| = N$.
We write $\eps$ for the \defn{empty string}, the unique string of length 0.
A \defn{formal language} $\lang$ over $\alphabet$ is a subset of $\kleene{\alphabet}$.

\subsection{Linear Temporal Logic}
\defn{Regular languages} are languages that can be described by regular expressions; see \Cref{subsec:regex} for details. 
\emph{Subclasses} of regular languages can equivalently be described with different formalisms such as automata and logic. 
Past work \citep{yang2024maskedhardattentiontransformersrecognize, li2025characterizingexpressivityfixedprecisiontransformer,jerad-etal-2025-unique} links finite-precision transformers with subregular classes of languages. 
We use linear temporal logic (\ltlAcr) as our primary descriptive formalism. 

Given $\str=\sym_1\cdots\sym_\strlen$, \ltlAcr{} formulas are evaluated at a position $\idxi$; a formula accepts $\str$ when it holds at the readout position $\strlen+1$. 
The atomic predicate $\atom_\syma(\idxi)$ tests whether the symbol $\sym_\idxi$ at position $\idxi$ matches $\syma$.
Furthermore, \ltlAcr{} leverages \emph{temporal} operators to evaluate context-dependent properties at positions.
The operators we need are the following:
\begin{equation}
\str,\idxi\models\past\tlf\iff\exists\idxj<\idxi:\str,\idxj\models\tlf,
\qquad
\str,\idxi\models\yesterday\tlf\iff\idxi>1\ \land\ \str,\idxi-1\models\tlf.
\end{equation}
For $k\geq1$, we write $\yesterday^k\tlf$ for $k$-fold composition of $\yesterday$; equivalently, $\str,\idxi\models\yesterday^k\tlf$ if and only if $\idxi>k$ and $\str,\idxi-k\models\tlf$.
We write $\ltlAcr[\past]$ and $\ltlAcr[\past,\yesterday]$ for the fragments generated from these operators, atomic predicates, and Boolean connectives. The detailed \ltlAcr{} semantics and equivalent first-order, automata, and algebraic formalisms appear in \cref{app:flt}.
A third operator, $\since$ (``since''), subsumes both $\past$ and $\yesterday$ \citep{gabbay}, so $\stl$ is all of \ltlAcr{}---exactly the star-free languages \citep{Kamp1968-KAMTLA,mcnaughton1971counter}.
We use \defn{first-order logic} over the linear order $<$ interchangeably with \ltlAcr{}, via $\stl=\fo$ \citep{Kamp1968-KAMTLA} and $\ptl=\pfo$, the past fragment of the two-variable logic $\fotwo$ \citep{li2025characterizingexpressivityfixedprecisiontransformer}; \cref{app:first-order-logic} gives the formal definitions.

\paragraph{Modular predicates.}
The unary predicate $\modpred_m^r\left(\idxi\right)$ holds exactly when $\idxi\equiv r\pmod m$. 
We write $\ptl[\modpred]$ for $\ptl$ augmented with all such predicates. 
Adding modular predicates preserves the $\ltlAcr{}$--$\fo$ correspondence: $\ptl[\modpred]=\pfo[\modpred]$ (\cref{thm:main1}).

\subsection{Transformers}
A consistent takeaway of theoretical work is that transformers' computational abilities heavily depend on the modeling assumptions one makes \citep{hao-etal-2022-formal, jerad-etal-2025-unique,svete2026revisiting}.
Understanding the \emph{minimal viable} definition of a transformer is particularly useful, as it facilitates the analysis of the fundamental components of every transformer.
To this end, we adopt the fully uniform fixed-precision model of \citet{li2025characterizingexpressivityfixedprecisiontransformer}. 
Uniformity requires \emph{all strings} be processed by a model defined by a single set of parameters, in contrast to having a separate model for each length.

\paragraph{Finite precision.}
Transformer parameters and activations belong to a \emph{finite} set $\finiteset \defeq \set{\floatvalue_1, \ldots, \floatvalue_K}$, corresponding to real-world transformers (e.g., 32-bit floating-point arithmetic).
All computations are performed rounded to $\finiteset$, with
$\infty$ and $-\infty$ in $\finiteset$ which describe overflow.
We denote by $\floatlarge$ the smallest positive element of $\finiteset$ such that $\exp\left(-\floatlarge\right)$ rounds to $0$. 
This permits exact hard selection whenever the winning score exceeds every alternative by at least $\floatlarge$.

\paragraph{Transformers as string classifiers.}
We study transformers as \defn{string classifiers}.
An input to a transformer is a string $\str \in \kleene{\alphabet}$ augmented with the \underline{b}eginning-\underline{o}f-\underline{s}tring $\bos$ and \underline{e}nd-\underline{o}f-\underline{s}tring $\eos$ symbols, $\bos  \str  \eos$.
We set $\eosalphabet \defeq \alphabet \cup \set{\bos,\eos}$.
An initial embedding layer maps symbols in $\eosalphabet$ to distinct elements of $\finiteset^\hiddDim$, where $\hiddDim \in \N$ is the hidden state size.
The embedding layer is followed by a stack of $\nLayers$ transformer layers, each composed of an attention layer, a feedforward network, and layer-normalization.
For a layer $\layerIdx$ and position $\idxi$, $\vx_{\idxi}^\layerIdx\in \finiteset^\hiddDim$ is the \defn{contextual representation} of symbol $\sym_\idxi$ after being processed by layer $\layerIdx$.
We denote by $\tffunc \colon \kleene{(\eosalphabet)} \to \kleene{(\finiteset^\hiddDim)}$ the function that maps an input $\bos  \str \eos$ to contextual representations through the embedding layer and the $\nLayers$ transformer layers.
For a string $\str$ and string position $\idxi$, we denote by $\tffunc(\str)_\idxi\in \finiteset^\hiddDim$ the final $\idxi$'th contextual representation $\vx_{\idxi}^\nLayers$.
To enable language recognition, $\tffunc$ is paired with an output classification layer $\outputlayer \colon (\finiteset^\hiddDim) \to \set{0,1}$ acting on $\tffunc(\str)_\eos$ for an input $\str$.
Given a precision regime $\finiteset$, the tuple $\tf \defeq (\tffunc, \outputlayer)$ is an $\finiteset$-\defn{transformer} and defines a language $\lang(\tf) \defeq \set{\str \mid \outputlayer(\tffunc(\str)_\eos)=1}$.

\paragraph{Soft attention.}
We assume soft attention, which corresponds to standard implementations.
We note a minute but important practical detail: For numerical stability, soft attention is usually computed by subtracting from $\mS_{\idxi,\idxj}$---the attention score between the query $\vq_\idxi$ and the key $\vk_\idxj$---the maximum score.
We define the normalized score $\valpha_{\idxi,\idxj}$ as follows:
\begin{equation}
    \valpha_{\idxi,\idxj}=
    \frac{\exp\left(\mS_{\idxi,\idxj}-\max_{\idxk<\idxi}\mS_{\idxi,\idxk}\right)}
    {\sum_{\idxl=0}^{\idxi-1}\exp\left(\mS_{\idxi,\idxl}-\max_{\idxk<\idxi}\mS_{\idxi,\idxk}\right)}.
\end{equation}

These $\nope$ transformers ($\smat$) recognize exactly $\ptl=\pfo$ \citep{li2025characterizingexpressivityfixedprecisiontransformer}.

\subsection{Rotary Position Embeddings}\label{subsec:rope}
$\rope$ \citep{su2023roformerenhancedtransformerrotary} partitions the query's and key's $\hiddDim$-dimensional space into two-dimensional subspaces and \emph{rotates} each by an angle that depends on the absolute position $\idxi$.
Concretely, let
\begin{equation}
\mR_{\rotAngles,\idxi}\defeq
\operatorname{diag}\left(\rotationmat{\theta_1}{\idxi},\ldots,\rotationmat{\theta_{\hiddDim/2}}{\idxi}\right), 
\qquad
\rotationmat{\theta}{\idxi}\defeq
\begin{pmatrix}
\cos\left(\theta\idxi\right)&-\sin\left(\theta\idxi\right)\\
\sin\left(\theta\idxi\right)&\cos\left(\theta\idxi\right)
\end{pmatrix}
\end{equation}
where $\mR_{\rotAngles, \idxi}$ is a block-diagonal matrix and $\rotAngles \in [0, 2\pi]^{\hiddDim/2}$ is a tuple of angles assigned to the different two-by-two rotation matrices in $\mR_{\rotAngles, \idxi}$, which we call the \defn{schedule}.
When the schedule is implicit, we drop it from the subscript and write, e.g., $\mR_\idxi$; we restore it whenever several schedules are in play.
In conventional implementations, the angular frequencies are parameterized as $\theta_\idxd=\beta^{-2\left(\idxd-1\right)/\hiddDim}$, where $\beta>0$ is called the \emph{base}; the original $\rope$ construction uses $\beta=10^4$ \citep{su2023roformerenhancedtransformerrotary}.

At position $\idxi$, queries and keys become
\begin{equation}
\vq_\idxi=\mR_{\rotAngles,\idxi}\left(\mQ\vx_\idxi+\biasVector_q\right),
\qquad
\vk_\idxi=\mR_{\rotAngles,\idxi}\left(\mK\vx_\idxi+\biasVector_k\right).
\end{equation}
For exact rotations and zero biases, orthogonality gives the relative-position identity
\begin{equation}
\transpose{\vk_\idxj}\vq_\idxi
=\transpose{(\mR_{\rotAngles,\idxj}\mK\vx_\idxj)}\mR_{\rotAngles,\idxi}\mQ\vx_\idxi
= \transpose{\vx_\idxj}\transpose{\mK}\transpose{(\mR_{\rotAngles,\idxj})}\mR_{\rotAngles,\idxi}\mQ\vx_\idxi,
=\transpose{\vx_\idxj}\transpose{\mK}\mR_{\rotAngles,\idxi-\idxj}\mQ\vx_\idxi,
\end{equation}
making the score depend on the distance $\idxi-\idxj$. 

\paragraph{Ideal and realized rotations.}
The rotation above is \defn{ideal}: It is evaluated in exact arithmetic, and it stores real numbers.
A fixed-precision model computes something else.
It stores a schedule $\widehat{\rotAngles}$ and evaluates the rotation with the finite-precision arithmetic.
We write $\widehat{\mR}_{\widehat{\rotAngles},\idxi}\in\finiteset^{\hiddDim\times\hiddDim}$ for the \defn{realized} matrix in finite-precision arithmetic at position $\idxi$, and $\widehat{\mR}^{(\idxd)}_{\widehat{\rotAngles},\idxi}\in\finiteset^{2\times2}$ for its $\idxd\textsuperscript{th}$ block.
The hat marks the realized map as opposed to the ideal one $\idxi\mapsto\mR_{\rotAngles,\idxi}$.
The distinction matters because rounding need not preserve the structure of the ideal map; our periodicity definitions concern the realized map.

\paragraph{Custom schedules and unrotated subspaces.}
To establish lower bounds for \rope{} transformers, we permit a \emph{custom} schedule---a bespoke set of angles $\rotAngles$ chosen for the target language. 
Custom schedules are standard in expressivity literature \citep[e.g.,][]{chiang2023tighterboundsexpressivitytransformer,yang2024maskedhardattentiontransformersrecognize, yang2025kneedeepcrasptransformerdepth}. 
We also allow \emph{unrotated} subspaces with 
$\theta_\idxd=0$; these subsume partial-$\rope$ variants \citep{barbero2025llmsattendtoken,yangropetonope,khan2026fractionalrotationpotentialinvestigating} and retain the $\ptl$ computation available to $\nope$, while separate rotary dimensions supply modular or fixed-offset information.

\subsection{Component-Periodic and Non-Periodic \rope{}}\label{subsec:componentperiodic}

Intuitively, PEs with \emph{periodicity} behave similarly to modular positional predicates in logic \citep{chiang2023tighterboundsexpressivitytransformer,yang2024counting,yang2025kneedeepcrasptransformerdepth}.
We formalize this intuition for finite-precision \rope{}, with the proofs in \cref{app:schedules-proofs}.
Throughout, $\widehat{\mR}_\idxi$ and $\widehat{\mR}^{(\idxd)}_\idxi$ are the realized matrix and its $\idxd\textsuperscript{th}$ block.
\begin{definition}[Component-periodic and non-periodic $\rope$]\label{def:component-periodic-rope}
    A realized $\rope$ schedule $\widehat{\rotAngles}$ is \defn{component-periodic} if, for every $\idxd$, there is an integer $m_\idxd\geq1$ such that
    \begin{equation}
    \widehat{\mR}^{(\idxd)}_{\idxi+m_\idxd}=\widehat{\mR}^{(\idxd)}_\idxi
    \qquad\text{for all }\idxi\in\N.
    \end{equation}
    It is \defn{non-periodic} if this condition fails: Some component has no positive integer period.
\end{definition}
Because $\hiddDim$ is fixed, component-periodicity is equivalent to periodicity of the full map $\idxi\mapsto\widehat{\mR}_\idxi$.
In particular, $M\defeq\operatorname{lcm}\left(m_1,\ldots,m_{\hiddDim/2}\right)$ satisfies $\widehat{\mR}_{\idxi+M}=\widehat{\mR}_\idxi$ for every $\idxi$, with unrotated blocks having period $1$.
We denote by $\smat[\ropeperiodic]$ and $\smat[\ropenonperiodic]$ the class of languages recognized by fixed-precision component-periodic and non-periodic $\rope$ transformers, respectively.

We show that periodic schedules can be expressed via finitely many modular predicates:
\begin{restatable}[Component-periodicity via $\modpred$]{proposition}{periodicinputs}\label{prop:periodicinputs}
    If a realized $\rope$ schedule $\widehat{\rotAngles}$ is component-periodic, then its image is finite and, for every matrix coordinate and realized value, the positions attaining that value are definable by unary modular predicates.
\end{restatable}
A standard class of periodic schedules is the set of \defn{rational} multiples of $\pi$ \citep{yang2024maskedhardattentiontransformersrecognize}.
\begin{restatable}[Rational-angle functions are periodic]
{proposition}{rationalangles}\label{prop:rationalangles}
If $\theta_\idxd/\left(2\pi\right)=a_\idxd/m_\idxd\in\Q$, then $\idxi\mapsto\rotationmat{\theta_\idxd}{\idxi}$ has period $m_\idxd$.
\end{restatable}
Intuitively, adding $m_\idxd$ to $\idxi$ changes the phase by $2\pi a_\idxd$.

\Cref{prop:rationalangles} concerns the \emph{ideal} map, and it may not transfer to finite precision.
The position $\idxi$ grows without bound as the string does, so neither $\idxi$ nor the phase $\theta_\idxd\idxi$ stays representable in $\finiteset$ at all lengths.
An implementation that forms the phase and then rounds therefore computes a map that need not agree with $\rotationmat{\theta_\idxd}{\idxi}$, and that realized map need not repeat.
Rational angles alone thus do not make a schedule component-periodic in the sense of \cref{def:component-periodic-rope}.
We resort to realizing the same rational-angle function without forming the product $\theta_\idxd\idxi$.
For $\theta_\idxd=2\pi a_\idxd/m_\idxd$, we precompute the finite table
\begin{equation}
    T_\idxd\left(r\right)\defeq\operatorname{round}_\finiteset\left(\rotationmat{\theta_\idxd}{r}\right),
    \qquad 0\leq r<m_\idxd,
\end{equation}
and define the realized block directly by $\widehat{\mR}^{(\idxd)}_\idxi=T_\idxd\left(\idxi\bmod m_\idxd\right)$.
The table has $m_\idxd$ entries and the counter has $m_\idxd$ states, and both depend on the schedule but not on the input length, so the construction stays inside the fully uniform fixed-precision model.
The realized map then satisfies $\widehat{\mR}^{(\idxd)}_{\idxi+m_\idxd}=\widehat{\mR}^{(\idxd)}_\idxi$ for every $\idxi$ by construction.
Similar cyclic lookups appear in practice.
\citet{huo2026periodicropeinfinitecontext} precomputes one period of P-RoPE and indexes it by position modulo the period.
\citet{wang-etal-2024-resonance} round each $\rope$ wavelength to the nearest integer, so that every component repeats after a whole number of positions.
Both were introduced to improve length generalization, and both are component-periodic in the sense of \cref{def:component-periodic-rope}, so the characterization we prove next applies to them directly.

\section{A Characterization of $\smat[\ropeperiodic]$}\label{sec:smatrope}
This section describes the equivalence of $\smat[\ropeperiodic]$ with $\ptl[\modpred] =\pfo[\modpred]$.
The next two subsections sketch the intuition for the lower and upper bounds; full proofs are in \cref{app:periodic-proofs}.
\begin{restatable}{theorem}{SmatEquivPfo}\label{thm:SmatEquivPfo}
    $\smat[\ropeperiodic] = \ptl[\modpred] = \pfo[\modpred]$.
\end{restatable}

\subsection{From $\smat[\ropeperiodic]$ to $\pfo[\modpred]$}
Let $M$ be the common period to all the rotation matrices in the block diagonal matrix (\cref{subsec:componentperiodic}).
Then, every realized entry depends only on $\idxi\bmod M$, so equality to any realized value is a finite disjunction of predicates $\modpred_M^r\left(\idxi\right)$ (\cref{prop:periodicinputs}). 
The fixed-precision simulation of all remaining transformer operations by $\pfo$ \citep{li2025characterizingexpressivityfixedprecisiontransformer} then gives $\smat[\ropeperiodic]\subseteq\pfo[\modpred]$.
\begin{restatable}{lemma}{ropepupperbound}\label{lem:upperperiodic}
    Any $\finiteset$-transformer with $\ropeperiodic$ can be simulated by $\pfo[\modpred]$. 
    Hence, $\smat[\ropeperiodic] \subseteq \pfo[\modpred]=\ptl[\modpred]$.
\end{restatable}

\paragraph{\rope{} as a unary predicate.}
Although $\rope$ scores encode \emph{pairwise} distances, it does not require binary primitives; \emph{unary} ones suffice.
To see why, consider the binary modular predicate
\begin{equation}
    \modpred^r_m(\idxi, \idxj) = \true \iff (\idxj - \idxi) \equiv r\pmod m,
\end{equation}
which models the relative position dependency of $\rope$ by passing the distance between two positions through a periodic function.
We show in \cref{app:upperperiodic-proofs} that the binary modular predicate can be expressed by a Boolean combination of unary modular predicates. 
\begin{restatable}{proposition}{binarymodpred} \label{prop:binarymodpred}
    Binary modular predicates can be expressed with unary modular predicates and Boolean logic.
\end{restatable}

\subsection{From $\ptl[\modpred]$ to $\smat[\ropeperiodic]$}
For the reverse inclusion, $\finiteset$-transformers with $\ropeperiodic$ can first implement Boolean connectives and $\past$ in unrotated dimensions via structural induction \citep{li2025characterizingexpressivityfixedprecisiontransformer}. 
For each modular atom $\modpred_m^r\left(\idxi\right)$, a period-$m$ rotary pair uses the lookup table $T_m\left(s\right)=\rotationmat{2\pi/m}{s}$ for $0\leq s<m$ and returns $T_m\left(\idxi\bmod m\right)$. 
An attention head compares the returned phase with the fixed $\bos$ phase. A constant non-$\bos$ baseline separates the matching residue from every nonmatching residue, and stable softmax turns the comparison into an exact Boolean value. 
Thus $\ptl[\modpred]\subseteq\smat[\ropeperiodic]$.

\begin{restatable}{lemma}{ropeplowerbound}
    For every $\ptl[\modpred]$ formula, there is an $\finiteset$-transformer with $\ropeperiodic$ that simulates the formula. 
    Hence $\ptl[\modpred]\subseteq\smat[\ropeperiodic]$.
\end{restatable}

\paragraph{Connection to attention sinks.}
The use of $\bos$ as a fixed anchor connects the tightness construction to \defn{attention sinks}, where language models allocate disproportionate attention to initial tokens \citep{xiao2024efficientstreaminglanguagemodels}. 
One account of why sinks are useful is that they prevent contextual representations from \emph{over-mixing} and keep them sufficiently distinct \citep{barbero2025llmsattendtoken}.
Our construction offers a complementary reason: The anchor is what makes unary modular predicates computable, because the query's phase must be compared against a key whose own phase does not move.
This suggests that sinks in $\rope$ transformers serve in part to expose positional phase, and we note that Gemma 7B, which uses $\rope$, does concentrate attention on $\bos$ \citep{barbero2025roundroundgomakes}.

\subsection{Periodic Sinusoidal Positional Encodings}\label{sec:smatsin}
We find \defn{absolute sinusoidal} PEs \citep{vaswani2023attentionneed} to be equivalent to \ropeperiodic. 
As for $\ropeperiodic$, we denote by $\sinperiodic$ the class of
\defn{periodic} sinusoidal encodings such that the realized sinusoidal position-to-embedding map has a finite integer period as in \cref{def:component-periodic-rope}.
We similarly define $\sinnonperiodic$.

\begin{restatable}{theorem}{smatsinpf}\label{thm:smatsin}
    $\smat[\sinperiodic] = \ptl[\modpred] = \pfo[\modpred]$.
\end{restatable}
The proof (\cref{app:proofsmatsin}) again expresses the finite periodic image with modular predicates and uses custom periodic features to recover each modular atom. 
This shows that the key factor for the corresponding all-length logical class is the \emph{periodicity} of the encoding;
both the absolute and relative encodings are unified under modular predicates.

\subsection{Characterizing $\ptl[\modpred]$}\label{subsec:ptl-mod}
We have established that $\smat[\ropeperiodic]$ and $\smat[\sinperiodic]$ can both be exactly characterized by $\ptl[\modpred]$.
To better grasp the expressive power of \ropeperiodic{} and \sinperiodic{}, we provide additional characterizations of $\ptl[\modpred]$.
Building on \citet{dartois_et_al:LIPIcs.STACS.2013.329, li2025characterizingexpressivityfixedprecisiontransformer}, we prove the following equivalences\footnote{The relevant formal language theory tools are properly introduced in \cref{app:flt}.} in \cref{app:pfo2mod}.
\begin{restatable}{theorem}{equivalence}\label{thm:main1}
    Let $\lang\subseteq\kleene{\alphabet}$ be a regular language, $\monoid$ be its syntactic monoid, and $\automaton$ be the minimal DFA accepting it. The following assertions are equivalent:
    \begin{enumerate}[label=(\roman*),topsep=0pt,noitemsep]
        \item $\lang$ is a left-deterministic modular polynomial
        \item $\monoid \in \QRMonoid$
        \item $\exists k \in \N$ such that $\automaton_k$ is partially ordered
        \item $\lang = \lang(\phi)$ for some $\pfo[\modpred]$ formula $\phi$
        \item $\lang = \lang(\tlf)$ for some $\ptl[\modpred]$ formula $\tlf$
    \end{enumerate}
\end{restatable}
The characterization provides novel tools for (in)expressibility statements about $\ptl[\modpred]$.
We can first show via $\pfo[\modpred]$ that $\ptl[\modpred]$ can leverage modular predicates to recognize languages with a \emph{periodic behavior} where a substring is repeated indefinitely.
\begin{proposition}\label{prop:periodic}
    For every nonempty word $\sym_1\cdots\sym_m$, the language $\kleene{\left(\sym_1\cdots\sym_m\right)}$ belongs to $\pfo[\modpred]$.
\end{proposition}
\begin{proof}
    It is defined by $\modpred_m^0\land\forall\idxi\,\bigwedge_{0\leq\idxj< m}\left(\modpred_m^{\idxj}\left(\idxi\right)\Rightarrow\atom_{\sym_\idxj}\left(\idxi\right)\right)$, which is in $\pfo[\modpred]$.
\end{proof}
Importantly, \cref{prop:periodic} implies that languages such as $\kleene{(\syma \symb)}$ and $\kleene{(\syma \syma)}$ belong to $\pfo[\modpred]$, leading to the following conclusions.

\paragraph{$\ptl[\modpred]$ and $\stl$ are incomparable.}
In other words, neither class contains the other.
In one direction, $\kleene{(\syma\syma)}$ is not star-free \citep{yang2024maskedhardattentiontransformersrecognize} and so lies outside $\stl$, yet \cref{prop:periodic} places it in $\ptl[\modpred]$.
In the other direction, $\stl$ can recognize the \defn{locally testable} languages \citep{ZALCSTEIN1972151}---those that require checking that a finite sequence of symbols occur at \emph{consecutive} positions.
Written in \ltlAcr{}, this needs the $\yesterday$ operator: $\kleene{\alphabet}\syma\symb\kleene{\alphabet}$ is defined by $\past\left(\atom_\symb\land\yesterday\atom_\syma\right)$.
The only temporal operator of $\ptl[\modpred]$ is $\past$, which cannot step to an adjacent position, and we show in \cref{app:inexpressibility} that $\ptl[\modpred]$ contains no such language.

\paragraph{$\ptl[\modpred]$ cannot perform modular counting.}
Modular predicates count \emph{positions}, not \emph{symbols}.
Adding them to $\ptl$ therefore does not let it decide whether the number of occurrences of a symbol is divisible by an integer.
The canonical example is \textsc{Parity}, the set of bitstrings with an even number of $1$s.
We prove in \cref{app:inexpressibility} that \textsc{Parity} is not $\ptl[\modpred]$.
\subsection{Testing the Modular Account}\label{sec:experiments-periodic}\label{sec:experiments-rational}\label{sec:experiments}
We train transformer classifiers on the languages in \cref{tab:languages}, which also records the account each language probes. Training strings have length at most $40$; test strings have lengths $41$--$500$. Every configuration is run over 5 seeds and 3 learning rates, and every conventional-schedule configuration is fully rotated. Besides accuracy, we report the longest perfect length $N^*$: the largest $N$ for which accuracy is $100\%$ at every tested length through $N$. \Cref{app:experiments} gives the architecture and optimization settings and reports per-run means and standard deviations.
\Cref{tab:exp-periodic} compares $\nope$ with the period-targeting variants, while \cref{fig:exp-modular-base-sweep} reports the conventional non-periodic controls. 
In both variants, $\rope$ is applied on all dimensions.

\begin{table}[t]
\centering
\small
\begin{tabular}{@{}llll@{}}
\toprule
Language & Defining property & Class & Account \\
\midrule
$\kleene{(\syma\syma)}$ & even length & $\ptl[\modpred]$; not star-free & modular \\
$\kleene{(\syma\symb)}$ & $\syma\symb$ repeated & $\ptl[\modpred]$; not $\ptl$ & modular \\
$\syma\kleene{\alphabet}$ & starts with $\syma$ & $\ptl$ & locality \\
$\kleene{\alphabet}\syma$ & ends with $\syma$ & $\ltlAcr[\yesterday]$; not $\ptl$ & locality \\
$\kleene{\alphabet}\syma\symb\kleene{\alphabet}$ & contains $\syma\symb$ & $\ltlAcr[\past,\yesterday]$; locally testable & locality \\
\bottomrule
\end{tabular}
\caption{The formal languages we train on, each over an alphabet $\alphabet$ with $|\alphabet|>1$. The last column records which account the language probes: the modular account of \cref{sec:smatrope} or the locality account of \cref{sec:lengthbound}. \Cref{app:experiment-languages} gives the full definitions.}
\label{tab:languages}
\end{table}

\begin{table}[H]
\centering
\begin{tabular}{ll cc cc cc}
\toprule
& & \multicolumn{2}{c}{$\nope$} & \multicolumn{2}{c}{$\ropeperiodic$} & \multicolumn{2}{c}{$\sinperiodic$} \\
\cmidrule(lr){3-4} \cmidrule(lr){5-6} \cmidrule(lr){7-8}
Language & $\in$ Class & Acc. & $N^*$ & Acc. & $N^*$ & Acc. & $N^*$ \\
\midrule
$\kleene{(\syma \syma)}$          & $\ptl[\modpred]$ & 0.50 & 41 & \textbf{1.00} & \textbf{500} & 0.50 & 41 \\
$\kleene{(\syma \symb)}$          & $\ptl[\modpred]$ & 0.96 & 96 & \textbf{1.00} & \textbf{500} & 0.50 & 44 \\
\bottomrule
\end{tabular}
\caption{Periodic-construction experiments. Maximum accuracy and longest perfect length over the run grid achieved by different transformer variants.}
\label{tab:exp-periodic}
\label{tab:exp-pfo2mod}
\end{table}

\begin{figure}[H]
\centering
\resizebox{\textwidth}{!}{%
\begin{tikzpicture}
\begin{groupplot}[
    group style={group size=2 by 1, horizontal sep=1.2cm},
    width=0.38\textwidth,
    height=0.30\textwidth,
    xmin=-13,
    xmax=13,
    ymin=-15,
    ymax=520,
    y coord trafo/.code={\pgfmathparse{#1<=100 ? #1 : 100+(#1-100)/10}},
    y coord inv trafo/.code={\pgfmathparse{#1<=100 ? #1 : 100+(#1-100)*10}},
    ytick={0,25,50,75,100,500},
    yticklabels={0,25,50,75,100,500},
    extra y ticks={300},
    extra y tick labels={$\sslash$},
    extra y tick style={grid=none,tick style={draw=none},tick label style={font=\scriptsize}},
    xtick={-12,-8,-4,0,4,8,12},
    xticklabels={$10^{-12}$,$10^{-8}$,$10^{-4}$,$1$,$10^4$,$10^8$,$10^{12}$},
    xticklabel style={font=\tiny,rotate=45,anchor=east},
    xlabel={Base},
    tick label style={font=\scriptsize},
    label style={font=\small},
    title style={font=\small},
    grid=major,
    grid style={gray!20},
    axis line style={gray!60},
    error bars/error bar style={line width=0.5pt},
    error bars/error mark options={rotate=90,mark size=1.5pt},
]
\nextgroupplot[title={$\kleene{(\syma\syma)}$},ylabel={Mean $N^*$},
    legend to name=modularerrlegend,
    legend columns=1,
    legend cell align=left,
    legend style={draw=none,font=\scriptsize,row sep=2pt}]
\addplot[ETHBlue,thick,mark=*,mark size=1.5pt,error bars/.cd,y dir=both,y explicit]
    coordinates {(-12,41.3) +- (0,11.6) (-8,50.1) +- (0,14.9) (-4,49.5) +- (0,3.1)
                 (0,34.3) +- (0,17.2) (4,26.3) +- (0,21.5) (8,22.7) +- (0,21.3) (12,23.3) +- (0,21.9)};
\addlegendentry{$\ropenonperiodic$}
\addplot[ETHRed,thick,dashed,mark=square*,mark size=1.5pt,error bars/.cd,y dir=both,y explicit]
    coordinates {(-12,16.4) +- (0,20.1) (-8,71.3) +- (0,11.9) (-4,75.7) +- (0,6.3)
                 (0,8.5) +- (0,16.9) (4,2.7) +- (0,10.2) (8,0.0) +- (0,0.0) (12,0.0) +- (0,0.0)};
\addlegendentry{$\sinnonperiodic$}
\addplot[black!80,very thick,densely dashed,mark=none] coordinates {(-12,41) (12,41)};
\addlegendentry{$\nope$ (best run)}
\addplot[ETHGreen!70!black,very thick,loosely dashed,mark=none] coordinates {(-12,500) (12,500)};
\addlegendentry{$\ropeperiodic$ (best run)}

\nextgroupplot[title={$\kleene{(\syma\symb)}$}]
\addplot[ETHBlue,thick,mark=*,mark size=1.5pt,error bars/.cd,y dir=both,y explicit]
    coordinates {(-12,65.1) +- (0,13.0) (-8,48.9) +- (0,5.2) (-4,52.8) +- (0,5.6)
                 (0,48.1) +- (0,3.1) (4,66.0) +- (0,15.7) (8,57.1) +- (0,14.7) (12,50.1) +- (0,6.5)};
\addplot[ETHRed,thick,dashed,mark=square*,mark size=1.5pt,error bars/.cd,y dir=both,y explicit]
    coordinates {(-12,16.8) +- (0,20.6) (-8,41.9) +- (0,16.7) (-4,52.9) +- (0,7.6)
                 (0,23.1) +- (0,21.6) (4,17.1) +- (0,20.9) (8,26.3) +- (0,21.6) (12,24.1) +- (0,22.8)};
\addplot[black!80,very thick,densely dashed,mark=none] coordinates {(-12,96) (12,96)};
\addplot[ETHGreen!70!black,very thick,loosely dashed,mark=none] coordinates {(-12,500) (12,500)};
\end{groupplot}
\node[anchor=west] at ($(group c2r1.east)+(0.6cm,0)$) {\ref{modularerrlegend}};
\end{tikzpicture}
}
\caption{Mean longest perfect length $N^*$ for conventional non-periodic encodings on the modular languages. Markers are means over the run grid and whiskers are $\pm1$ standard deviation. The two reference lines are best-run values for $\nope$ and $\ropeperiodic$ on these tasks. Neither non-periodic encoding generalizes perfectly at any tested base.}
\label{fig:exp-modular-base-sweep}
\end{figure}
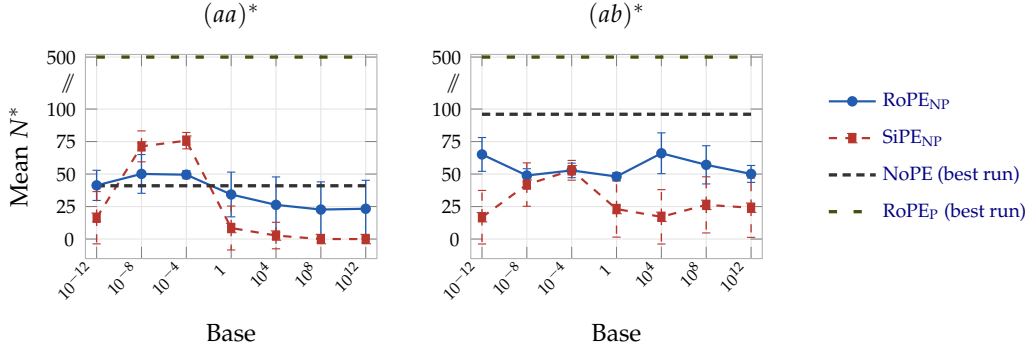

\paragraph{$\ropeperiodic$ generalizes perfectly on $\ptl[\modpred]$.}
The period-targeting $\ropeperiodic$ model achieves accuracy $1.00$ and $N^*=500$ on both $\kleene{(\syma\symb)}$ and $\kleene{(\syma\syma)}$, matching the modular-predicate construction.

\paragraph{$\sinperiodic$ struggles on $\ptl[\modpred]$.}
Although $\sinperiodic$ has the same existential expressivity, the trained model does not learn either modular language. The theorem guarantees a parameterization, not that optimization will find it; one possible explanation is that absolute sinusoidal features in the residual stream are more easily overwritten.

\paragraph{Conventional rotations do not learn the constructed modular behavior.}
Neither $\ropenonperiodic$ nor $\sinnonperiodic$ generalizes perfectly on the two periodic languages. The full sweep in \cref{fig:exp-modular-base-sweep} reaches the same conclusion at every tested base across twenty-four orders of magnitude, separating the engineered periodic behavior from conventional frequencies.

\section{A Practical Account of Non-Periodic RoPE}\label{sec:lengthbound}
We now analyze $\rope$'s role implementing fixed-offset heads \citep{barbero2025roundroundgomakes}---referencing positions at a specific displacement from the query position.
This account applies naturally to \emph{conventional} $\ropenonperiodic$, but does not capture its \emph{asymptotic} behavior.
Rather, it relies on a finite-precision range over which $\rope$ can create a reliable fixed-offset head, partly explaining its empirical advantage over $\sinnonperiodic$ \citep{su2023roformerenhancedtransformerrotary}.
The proof for the bounded non-periodic $\rope{}$ simulations is in \cref{app:nonperiodic-proofs}.

\subsection{Conventional RoPE Is Non-Periodic}\label{sec:fixed-precision-nonperiodicity}
We first pinpoint why conventional $\rope{}$ is non-periodic.
Namely, \emph{irrational} rotations, those that conventional $\rope$ leverages, imply non-periodicity.
\begin{restatable}[Specialization of \citealt{walters1982introduction}, Thm 1.8]{proposition}{roundedropenonperiodic}\label{prop:rounded-rope-nonperiodic}
    If $g/\left(2\pi\right)\notin\Q$, then $\idxi\mapsto\operatorname{round}_\finiteset\left(\rotationmat{g}{\idxi}\right)$ is non-periodic. Hence, conventional $\rope$, which contains $g=1$, is non-periodic under output-rounded ideal-phase evaluation.
\end{restatable}

\subsection{Fixed-Offset Attention up to a Length Bound}
We now proceed to show that with irrational rotations, $\ropenonperiodic$ transformers can implement fixed-offset attention heads.  

Fix an offset $k\geq1$. With exact arithmetic, a frequency $g/\left(2\pi\right)\notin\Q$ yields a $k$-offset head whose score is
\begin{equation}\label{eq:ideal-previous-score}
    \mS^{g,C,k}_{\idxi,\idxj}=C\cos\left(g\left(\idxi-k-\idxj\right)\right),\qquad 0\leq\idxj<\idxi,
\end{equation}
for some $C>0$ \citep{barbero2025roundroundgomakes}. 
At positions $\idxi\geq k+1$, the target $\idxj=\idxi-k$ uniquely maximizes \Cref{eq:ideal-previous-score} with score $C$.
For the uniqueness to hold under finite precision, the target must remain ahead by at least $\floatlarge$ after rounding in the score computation.

We define the $k$-offset head's \defn{certified range} as the maximum length at which it scores its target position at least $\floatlarge$ more than any non-target position:
\begin{equation}\label{eq:yesterday-bound}
\boundyesterday[(k)]\left(g,C,\finiteset\right)\defeq
\sup\left\{L\in\N:\widehat{\mS}^{g,C,k,\finiteset}_{\idxi,\idxi-k}-\widehat{\mS}^{g,C,k,\finiteset}_{\idxi,\idxj}\geq\floatlarge
\ \substack{\text{for all }k+1\leq\idxi\leq L+1,\\0\leq\idxj<\idxi,\ \idxj\neq\idxi-k}\right\}.
\end{equation}
Under exact evaluation of \cref{eq:ideal-previous-score} through length $L\geq k$, the sufficient gap is $C\delta_L\left(g\right)$, where
$\delta_L\left(g\right)\defeq1-\max_{1\leq\idxn\leq L}\cos\left(g\idxn\right)$.
Thus $C\delta_L\left(g\right)\geq\floatlarge$ certifies $\boundyesterday[(k)]\left(g,C,\finiteset\right)\geq L$.

\subsection{From Fixed-Offset Heads to $\yesterday^k$}
If the certified range is $L$, stable softmax makes position $\idxi-k$'s weight exactly $1$ for every query position $\idxi \in \{k+1, \ldots L+1\}$. 
Unrotated dimensions independently simulate $\ptl$ and provide the boundary test for the first $k$ positions.
\begin{restatable}{corollary}{yesterdaypf}\label{corr:yesterday}
    Fix $g$, $C$, and $\finiteset$, and let $L\in\N$. Write $\smat[\ropenonperiodic]$-transformers for soft-attention $\ropenonperiodic$ transformers with unrotated subspaces, and let $L$-simulation be as in \cref{def:boundedsimulation}.
    \begin{enumerate}[label=(\roman*),leftmargin=*,topsep=1pt,itemsep=0.5pt]
        \item For every offset $k\geq1$ with $k\leq L\leq\boundyesterday[(k)]\left(g,C,\finiteset\right)$, such a transformer $L$-simulates the fixed-offset operator $\yesterday^k$.
        \item If $1\leq L\leq\boundyesterday[(1)]\left(g,C,\finiteset\right)$, it $L$-simulates every $\ltlAcr[\past,\yesterday]$-definable formula. More generally, for a finite set of offsets $\offsetset\subseteq\N_{\geq1}$ with $\max\offsetset\leq L\leq\min_{k\in\offsetset}\boundyesterday[(k)]\left(g,C,\finiteset\right)$, it $L$-simulates every formula built from atomic predicates, Boolean connectives, $\past$, and the operators $\set{\yesterday^k:k\in\offsetset}$.
        \item Moreover, for every $L\in\N$ and every finite offset set $\offsetset$ there exist a rotation angle $g$ with $g/\left(2\pi\right)\notin\Q$, scales $C$, and a finite-precision regime $\finiteset$ satisfying the hypothesis of (ii).
    \end{enumerate}
\end{restatable}
The connection to $\ltlAcr[\past,\yesterday]$ links the result to the extra power of local attention \citep{li2026characterizingexpressivitylocalattention} and resembles the finite-precision implementation of $n$-gram heads \citep{svete-cotterell-2024-transformers,svete-etal-2024-transformers}.

\subsection{Testing the Locality Account}\label{sec:experiments-nonperiodic}\label{sec:experiments-irrational}

We evaluate whether the bounded fixed-offset mechanism appears in length generalization, under the protocol of \cref{sec:experiments-periodic}.
$\rope$ is again applied on all dimensions.

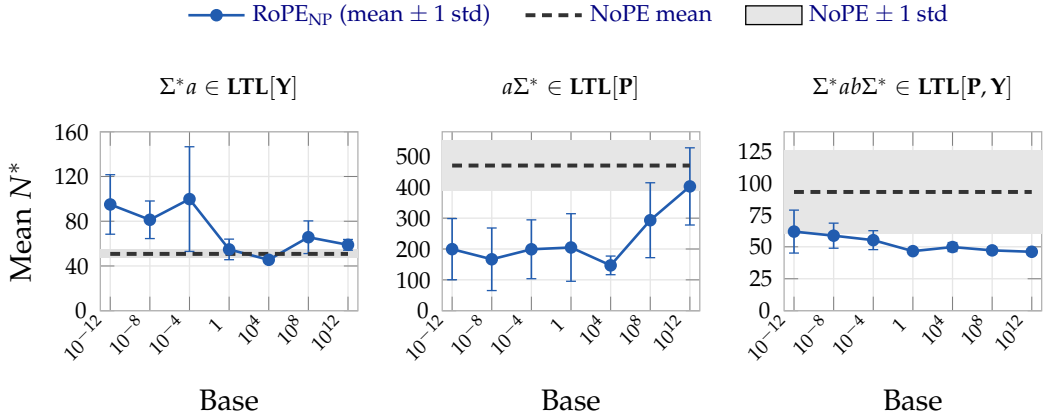
\begin{figure}[H]
\centering
\resizebox{\textwidth}{!}{%
\begin{tikzpicture}
\begin{groupplot}[
    group style={group size=3 by 1, horizontal sep=0.9cm},
    width=0.31\textwidth,
    height=0.25\textwidth,
    xmin=-13,
    xmax=13,
    ymin=0,
    xtick={-12,-8,-4,0,4,8,12},
    xticklabels={$10^{-12}$,$10^{-8}$,$10^{-4}$,$1$,$10^4$,$10^8$,$10^{12}$},
    xticklabel style={font=\tiny,rotate=45,anchor=east},
    xlabel={Base},
    tick label style={font=\scriptsize},
    label style={font=\small},
    title style={font=\scriptsize},
    grid=major,
    grid style={gray!20},
    axis line style={gray!60},
    error bars/error bar style={line width=0.5pt},
    error bars/error mark options={rotate=90,mark size=1.5pt},
]
\nextgroupplot[title={$\kleene{\alphabet}\syma\in\ltlAcr[\yesterday]$},ylabel={Mean $N^*$},
    ymax=160,ytick={0,40,80,120,160},
    legend to name=localityerrlegend,
    legend columns=3,
    legend cell align=left,
    legend style={draw=none,font=\scriptsize,/tikz/every even column/.append style={column sep=0.8em}}]
\path[fill=black!12] (axis cs:-13,46.9) rectangle (axis cs:13,54.7);
\addplot[ETHBlue,thick,mark=*,mark size=1.5pt,error bars/.cd,y dir=both,y explicit]
    coordinates {(-12,95.0) +- (0,26.6) (-8,81.3) +- (0,16.8) (-4,99.8) +- (0,46.8)
                 (0,54.8) +- (0,9.2) (4,45.5) +- (0,2.3) (8,65.7) +- (0,14.6) (12,58.9) +- (0,4.8)};
\addlegendentry{$\ropenonperiodic$ (mean $\pm$ 1 std)}
\addplot[black!80,very thick,densely dashed,mark=none] coordinates {(-12,50.8) (12,50.8)};
\addlegendentry{$\nope$ mean}
\addlegendimage{area legend,fill=black!10,draw=none}
\addlegendentry{$\nope$ $\pm$ 1 std}

\nextgroupplot[title={$\syma\kleene{\alphabet}\in\ltlAcr[\past]$},
    ymax=580,ytick={0,100,200,300,400,500}]
\path[fill=black!10] (axis cs:-13,388.5) rectangle (axis cs:13,552.5);
\addplot[ETHBlue,thick,mark=*,mark size=1.5pt,error bars/.cd,y dir=both,y explicit]
    coordinates {(-12,199.5) +- (0,99.1) (-8,166.7) +- (0,101.4) (-4,199.1) +- (0,95.4)
                 (0,204.9) +- (0,109.5) (4,146.9) +- (0,30.2) (8,293.3) +- (0,121.4) (12,402.9) +- (0,125.1)};
\addplot[black!80,very thick,densely dashed,mark=none] coordinates {(-12,470.5) (12,470.5)};

\nextgroupplot[title={$\kleene{\alphabet}\syma\symb\kleene{\alphabet}\in\ltlAcr[\past,\yesterday]$},
    ymax=140,ytick={0,25,50,75,100,125}]
\path[fill=black!10] (axis cs:-13,60.3) rectangle (axis cs:13,125.5);
\addplot[ETHBlue,thick,mark=*,mark size=1.5pt,error bars/.cd,y dir=both,y explicit]
    coordinates {(-12,61.9) +- (0,16.8) (-8,58.7) +- (0,9.8) (-4,55.2) +- (0,7.4)
                 (0,46.6) +- (0,1.5) (4,49.8) +- (0,3.5) (8,47.2) +- (0,2.6) (12,46.1) +- (0,2.2)};
\addplot[black!80,very thick,densely dashed,mark=none] coordinates {(-12,92.9) (12,92.9)};
\end{groupplot}
\node[anchor=south] at ($(group c1r1.north)!0.5!(group c3r1.north)+(0,0.85cm)$) {\ref{localityerrlegend}};
\end{tikzpicture}
}
\caption{Mean longest perfect length $N^*$. Whiskers are $\pm1$ standard deviation over the run grid and the grey band is the corresponding $\nope$ interval. 
}
\label{fig:exp-nonperiodic-base-sweep}
\end{figure}

\paragraph{The standard base is not uniformly best.}
On each single-operator language, the conventional base $\beta=10^4$ gives the lowest mean $N^*$ among the seven tested bases: $45.5$ on $\kleene{\alphabet}\syma$ and $146.9$ on $\syma\kleene{\alphabet}$. Every tested base on either side of $10^4$ improves on these two values. This mirrors the extrapolation experiments of \citet{liu2024scalinglawsropebasedextrapolation}, who find $10^4$ to be the worst tested base after fine-tuning and report improvements from both smaller and larger bases. The pattern is not universal, as, on the combined $\past$--$\yesterday$ language, $10^12$ performs worse. 
This suggests that the optimal base is task-dependent.

\paragraph{$\ropenonperiodic$ improves on a $\yesterday$ language.}
On $\kleene{\alphabet}\syma$, $\ropenonperiodic$'s mean $N^*$ exceeds $\nope$'s at six of seven bases (\cref{fig:exp-nonperiodic-base-sweep}). The certified-range result offers a mechanism: Relative scores can isolate the preceding position while their finite-precision gap remains large enough. As length grows, more competing offsets appear, increasing the chance of a numerical collision and making the behavior brittle.

\paragraph{$\ropenonperiodic$ can disrupt $\past$ reasoning.}
On $\syma\kleene{\alphabet}$, the $\nope$ mean is $470.5$, whereas the $\ropenonperiodic$ means range from $146.9$ to $402.9$. On $\kleene{\alphabet}\syma\symb\kleene{\alphabet}$, the $\nope$ mean is $92.9$, whereas the $\ropenonperiodic$ means range from $46.1$ to $61.9$. This suggests a tradeoff: Relative displacement can help retrieve a nearby symbol while disrupting the nearly position-invariant aggregation that supports conditioning on distant information. Conventional $\rope$ therefore behaves more like a bounded source of $\yesterday$ access than a monotonic expressivity improvement over $\nope$.
We conjecture that partial-$\rope{}$ \citep{barbero2025roundroundgomakes,khan2026fractionalrotationpotentialinvestigating} and $\nope$-$\rope{}$ hybrid attention \citep{yangropetonope} variants should empirically improve on long-context $\past$ reasoning while retaining the locality bias of $\ropenonperiodic$.

\section{Discussion}\label{sec:discussion}

\paragraph{Two regimes.}
For component-periodic schedules, $\ropeperiodic$ yields precisely $\ptl[\modpred]$ over arbitrary input lengths, enabling modular reasoning over input positions. 
The conventional $\ropenonperiodic$ behaves qualitatively differently: Relative scores can directly identify any fixed offset $k$, but only while finite precision preserves a sufficient score gap, yielding bounded $\yesterday^k$ operators and a bounded simulation of $\ltlAcr[\past,\yesterday]$. 
The two accounts therefore differ both in their schedule assumption and in whether the claim is uniform over all lengths.
With unrotated subspaces, \rope{} also retains the $\past$ reasoning of $\nope$---checking whether a symbol occurs anywhere in the context.

\paragraph{Limitations.}
Our periodic lower bound uses a repeated finite table and a tailored exact value set, and does not imply that standard frequencies or optimization discover modular behavior.
Empirical results, however, suggest that strong modular behavior often appears.
The locality result of \cref{sec:lengthbound} is bounded and is not an all-length characterization of $\ropenonperiodic$. 

\paragraph{Architectural implications.}
The comparison suggests reserving unrotated dimensions when a model must combine global $\past$ reasoning with periodic or local relative-position information. It also supports treating rotation structure as an explicit architectural choice, consistent with work questioning broad long-context claims for $\rope$ \citep{du2026ropedistinguishespositionstokens}. 
Neither account gives $\rope$ general state-tracking or regular language recognition abilities \citep{liu2023transformers}, such as those required by Flip-Flop language modeling \citep{liu2023exposingattentionglitchesflipflop}. 
Alternative state-tracking mechanisms include rational transductors \citep{mohri2026rationaltransductors} and transformer--linear-RNN hybrids \citep{merrill2026olmohybridtheorypractice,merrill2026linearrnnsparallelizable}. A remaining theoretical question is how the known depth hierarchy for $\rope$ transformers \citep{yang2025kneedeepcrasptransformerdepth} differs between exact periodic and bounded-local regimes.

\section*{Acknowledgments}
We are grateful to William Merrill and Andy Yang for valuable comments on this work.
Selim Jerad thanks Makoto Yamada and the MLDS Unit at the Okinawa Institute of Science and Technology for hosting him while writing this paper.
Anej Svete and Jiaoda Li are supported by an ETH AI Center Doctoral Fellowship.
We used generative AI to
improve our writing.
Every modification introduced by generative AI
was carefully reviewed by the authors, who take
full responsibility for it.

\bibliography{colm}

@misc{grattafiori2024llama3herdmodels,
  archiveprefix = {arXiv},
  author        = {Aaron Grattafiori and Abhimanyu Dubey and Abhinav Jauhri and Abhinav Pandey and Abhishek Kadian and Ahmad Al-Dahle and Aiesha Letman and Akhil Mathur and Alan Schelten and Alex Vaughan and Amy Yang and Angela Fan and Anirudh Goyal and Anthony Hartshorn and Aobo Yang and Archi Mitra and Archie Sravankumar and Artem Korenev and Arthur Hinsvark and Arun Rao and Aston Zhang and Aurelien Rodriguez and Austen Gregerson and Ava Spataru and Baptiste Roziere and Bethany Biron and Binh Tang and Bobbie Chern and Charlotte Caucheteux and Chaya Nayak and Chloe Bi and Chris Marra and Chris McConnell and Christian Keller and Christophe Touret and Chunyang Wu and Corinne Wong and Cristian Canton Ferrer and Cyrus Nikolaidis and Damien Allonsius and Daniel Song and Danielle Pintz and Danny Livshits and Danny Wyatt and David Esiobu and Dhruv Choudhary and Dhruv Mahajan and Diego Garcia-Olano and Diego Perino and Dieuwke Hupkes and Egor Lakomkin and Ehab AlBadawy and Elina Lobanova and Emily Dinan and Eric Michael Smith and Filip Radenovic and Francisco Guzmán and Frank Zhang and Gabriel Synnaeve and Gabrielle Lee and Georgia Lewis Anderson and Govind Thattai and Graeme Nail and Gregoire Mialon and Guan Pang and Guillem Cucurell and Hailey Nguyen and Hannah Korevaar and Hu Xu and Hugo Touvron and Iliyan Zarov and Imanol Arrieta Ibarra and Isabel Kloumann and Ishan Misra and Ivan Evtimov and Jack Zhang and Jade Copet and Jaewon Lee and Jan Geffert and Jana Vranes and Jason Park and Jay Mahadeokar and Jeet Shah and Jelmer van der Linde and Jennifer Billock and Jenny Hong and Jenya Lee and Jeremy Fu and Jianfeng Chi and Jianyu Huang and Jiawen Liu and Jie Wang and Jiecao Yu and Joanna Bitton and Joe Spisak and Jongsoo Park and Joseph Rocca and Joshua Johnstun and Joshua Saxe and Junteng Jia and Kalyan Vasuden Alwala and Karthik Prasad and Kartikeya Upasani and Kate Plawiak and Ke Li and Kenneth Heafield and Kevin Stone and Khalid El-Arini and Krithika Iyer and Kshitiz Malik and Kuenley Chiu and Kunal Bhalla and Kushal Lakhotia and Lauren Rantala-Yeary and Laurens van der Maaten and Lawrence Chen and Liang Tan and Liz Jenkins and Louis Martin and Lovish Madaan and Lubo Malo and Lukas Blecher and Lukas Landzaat and Luke de Oliveira and Madeline Muzzi and Mahesh Pasupuleti and Mannat Singh and Manohar Paluri and Marcin Kardas and Maria Tsimpoukelli and Mathew Oldham and Mathieu Rita and Maya Pavlova and Melanie Kambadur and Mike Lewis and Min Si and Mitesh Kumar Singh and Mona Hassan and Naman Goyal and Narjes Torabi and Nikolay Bashlykov and Nikolay Bogoychev and Niladri Chatterji and Ning Zhang and Olivier Duchenne and Onur Çelebi and Patrick Alrassy and Pengchuan Zhang and Pengwei Li and Petar Vasic and Peter Weng and Prajjwal Bhargava and Pratik Dubal and Praveen Krishnan and Punit Singh Koura and Puxin Xu and Qing He and Qingxiao Dong and Ragavan Srinivasan and Raj Ganapathy and Ramon Calderer and Ricardo Silveira Cabral and Robert Stojnic and Roberta Raileanu and Rohan Maheswari and Rohit Girdhar and Rohit Patel and Romain Sauvestre and Ronnie Polidoro and Roshan Sumbaly and Ross Taylor and Ruan Silva and Rui Hou and Rui Wang and Saghar Hosseini and Sahana Chennabasappa and Sanjay Singh and Sean Bell and Seohyun Sonia Kim and Sergey Edunov and Shaoliang Nie and Sharan Narang and Sharath Raparthy and Sheng Shen and Shengye Wan and Shruti Bhosale and Shun Zhang and Simon Vandenhende and Soumya Batra and Spencer Whitman and Sten Sootla and Stephane Collot and Suchin Gururangan and Sydney Borodinsky and Tamar Herman and Tara Fowler and Tarek Sheasha and Thomas Georgiou and Thomas Scialom and Tobias Speckbacher and Todor Mihaylov and Tong Xiao and Ujjwal Karn and Vedanuj Goswami and Vibhor Gupta and Vignesh Ramanathan and Viktor Kerkez and Vincent Gonguet and Virginie Do and Vish Vogeti and Vítor Albiero and Vladan Petrovic and Weiwei Chu and Wenhan Xiong and Wenyin Fu and Whitney Meers and Xavier Martinet and Xiaodong Wang and Xiaofang Wang and Xiaoqing Ellen Tan and Xide Xia and Xinfeng Xie and Xuchao Jia and Xuewei Wang and Yaelle Goldschlag and Yashesh Gaur and Yasmine Babaei and Yi Wen and Yiwen Song and Yuchen Zhang and Yue Li and Yuning Mao and Zacharie Delpierre Coudert and Zheng Yan and Zhengxing Chen and Zoe Papakipos and Aaditya Singh and Aayushi Srivastava and Abha Jain and Adam Kelsey and Adam Shajnfeld and Adithya Gangidi and Adolfo Victoria and Ahuva Goldstand and Ajay Menon and Ajay Sharma and Alex Boesenberg and Alexei Baevski and Allie Feinstein and Amanda Kallet and Amit Sangani and Amos Teo and Anam Yunus and Andrei Lupu and Andres Alvarado and Andrew Caples and Andrew Gu and Andrew Ho and Andrew Poulton and Andrew Ryan and Ankit Ramchandani and Annie Dong and Annie Franco and Anuj Goyal and Aparajita Saraf and Arkabandhu Chowdhury and Ashley Gabriel and Ashwin Bharambe and Assaf Eisenman and Azadeh Yazdan and Beau James and Ben Maurer and Benjamin Leonhardi and Bernie Huang and Beth Loyd and Beto De Paola and Bhargavi Paranjape and Bing Liu and Bo Wu and Boyu Ni and Braden Hancock and Bram Wasti and Brandon Spence and Brani Stojkovic and Brian Gamido and Britt Montalvo and Carl Parker and Carly Burton and Catalina Mejia and Ce Liu and Changhan Wang and Changkyu Kim and Chao Zhou and Chester Hu and Ching-Hsiang Chu and Chris Cai and Chris Tindal and Christoph Feichtenhofer and Cynthia Gao and Damon Civin and Dana Beaty and Daniel Kreymer and Daniel Li and David Adkins and David Xu and Davide Testuggine and Delia David and Devi Parikh and Diana Liskovich and Didem Foss and Dingkang Wang and Duc Le and Dustin Holland and Edward Dowling and Eissa Jamil and Elaine Montgomery and Eleonora Presani and Emily Hahn and Emily Wood and Eric-Tuan Le and Erik Brinkman and Esteban Arcaute and Evan Dunbar and Evan Smothers and Fei Sun and Felix Kreuk and Feng Tian and Filippos Kokkinos and Firat Ozgenel and Francesco Caggioni and Frank Kanayet and Frank Seide and Gabriela Medina Florez and Gabriella Schwarz and Gada Badeer and Georgia Swee and Gil Halpern and Grant Herman and Grigory Sizov and Guangyi and Zhang and Guna Lakshminarayanan and Hakan Inan and Hamid Shojanazeri and Han Zou and Hannah Wang and Hanwen Zha and Haroun Habeeb and Harrison Rudolph and Helen Suk and Henry Aspegren and Hunter Goldman and Hongyuan Zhan and Ibrahim Damlaj and Igor Molybog and Igor Tufanov and Ilias Leontiadis and Irina-Elena Veliche and Itai Gat and Jake Weissman and James Geboski and James Kohli and Janice Lam and Japhet Asher and Jean-Baptiste Gaya and Jeff Marcus and Jeff Tang and Jennifer Chan and Jenny Zhen and Jeremy Reizenstein and Jeremy Teboul and Jessica Zhong and Jian Jin and Jingyi Yang and Joe Cummings and Jon Carvill and Jon Shepard and Jonathan McPhie and Jonathan Torres and Josh Ginsburg and Junjie Wang and Kai Wu and Kam Hou U and Karan Saxena and Kartikay Khandelwal and Katayoun Zand and Kathy Matosich and Kaushik Veeraraghavan and Kelly Michelena and Keqian Li and Kiran Jagadeesh and Kun Huang and Kunal Chawla and Kyle Huang and Lailin Chen and Lakshya Garg and Lavender A and Leandro Silva and Lee Bell and Lei Zhang and Liangpeng Guo and Licheng Yu and Liron Moshkovich and Luca Wehrstedt and Madian Khabsa and Manav Avalani and Manish Bhatt and Martynas Mankus and Matan Hasson and Matthew Lennie and Matthias Reso and Maxim Groshev and Maxim Naumov and Maya Lathi and Meghan Keneally and Miao Liu and Michael L. Seltzer and Michal Valko and Michelle Restrepo and Mihir Patel and Mik Vyatskov and Mikayel Samvelyan and Mike Clark and Mike Macey and Mike Wang and Miquel Jubert Hermoso and Mo Metanat and Mohammad Rastegari and Munish Bansal and Nandhini Santhanam and Natascha Parks and Natasha White and Navyata Bawa and Nayan Singhal and Nick Egebo and Nicolas Usunier and Nikhil Mehta and Nikolay Pavlovich Laptev and Ning Dong and Norman Cheng and Oleg Chernoguz and Olivia Hart and Omkar Salpekar and Ozlem Kalinli and Parkin Kent and Parth Parekh and Paul Saab and Pavan Balaji and Pedro Rittner and Philip Bontrager and Pierre Roux and Piotr Dollar and Polina Zvyagina and Prashant Ratanchandani and Pritish Yuvraj and Qian Liang and Rachad Alao and Rachel Rodriguez and Rafi Ayub and Raghotham Murthy and Raghu Nayani and Rahul Mitra and Rangaprabhu Parthasarathy and Raymond Li and Rebekkah Hogan and Robin Battey and Rocky Wang and Russ Howes and Ruty Rinott and Sachin Mehta and Sachin Siby and Sai Jayesh Bondu and Samyak Datta and Sara Chugh and Sara Hunt and Sargun Dhillon and Sasha Sidorov and Satadru Pan and Saurabh Mahajan and Saurabh Verma and Seiji Yamamoto and Sharadh Ramaswamy and Shaun Lindsay and Shaun Lindsay and Sheng Feng and Shenghao Lin and Shengxin Cindy Zha and Shishir Patil and Shiva Shankar and Shuqiang Zhang and Shuqiang Zhang and Sinong Wang and Sneha Agarwal and Soji Sajuyigbe and Soumith Chintala and Stephanie Max and Stephen Chen and Steve Kehoe and Steve Satterfield and Sudarshan Govindaprasad and Sumit Gupta and Summer Deng and Sungmin Cho and Sunny Virk and Suraj Subramanian and Sy Choudhury and Sydney Goldman and Tal Remez and Tamar Glaser and Tamara Best and Thilo Koehler and Thomas Robinson and Tianhe Li and Tianjun Zhang and Tim Matthews and Timothy Chou and Tzook Shaked and Varun Vontimitta and Victoria Ajayi and Victoria Montanez and Vijai Mohan and Vinay Satish Kumar and Vishal Mangla and Vlad Ionescu and Vlad Poenaru and Vlad Tiberiu Mihailescu and Vladimir Ivanov and Wei Li and Wenchen Wang and Wenwen Jiang and Wes Bouaziz and Will Constable and Xiaocheng Tang and Xiaojian Wu and Xiaolan Wang and Xilun Wu and Xinbo Gao and Yaniv Kleinman and Yanjun Chen and Ye Hu and Ye Jia and Ye Qi and Yenda Li and Yilin Zhang and Ying Zhang and Yossi Adi and Youngjin Nam and Yu and Wang and Yu Zhao and Yuchen Hao and Yundi Qian and Yunlu Li and Yuzi He and Zach Rait and Zachary DeVito and Zef Rosnbrick and Zhaoduo Wen and Zhenyu Yang and Zhiwei Zhao and Zhiyu Ma},
  eprint        = {2407.21783},
  primaryclass  = {cs.AI},
  title         = {The Llama 3 Herd of Models},
  url           = {https://arxiv.org/abs/2407.21783},
  year          = {2024}
}

@article{Almeida_2025,
  author    = {Almeida, Jorge},
  doi       = {10.1142/s1793557125400078},
  issn      = {1793-7183},
  journal   = {Asian-European Journal of Mathematics},
  month     = may,
  publisher = {World Scientific Pub Co Pte Ltd},
  title     = {Pseudovarieties of semigroups},
  url       = {http://dx.doi.org/10.1142/S1793557125400078},
  year      = {2025}
}

@misc{kazemnejad2023impactpositionalencodinglength,
  archiveprefix = {arXiv},
  author        = {Amirhossein Kazemnejad and Inkit Padhi and Karthikeyan Natesan Ramamurthy and Payel Das and Siva Reddy},
  eprint        = {2305.19466},
  primaryclass  = {cs.CL},
  title         = {The Impact of Positional Encoding on Length Generalization in Transformers},
  url           = {https://arxiv.org/abs/2305.19466},
  year          = {2023}
}

@misc{yang2025qwen3technicalreport,
  archiveprefix = {arXiv},
  author        = {An Yang and Anfeng Li and Baosong Yang and Beichen Zhang and Binyuan Hui and Bo Zheng and Bowen Yu and Chang Gao and Chengen Huang and Chenxu Lv and Chujie Zheng and Dayiheng Liu and Fan Zhou and Fei Huang and Feng Hu and Hao Ge and Haoran Wei and Huan Lin and Jialong Tang and Jian Yang and Jianhong Tu and Jianwei Zhang and Jianxin Yang and Jiaxi Yang and Jing Zhou and Jingren Zhou and Junyang Lin and Kai Dang and Keqin Bao and Kexin Yang and Le Yu and Lianghao Deng and Mei Li and Mingfeng Xue and Mingze Li and Pei Zhang and Peng Wang and Qin Zhu and Rui Men and Ruize Gao and Shixuan Liu and Shuang Luo and Tianhao Li and Tianyi Tang and Wenbiao Yin and Xingzhang Ren and Xinyu Wang and Xinyu Zhang and Xuancheng Ren and Yang Fan and Yang Su and Yichang Zhang and Yinger Zhang and Yu Wan and Yuqiong Liu and Zekun Wang and Zeyu Cui and Zhenru Zhang and Zhipeng Zhou and Zihan Qiu},
  eprint        = {2505.09388},
  primaryclass  = {cs.CL},
  title         = {Qwen3 Technical Report},
  url           = {https://arxiv.org/abs/2505.09388},
  year          = {2025}
}

@inproceedings{yang2024counting,
  author    = {Andy Yang and David Chiang},
  booktitle = {First Conference on Language Modeling},
  title     = {Counting Like Transformers: Compiling Temporal Counting Logic Into Softmax Transformers},
  url       = {https://openreview.net/forum?id=FmhPg4UJ9K},
  year      = {2024}
}

@misc{yang2024maskedhardattentiontransformersrecognize,
  archiveprefix = {arXiv},
  author        = {Andy Yang and David Chiang and Dana Angluin},
  eprint        = {2310.13897},
  primaryclass  = {cs.FL},
  title         = {Masked Hard-Attention Transformers Recognize Exactly the Star-Free Languages},
  url           = {https://arxiv.org/abs/2310.13897},
  year          = {2024}
}

@misc{yang2025kneedeepcrasptransformerdepth,
  archiveprefix = {arXiv},
  author        = {Andy Yang and Michaël Cadilhac and David Chiang},
  eprint        = {2506.16055},
  primaryclass  = {cs.CL},
  title         = {Knee-Deep in C-RASP: A Transformer Depth Hierarchy},
  url           = {https://arxiv.org/abs/2506.16055},
  year          = {2025}
}

@inproceedings{svete2026revisiting,
  author    = {Anej Svete and William Merrill and Ryan Cotterell and Ashish Sabharwal},
  booktitle = {Forty-third International Conference on Machine Learning},
  title     = {Revisiting Padded Transformer Expressivity: Which Architectural Choices Matter and Which Don't},
  url       = {https://openreview.net/forum?id=nBuL6HywFX},
  year      = {2026}
}

@misc{vaswani2023attentionneed,
  archiveprefix = {arXiv},
  author        = {Ashish Vaswani and Noam Shazeer and Niki Parmar and Jakob Uszkoreit and Llion Jones and Aidan N. Gomez and Lukasz Kaiser and Illia Polosukhin},
  eprint        = {1706.03762},
  primaryclass  = {cs.CL},
  title         = {Attention Is All You Need},
  url           = {https://arxiv.org/abs/1706.03762},
  year          = {2017}
}

@misc{liu2023exposingattentionglitchesflipflop,
  archiveprefix = {arXiv},
  author        = {Bingbin Liu and Jordan T. Ash and Surbhi Goel and Akshay Krishnamurthy and Cyril Zhang},
  eprint        = {2306.00946},
  primaryclass  = {cs.LG},
  title         = {Exposing Attention Glitches with Flip-Flop Language Modeling},
  url           = {https://arxiv.org/abs/2306.00946},
  year          = {2023}
}

@inproceedings{liu2023transformers,
  author    = {Bingbin Liu and Jordan T. Ash and Surbhi Goel and Akshay Krishnamurthy and Cyril Zhang},
  booktitle = {The Eleventh International Conference on Learning Representations },
  title     = {Transformers Learn Shortcuts to Automata},
  url       = {https://openreview.net/forum?id=De4FYqjFueZ},
  year      = {2023}
}

@inproceedings{dartois_et_al:LIPIcs.STACS.2013.329,
  address   = {Dagstuhl, Germany},
  author    = {Dartois, Luc and Paperman, Charles},
  booktitle = {30th International Symposium on Theoretical Aspects of Computer Science (STACS 2013)},
  doi       = {10.4230/LIPIcs.STACS.2013.329},
  editor    = {Portier, Natacha and Wilke, Thomas},
  isbn      = {978-3-939897-50-7},
  issn      = {1868-8969},
  pages     = {329--340},
  publisher = {Schloss Dagstuhl -- Leibniz-Zentrum f{\"u}r Informatik},
  series    = {Leibniz International Proceedings in Informatics (LIPIcs)},
  title     = {{Two-variable first order logic with modular predicates over words}},
  url       = {https://drops.dagstuhl.de/entities/document/10.4230/LIPIcs.STACS.2013.329},
  urn       = {urn:nbn:de:0030-drops-39450},
  volume    = {20},
  year      = {2013}
}

@misc{chiang2023tighterboundsexpressivitytransformer,
  archiveprefix = {arXiv},
  author        = {David Chiang and Peter Cholak and Anand Pillay},
  eprint        = {2301.10743},
  primaryclass  = {cs.LG},
  title         = {Tighter Bounds on the Expressivity of Transformer Encoders},
  url           = {https://arxiv.org/abs/2301.10743},
  year          = {2023}
}

@misc{deepseekai2026deepseekv4highlyefficientmilliontoken,
  archiveprefix = {arXiv},
  author        = {DeepSeek-AI and Anyi Xu and Bangcai Lin and Bing Xue and Bingxuan Wang and Bingzheng Xu and Bochao Wu and Bowei Zhang and Chaofan Lin and Chen Dong and Chenchen Ling and Chengda Lu and Chenggang Zhao and Chengqi Deng and Chengyu Hou and Chenhao Xu and Chenze Shao and Chong Ruan and Conner Sun and Damai Dai and Daya Guo and Dejian Yang and Deli Chen and Donghao Li and Dongjie Ji and Erhang Li and Fang Wei and Fangyun Lin and Fangzhou Yuan and Feiyu Xia and Fucong Dai and Guangbo Hao and Guanting Chen and Guoai Cao and Guolai Meng and Guowei Li and Han Yu and Han Zhang and Hanwei Xu and Hao Li and Haofen Liang and Haoling Zhang and Haoming Luo and Haoran Wei and Haotian Yuan and Haowei Zhang and Haowen Luo and Haoyu Chen and Haozhe Ji and Hengqing Zhang and Honghui Ding and Hongxuan Tang and Huanqi Cao and Huazuo Gao and Hui Qu and Hui Zeng and J Yang and JQ Zhu and Jia Luo and Jia Song and Jia Yu and Jialiang Huang and Jialu Cai and Jian Liang and Jiangting Zhou and Jiasheng Ye and Jiashi Li and Jiaxin Xu and Jiewen Hu and Jieyu Yang and Jin Chen and Jin Yan and Jingchang Chen and Jingli Zhou and Jingting Xiang and Jingyang Yuan and Jingyuan Cheng and Jingzi Zhou and Jinhua Zhu and Jiping Yu and Joseph Sun and Jun Ran and Junguang Jiang and Junjie Qiu and Junlong Li and Junmin Zheng and Junxiao Song and Kai Dong and Kaige Gao and Kang Guan and Kexing Zhou and Kezhao Huang and Kuai Yu and Lean Wang and Lecong Zhang and Lei Wang and Leyi Xia and Li Zhang and Liang Zhao and Lihua Guo and Lingxiao Luo and Linwang Ma and Linyan Zhu and Litong Wang and Liyu Cai and Liyue Zhang and Longhao Chen and MS Di and MY Xu and Max Mei and Miaojun Wang and Mingchuan Zhang and Minghua Zhang and Minghui Tang and Mingming Li and Mingxu Zhou and Minmin Han and Ning Wang and Panpan Huang and Panpan Wang and Peixin Cong and Peiyi Wang and Peng Zhang and Qiancheng Wang and Qihao Zhu and Qingyang Li and Qinyu Chen and Qiushi Du and Qiwei Jiang and Rui Tian and Ruifan Xu and Ruijie Lu and Ruiling Xu and Ruiqi Ge and Ruisong Zhang and Ruizhe Pan and Runji Wang and Runqian Chen and Runqiu Yin and Runxin Xu and Ruomeng Shen and Ruoyu Zhang and Ruyi Chen and SH Liu and Shanghao Lu and Shangmian Sun and Shangyan Zhou and Shanhuang Chen and Shaofei Cai and Shaoheng Nie and Shaoqing Wu and Shaoyuan Chen and Shengding Hu and Shengyu Liu and Shiqiang Hu and Shirong Ma and Shiyu Wang and Shuiping Yu and Shunfeng Zhou and Shuting Pan and Shuying Yu and Songyang Zhou and Tao Ni and Tao Yun and Tian Jin and Tian Pei and Tian Ye and Tianle Lin and Tianran Ji and Tianyi Cui and Tianyuan Yue and Tingting Yu and Tun Wang and W Zhang and WL Xiao and Wangding Zeng and Wei An and Weilin Zhao and Wen Liu and Wenfeng Liang and Wenjie Pang and Wenjing Luo and Wenjing Yao and Wenjun Gao and Wenkai Yang and Wenlve Huang and Wenqing Hou and Wentao Zhang and Wenting Ma and Xi Gao and Xiang He and Xiangwen Wang and Xianzu Wang and Xiao Bi and Xiaodong Liu and Xiaohan Wang and Xiaokang Chen and Xiaokang Zhang and Xiaotao Nie and Xiaowen Sun and Xiaoxiang Wang and Xin Cheng and Xin Liu and Xin Xie and Xingchao Liu and Xingchen Liu and Xingkai Yu and Xingyou Li and Xinyu Yang and Xinyu Zhang and Xu Chen and Xuanyu Wang and Xuecheng Su and Xueyin Chen and Xuheng Lin and Xuwei Fu and YC Yan and YQ Wang and YW Ma and Yanfeng Luo and Yang Zhang and Yanhong Xu and Yanru Ma and Yanwen Huang and Yao Li and Yao Li and Yao Xu and Yao Zhao and Yaofeng Sun and Yaohui Wang and Yi Qian and Yi Shao and Yi Yu and Yichao Zhang and Yifan Ding and Yifan Shi and Yijia Wu and Yiliang Xiong and Yiling Ma and Ying He and Ying Tang and Ying Zhou and Yingjia Luo and Yinmin Zhong and Yishi Piao and Yisong Wang and Yixiang Zhang and Yixiao Chen and Yixuan Tan and Yixuan Wei and Yiyang Ma and Yiyuan Liu and Yonglun Yang and Yongqiang Guo and Yongtong Wu and Yu Wu and YuKun Li and Yuan Cheng and Yuan Ou and Yuanfan Xu and Yuanhao Li and Yuduan Wang and Yuehan Yang and Yuer Xu and Yuhan Wu and Yuhao Meng and Yuheng Zou and Yukun Zha and Yunfan Xiong and Yupeng Chen and Yuping Lin and Yuqian Cao and Yuqian Wang and Yushun Zhang and Yuting Yan and Yutong Lin and Yuxian Gu and Yuxiang Luo and Yuxiang You and Yuxuan Liu and Yuxuan Zhou and Yuyang Zhou and Yuzhen Huang and ZF Wu and Zehao Wang and Zehua Zhao and Zehui Ren and Zekai Zhang and Zhangli Sha and Zhe Fu and Zhe Ju and Zhean Xu and Zhenda Xie and Zhengyan Zhang and Zheren Gao and Zhewen Hao and Zhibin Gou and Zhicheng Ma and Zhigang Yan and Zhihong Shao and Zhixian Huang and Zhixuan Chen and Zhiyu Wu and Zhizhou Ren and Zhongyu Wu and Zhuoshu Li and Zhuping Zhang and Zian Xu and Zihao Wang and Zihua Qu and Zihui Gu and Zijia Zhu and Zilin Li and Zipeng Zhang and Ziwei Xie and Ziyi Gao and Ziyi Wan and Zizheng Pan and Zongqing Yao},
  eprint        = {2606.19348},
  primaryclass  = {cs.CL},
  title         = {DeepSeek-V4: Towards Highly Efficient Million-Token Context Intelligence},
  url           = {https://arxiv.org/abs/2606.19348},
  year          = {2026}
}

@misc{barbero2025roundroundgomakes,
  archiveprefix = {arXiv},
  author        = {Federico Barbero and Alex Vitvitskyi and Christos Perivolaropoulos and Razvan Pascanu and Petar Veličković},
  eprint        = {2410.06205},
  primaryclass  = {cs.CL},
  title         = {Round and Round We Go! What makes Rotary Positional Encodings useful?},
  url           = {https://arxiv.org/abs/2410.06205},
  year          = {2025}
}

@misc{barbero2025llmsattendtoken,
  archiveprefix = {arXiv},
  author        = {Federico Barbero and Álvaro Arroyo and Xiangming Gu and Christos Perivolaropoulos and Michael Bronstein and Petar Veličković and Razvan Pascanu},
  eprint        = {2504.02732},
  primaryclass  = {cs.CL},
  title         = {Why do LLMs attend to the first token?},
  url           = {https://arxiv.org/abs/2504.02732},
  year          = {2025}
}

@inproceedings{gabbay,
  address   = {New York, NY, USA},
  author    = {Gabbay, Dov and Pnueli, Amir and Shelah, Saharon and Stavi, Jonathan},
  booktitle = {Proceedings of the 7th ACM SIGPLAN-SIGACT Symposium on Principles of Programming Languages},
  doi       = {10.1145/567446.567462},
  isbn      = {0897910117},
  location  = {Las Vegas, Nevada},
  numpages  = {11},
  pages     = {163–173},
  publisher = {Association for Computing Machinery},
  series    = {POPL '80},
  title     = {On the temporal analysis of fairness},
  url       = {https://doi.org/10.1145/567446.567462},
  year      = {1980}
}

@misc{gemmateam2024gemmaopenmodelsbased,
  archiveprefix = {arXiv},
  author        = {Gemma Team and Thomas Mesnard and Cassidy Hardin and Robert Dadashi and Surya Bhupatiraju and Shreya Pathak and Laurent Sifre and Morgane Rivière and Mihir Sanjay Kale and Juliette Love and Pouya Tafti and Léonard Hussenot and Pier Giuseppe Sessa and Aakanksha Chowdhery and Adam Roberts and Aditya Barua and Alex Botev and Alex Castro-Ros and Ambrose Slone and Amélie Héliou and Andrea Tacchetti and Anna Bulanova and Antonia Paterson and Beth Tsai and Bobak Shahriari and Charline Le Lan and Christopher A. Choquette-Choo and Clément Crepy and Daniel Cer and Daphne Ippolito and David Reid and Elena Buchatskaya and Eric Ni and Eric Noland and Geng Yan and George Tucker and George-Christian Muraru and Grigory Rozhdestvenskiy and Henryk Michalewski and Ian Tenney and Ivan Grishchenko and Jacob Austin and James Keeling and Jane Labanowski and Jean-Baptiste Lespiau and Jeff Stanway and Jenny Brennan and Jeremy Chen and Johan Ferret and Justin Chiu and Justin Mao-Jones and Katherine Lee and Kathy Yu and Katie Millican and Lars Lowe Sjoesund and Lisa Lee and Lucas Dixon and Machel Reid and Maciej Mikuła and Mateo Wirth and Michael Sharman and Nikolai Chinaev and Nithum Thain and Olivier Bachem and Oscar Chang and Oscar Wahltinez and Paige Bailey and Paul Michel and Petko Yotov and Rahma Chaabouni and Ramona Comanescu and Reena Jana and Rohan Anil and Ross McIlroy and Ruibo Liu and Ryan Mullins and Samuel L Smith and Sebastian Borgeaud and Sertan Girgin and Sholto Douglas and Shree Pandya and Siamak Shakeri and Soham De and Ted Klimenko and Tom Hennigan and Vlad Feinberg and Wojciech Stokowiec and Yu-hui Chen and Zafarali Ahmed and Zhitao Gong and Tris Warkentin and Ludovic Peran and Minh Giang and Clément Farabet and Oriol Vinyals and Jeff Dean and Koray Kavukcuoglu and Demis Hassabis and Zoubin Ghahramani and Douglas Eck and Joelle Barral and Fernando Pereira and Eli Collins and Armand Joulin and Noah Fiedel and Evan Senter and Alek Andreev and Kathleen Kenealy},
  eprint        = {2403.08295},
  primaryclass  = {cs.CL},
  title         = {Gemma: Open Models Based on Gemini Research and Technology},
  url           = {https://arxiv.org/abs/2403.08295},
  year          = {2024}
}

@misc{delétang2023neuralnetworkschomskyhierarchy,
  archiveprefix = {arXiv},
  author        = {Grégoire Delétang and Anian Ruoss and Jordi Grau-Moya and Tim Genewein and Li Kevin Wenliang and Elliot Catt and Chris Cundy and Marcus Hutter and Shane Legg and Joel Veness and Pedro A. Ortega},
  eprint        = {2207.02098},
  primaryclass  = {cs.LG},
  title         = {Neural Networks and the Chomsky Hierarchy},
  url           = {https://arxiv.org/abs/2207.02098},
  year          = {2023}
}

@misc{xiao2024efficientstreaminglanguagemodels,
  archiveprefix = {arXiv},
  author        = {Guangxuan Xiao and Yuandong Tian and Beidi Chen and Song Han and Mike Lewis},
  eprint        = {2309.17453},
  primaryclass  = {cs.CL},
  title         = {Efficient Streaming Language Models with Attention Sinks},
  url           = {https://arxiv.org/abs/2309.17453},
  year          = {2024}
}

@phdthesis{Kamp1968-KAMTLA,
  author = {Hans Kamp},
  school = {{UCLA}},
  title  = {Tense Logic and the Theory of Linear Order},
  url    = {https://www.proquest.com/openview/408039eb4ed228dc4cba3fe7e1774163/1?pq-origsite=gscholar&cbl=18750&diss=y},
  year   = {1968}
}

@article{hao-etal-2022-formal,
  address   = {Cambridge, MA},
  author    = {Hao, Yiding  and
               Angluin, Dana  and
               Frank, Robert},
  doi       = {10.1162/tacl_a_00490},
  editor    = {Roark, Brian  and
               Nenkova, Ani},
  journal   = {Transactions of the Association for Computational Linguistics},
  pages     = {800--810},
  publisher = {MIT Press},
  title     = {Formal Language Recognition by Hard Attention Transformers: Perspectives from Circuit Complexity},
  url       = {https://aclanthology.org/2022.tacl-1.46/},
  volume    = {10},
  year      = {2022}
}

@article{BRZOZOWSKI198032,
  author  = {Janusz Antoni Brzozowski and Faith Ellen},
  doi     = {https://doi.org/10.1016/0022-0000(80)90003-3},
  issn    = {0022-0000},
  journal = {Journal of Computer and System Sciences},
  number  = {1},
  pages   = {32-49},
  title   = {Languages of {R}-trivial monoids},
  url     = {https://www.sciencedirect.com/science/article/pii/0022000080900033},
  volume  = {20},
  year    = {1980}
}

@inproceedings{jerad-etal-2025-unique,
  address   = {Vienna, Austria},
  author    = {Jerad, Selim  and
               Svete, Anej  and
               Li, Jiaoda  and
               Cotterell, Ryan},
  booktitle = {Proceedings of the 63rd Annual Meeting of the Association for Computational Linguistics (Volume 2: Short Papers)},
  doi       = {10.18653/v1/2025.acl-short.76},
  editor    = {Che, Wanxiang  and
               Nabende, Joyce  and
               Shutova, Ekaterina  and
               Pilehvar, Mohammad Taher},
  isbn      = {979-8-89176-252-7},
  month     = jul,
  pages     = {977--996},
  publisher = {Association for Computational Linguistics},
  title     = {Unique Hard Attention: A Tale of Two Sides},
  url       = {https://aclanthology.org/2025.acl-short.76/},
  year      = {2025}
}

@misc{su2023roformerenhancedtransformerrotary,
  archiveprefix = {arXiv},
  author        = {Jianlin Su and Yu Lu and Shengfeng Pan and Ahmed Murtadha and Bo Wen and Yunfeng Liu},
  eprint        = {2104.09864},
  primaryclass  = {cs.CL},
  title         = {RoFormer: Enhanced Transformer with Rotary Position Embedding},
  url           = {https://arxiv.org/abs/2104.09864},
  year          = {2023}
}

@misc{li2026characterizingexpressivitylocalattention,
  archiveprefix = {arXiv},
  author        = {Jiaoda Li and Ryan Cotterell},
  eprint        = {2605.00768},
  primaryclass  = {cs.CL},
  title         = {Characterizing the Expressivity of Local Attention in Transformers},
  url           = {https://arxiv.org/abs/2605.00768},
  year          = {2026}
}

@misc{li2025characterizingexpressivityfixedprecisiontransformer,
  archiveprefix = {arXiv},
  author        = {Jiaoda Li and Ryan Cotterell},
  eprint        = {2505.23623},
  primaryclass  = {cs.CL},
  title         = {Characterizing the Expressivity of Fixed-Precision Transformer Language Models},
  url           = {https://arxiv.org/abs/2505.23623},
  year          = {2025}
}

@misc{bai2023qwentechnicalreport,
  archiveprefix = {arXiv},
  author        = {Jinze Bai and Shuai Bai and Yunfei Chu and Zeyu Cui and Kai Dang and Xiaodong Deng and Yang Fan and Wenbin Ge and Yu Han and Fei Huang and Binyuan Hui and Luo Ji and Mei Li and Junyang Lin and Runji Lin and Dayiheng Liu and Gao Liu and Chengqiang Lu and Keming Lu and Jianxin Ma and Rui Men and Xingzhang Ren and Xuancheng Ren and Chuanqi Tan and Sinan Tan and Jianhong Tu and Peng Wang and Shijie Wang and Wei Wang and Shengguang Wu and Benfeng Xu and Jin Xu and An Yang and Hao Yang and Jian Yang and Shusheng Yang and Yang Yao and Bowen Yu and Hongyi Yuan and Zheng Yuan and Jianwei Zhang and Xingxuan Zhang and Yichang Zhang and Zhenru Zhang and Chang Zhou and Jingren Zhou and Xiaohuan Zhou and Tianhang Zhu},
  eprint        = {2309.16609},
  primaryclass  = {cs.CL},
  title         = {Qwen Technical Report},
  url           = {https://arxiv.org/abs/2309.16609},
  year          = {2023}
}

@misc{dartois2015addingmodularpredicatesfirstorder,
  archiveprefix = {arXiv},
  author        = {Luc Dartois and Charles Paperman},
  eprint        = {1401.6576},
  primaryclass  = {cs.LO},
  title         = {Adding modular predicates to first-order fragments},
  url           = {https://arxiv.org/abs/1401.6576},
  year          = {2015}
}

@book{mcnaughton1971counter,
  author    = {McNaughton, Robert and Papert, Seymour},
  isbn      = {9780262130769},
  lccn      = {71153294},
  publisher = {M.I.T. Press},
  series    = {M.I.T. Press research monographs},
  title     = {Counter-Free Automata},
  url       = {https://mitpress.mit.edu/9780262130769/counter-free-automata/},
  year      = {1971}
}

@misc{mohri2026rationaltransductors,
  archiveprefix = {arXiv},
  author        = {Mehryar Mohri},
  eprint        = {2602.07599},
  primaryclass  = {cs.LG},
  title         = {Rational Transductors},
  url           = {https://arxiv.org/abs/2602.07599},
  year          = {2026}
}

@misc{khan2026fractionalrotationpotentialinvestigating,
  archiveprefix = {arXiv},
  author        = {Mohammad Aflah Khan and Krishna P. Gummadi and Manish Gupta and Abhilasha Ravichander},
  eprint        = {2603.11611},
  primaryclass  = {cs.LG},
  title         = {Fractional Rotation, Full Potential? Investigating Performance and Convergence of Partial RoPE},
  url           = {https://arxiv.org/abs/2603.11611},
  year          = {2026}
}

@inproceedings{neishi-yoshinaga-2019-relation,
  address   = {Hong Kong, China},
  author    = {Neishi, Masato  and
               Yoshinaga, Naoki},
  booktitle = {Proceedings of the 23rd Conference on Computational Natural Language Learning (CoNLL)},
  doi       = {10.18653/v1/K19-1031},
  editor    = {Bansal, Mohit  and
               Villavicencio, Aline},
  month     = nov,
  pages     = {328--338},
  publisher = {Association for Computational Linguistics},
  title     = {On the Relation between Position Information and Sentence Length in Neural Machine Translation},
  url       = {https://aclanthology.org/K19-1031/},
  year      = {2019}
}

@inproceedings{press2022train,
  author    = {Ofir Press and Noah Smith and Mike Lewis},
  booktitle = {International Conference on Learning Representations},
  title     = {Train Short, Test Long: Attention with Linear Biases Enables Input Length Extrapolation},
  url       = {https://openreview.net/forum?id=R8sQPpGCv0},
  year      = {2022}
}

@inproceedings{rosendahl-etal-2019-analysis,
  address   = {Hong Kong},
  author    = {Rosendahl, Jan  and
               Tran, Viet Anh Khoa  and
               Wang, Weiyue  and
               Ney, Hermann},
  booktitle = {Proceedings of the 16th International Conference on Spoken Language Translation},
  editor    = {Niehues, Jan  and
               Cattoni, Rolando  and
               St{\"u}ker, Sebastian  and
               Negri, Matteo  and
               Turchi, Marco  and
               Ha, Thanh-Le  and
               Salesky, Elizabeth  and
               Sanabria, Ramon  and
               Barrault, Loic  and
               Specia, Lucia  and
               Federico, Marcello},
  month     = nov # { 2-3},
  publisher = {Association for Computational Linguistics},
  title     = {Analysis of Positional Encodings for Neural Machine Translation},
  url       = {https://aclanthology.org/2019.iwslt-1.20/},
  year      = {2019}
}

@misc{huo2026periodicropeinfinitecontext,
  archiveprefix = {arXiv},
  author        = {Simin Huo},
  eprint        = {2605.27980},
  primaryclass  = {cs.CL},
  title         = {Periodic RoPE for Infinite Context LLMs},
  url           = {https://arxiv.org/abs/2605.27980},
  year          = {2026}
}

@inproceedings{svete-etal-2024-transformers,
  address   = {Miami, Florida, USA},
  author    = {Svete, Anej  and
               Borenstein, Nadav  and
               Zhou, Mike  and
               Augenstein, Isabelle  and
               Cotterell, Ryan},
  booktitle = {Proceedings of the 2024 Conference on Empirical Methods in Natural Language Processing},
  doi       = {10.18653/v1/2024.emnlp-main.550},
  editor    = {Al-Onaizan, Yaser  and
               Bansal, Mohit  and
               Chen, Yun-Nung},
  month     = nov,
  pages     = {9851--9867},
  publisher = {Association for Computational Linguistics},
  title     = {Can Transformers Learn $n$-gram Language Models?},
  url       = {https://aclanthology.org/2024.emnlp-main.550/},
  year      = {2024}
}

@inproceedings{svete-cotterell-2024-transformers,
  address   = {Mexico City, Mexico},
  author    = {Svete, Anej  and
               Cotterell, Ryan},
  booktitle = {Proceedings of the 2024 Conference of the North American Chapter of the Association for Computational Linguistics: Human Language Technologies (Volume 1: Long Papers)},
  doi       = {10.18653/v1/2024.naacl-long.381},
  editor    = {Duh, Kevin  and
               Gomez, Helena  and
               Bethard, Steven},
  month     = jun,
  pages     = {6845--6881},
  publisher = {Association for Computational Linguistics},
  title     = {Transformers Can Represent $n$-gram Language Models},
  url       = {https://aclanthology.org/2024.naacl-long.381/},
  year      = {2024}
}

@misc{olmo2025olmo3,
  archiveprefix = {arXiv},
  author        = {Team Olmo and : and Allyson Ettinger and Amanda Bertsch and Bailey Kuehl and David Graham and David Heineman and Dirk Groeneveld and Faeze Brahman and Finbarr Timbers and Hamish Ivison and Jacob Morrison and Jake Poznanski and Kyle Lo and Luca Soldaini and Matt Jordan and Mayee Chen and Michael Noukhovitch and Nathan Lambert and Pete Walsh and Pradeep Dasigi and Robert Berry and Saumya Malik and Saurabh Shah and Scott Geng and Shane Arora and Shashank Gupta and Taira Anderson and Teng Xiao and Tyler Murray and Tyler Romero and Victoria Graf and Akari Asai and Akshita Bhagia and Alexander Wettig and Alisa Liu and Aman Rangapur and Chloe Anastasiades and Costa Huang and Dustin Schwenk and Harsh Trivedi and Ian Magnusson and Jaron Lochner and Jiacheng Liu and Lester James V. Miranda and Maarten Sap and Malia Morgan and Michael Schmitz and Michal Guerquin and Michael Wilson and Regan Huff and Ronan Le Bras and Rui Xin and Rulin Shao and Sam Skjonsberg and Shannon Zejiang Shen and Shuyue Stella Li and Tucker Wilde and Valentina Pyatkin and Will Merrill and Yapei Chang and Yuling Gu and Zhiyuan Zeng and Ashish Sabharwal and Luke Zettlemoyer and Pang Wei Koh and Ali Farhadi and Noah A. Smith and Hannaneh Hajishirzi},
  eprint        = {2512.13961},
  primaryclass  = {cs.CL},
  title         = {Olmo 3},
  url           = {https://arxiv.org/abs/2512.13961},
  year          = {2025}
}

@book{walters1982introduction,
  author    = {Walters, P.},
  isbn      = {9783540905998},
  lccn      = {81009319},
  publisher = {Springer-Verlag},
  series    = {Graduate texts in mathematics},
  title     = {An Introduction to Ergodic Theory},
  url       = {https://books.google.com/books?id=GCH_wAEACAAJ},
  year      = {1982}
}

@inproceedings{wang-etal-2024-resonance,
  address   = {Bangkok, Thailand},
  author    = {Wang, Suyuchen  and
               Kobyzev, Ivan  and
               Lu, Peng  and
               Rezagholizadeh, Mehdi  and
               Liu, Bang},
  booktitle = {Findings of the Association for Computational Linguistics: ACL 2024},
  doi       = {10.18653/v1/2024.findings-acl.32},
  editor    = {Ku, Lun-Wei  and
               Martins, Andre  and
               Srikumar, Vivek},
  month     = aug,
  pages     = {586--598},
  publisher = {Association for Computational Linguistics},
  title     = {Resonance {R}o{PE}: Improving Context Length Generalization of Large Language Models},
  url       = {https://aclanthology.org/2024.findings-acl.32/},
  year      = {2024}
}

@misc{merrill2026linearrnnsparallelizable,
  archiveprefix = {arXiv},
  author        = {William Merrill and Hongjian Jiang and Yanhong Li and Anthony Lin and Ashish Sabharwal},
  eprint        = {2603.03612},
  primaryclass  = {cs.LG},
  title         = {Why Are Linear RNNs More Parallelizable?},
  url           = {https://arxiv.org/abs/2603.03612},
  year          = {2026}
}

@misc{merrill2026olmohybridtheorypractice,
  archiveprefix = {arXiv},
  author        = {William Merrill and Yanhong Li and Tyler Romero and Anej Svete and Caia Costello and Pradeep Dasigi and Dirk Groeneveld and David Heineman and Bailey Kuehl and Nathan Lambert and Chuan Li and Kyle Lo and Saumya Malik and DJ Matusz and Benjamin Minixhofer and Jacob Morrison and Luca Soldaini and Finbarr Timbers and Pete Walsh and Noah A. Smith and Hannaneh Hajishirzi and Ashish Sabharwal},
  eprint        = {2604.03444},
  primaryclass  = {cs.LG},
  title         = {Olmo Hybrid: From Theory to Practice and Back},
  url           = {https://arxiv.org/abs/2604.03444},
  year          = {2026}
}

@misc{liu2024scalinglawsropebasedextrapolation,
  archiveprefix = {arXiv},
  author        = {Xiaoran Liu and Hang Yan and Shuo Zhang and Chenxin An and Xipeng Qiu and Dahua Lin},
  eprint        = {2310.05209},
  primaryclass  = {cs.CL},
  title         = {Scaling Laws of RoPE-based Extrapolation},
  url           = {https://arxiv.org/abs/2310.05209},
  year          = {2024}
}

@misc{men2024baseropeboundscontext,
  archiveprefix = {arXiv},
  author        = {Xin Men and Mingyu Xu and Bingning Wang and Qingyu Zhang and Hongyu Lin and Xianpei Han and Weipeng Chen},
  eprint        = {2405.14591},
  primaryclass  = {cs.CL},
  title         = {Base of RoPE Bounds Context Length},
  url           = {https://arxiv.org/abs/2405.14591},
  year          = {2024}
}

@inproceedings{yangropetonope,
  author    = {Yang, Bowen and Venkitesh, Bharat and Talupuru, Dwaraknath Gnaneshwar and Lin, Hangyu and Cairuz, David and Blunsom, Phil and Locatelli, Acyr},
  booktitle = {Advances in Neural Information Processing Systems},
  editor    = {D. Belgrave and C. Zhang and H. Lin and R. Pascanu and P. Koniusz and M. Ghassemi and N. Chen},
  pages     = {64133--64157},
  publisher = {Curran Associates, Inc.},
  title     = {Rope to Nope and Back Again: A New Hybrid Attention Strategy},
  url       = {https://proceedings.neurips.cc/paper_files/paper/2025/file/5c9ab393551b7a39b4c02d88fe5e7e69-Paper-Conference.pdf},
  volume    = {38},
  year      = {2025}
}

@article{ZALCSTEIN1972151,
  author  = {Yechezkel Zalcstein},
  doi     = {https://doi.org/10.1016/S0022-0000(72)80020-5},
  issn    = {0022-0000},
  journal = {Journal of Computer and System Sciences},
  number  = {2},
  pages   = {151-167},
  title   = {Locally testable languages},
  url     = {https://www.sciencedirect.com/science/article/pii/S0022000072800205},
  volume  = {6},
  year    = {1972}
}

@misc{gelberg2025extendingcontextpretrainedllms,
  archiveprefix = {arXiv},
  author        = {Yoav Gelberg and Koshi Eguchi and Takuya Akiba and Edoardo Cetin},
  eprint        = {2512.12167},
  primaryclass  = {cs.CL},
  title         = {Extending the Context of Pretrained LLMs by Dropping Their Positional Embeddings},
  url           = {https://arxiv.org/abs/2512.12167},
  year          = {2025}
}

@misc{du2026ropedistinguishespositionstokens,
  archiveprefix = {arXiv},
  author        = {Yufeng Du and Phillip Harris and Minyang Tian and Eliu A Huerta and Srikanth Ronanki and Subendhu Rongali and Aram Galstyan and Hao Peng},
  eprint        = {2605.15514},
  primaryclass  = {cs.CL},
  title         = {RoPE Distinguishes Neither Positions Nor Tokens in Long Contexts, Provably},
  url           = {https://arxiv.org/abs/2605.15514},
  year          = {2026}
}
\bibliographystyle{colm2026_conference}

\clearpage
\appendix
\onecolumn

\section*{Appendix Contents}
\begingroup
\small
\noindent\textbf{\hyperref[app:flt]{A Preliminaries: Formal Language Theory}}\dotfill\pageref{app:flt}\par
\noindent\hspace{1.5em}\hyperref[app:ltl]{A.1 Linear Temporal Logic}\dotfill\pageref{app:ltl}\par
\noindent\hspace{1.5em}\hyperref[app:first-order-logic]{A.2 First-Order Logic}\dotfill\pageref{app:first-order-logic}\par
\noindent\hspace{1.5em}\hyperref[subsec:regex]{A.3 Regular Expressions}\dotfill\pageref{subsec:regex}\par
\noindent\hspace{1.5em}\hyperref[app:automata]{A.4 Automata}\dotfill\pageref{app:automata}\par
\noindent\hspace{1.5em}\hyperref[app:syntactic-monoid]{A.5 The Syntactic Monoid}\dotfill\pageref{app:syntactic-monoid}\par
\noindent\hspace{1.5em}\hyperref[app:formalism-equivalences]{A.6 Equivalences between Formalisms}\dotfill\pageref{app:formalism-equivalences}\par
\noindent\textbf{\hyperref[app:pfo2mod]{B Proofs: Characterizations of $\ptl[\modpred]$}}\dotfill\pageref{app:pfo2mod}\par
\noindent\hspace{1.5em}\hyperref[app:inexpressibility]{B.1 Inexpressibility Results}\dotfill\pageref{app:inexpressibility}\par
\noindent\textbf{\hyperref[app:periodic-proofs]{C Proofs: Periodic RoPE}}\dotfill\pageref{app:periodic-proofs}\par

\noindent\hspace{1.5em}\hyperref[app:schedules-proofs]{C.1 Component-Periodic Schedules}\dotfill\pageref{app:schedules-proofs}\par
\noindent\hspace{1.5em}\hyperref[app:upperperiodic-proofs]{C.2 Upper Bounding Periodic RoPE}\dotfill\pageref{app:upperperiodic-proofs}\par
\noindent\hspace{1.5em}\hyperref[app:lowerperiodic-proofs]{C.3 Lower Bounding Periodic RoPE}\dotfill\pageref{app:lowerperiodic-proofs}\par
\noindent\hspace{1.5em}\hyperref[app:proofsmatsin]{C.4 Characterizing Periodic SiPE}\dotfill\pageref{app:proofsmatsin}\par
\noindent\textbf{\hyperref[app:nonperiodic-proofs]{D Proofs: Non-Periodic RoPE}}\dotfill\pageref{app:nonperiodic-proofs}\par
\noindent\textbf{\hyperref[app:experiments]{E Experiments}}\dotfill\pageref{app:experiments}\par
\noindent\hspace{1.5em}\hyperref[app:experiment-languages]{E.1 Languages}\dotfill\pageref{app:experiment-languages}\par
\noindent\hspace{1.5em}\hyperref[app:experimental-setup]{E.2 Experimental Setup}\dotfill\pageref{app:experimental-setup}\par
\noindent\hspace{1.5em}\hyperref[app:periodic-settings]{E.3 Periodic-Construction Settings}\dotfill\pageref{app:periodic-settings}\par
\noindent\hspace{1.5em}\hyperref[app:nonperiodic-settings]{E.4 Conventional Non-Periodic Settings}\dotfill\pageref{app:nonperiodic-settings}\par
\noindent\hspace{1.5em}\hyperref[app:detailed-results]{E.5 Detailed Results}\dotfill\pageref{app:detailed-results}\par
\endgroup
\vspace{0.75em}

\section{Preliminaries: Formal Language Theory}\label{app:flt}
We introduced formal languages and linear temporal logic in \cref{sec:prelim}.
We now provide the detailed \ltlAcr{} semantics, and introduce the tools that provide a richer understanding of fragments of \ltlAcr{}.
We summarize in \cref{tab:fltequivalence} known characterizations of relevant fragments of \ltlAcr{}.
In \cref{app:pfo2mod}, we give a proof of the exact characterization of $\ptl[\modpred]$ in terms of polynomials, first-order logic, automata and the syntactic monoid.
Our proofs are based on similar characterizations of $\ptl$ and $\pftl[\modpred]$ \citep{li2025characterizingexpressivityfixedprecisiontransformer, dartois_et_al:LIPIcs.STACS.2013.329}.

\subsection{Linear Temporal Logic}\label{app:ltl}
The full semantics for $\ltlAcr{}$ are the following.
\begin{itemize}[nosep,noitemsep,leftmargin=*]
    \item $\str,\idxi\models \atom_\syma$ $\iff$ $\sym_\idxi=\syma$;
    \item $\str,\idxi\models \tlf_1 \lor \tlf_2$ $\iff$ $\str,\idxi\models \tlf_1 \lor \str,\idxi\models \tlf_2$;
    \item $\str,\idxi\models \tlf_1 \land \tlf_2$ $\iff$ $\str,\idxi\models \tlf_1 \land \str,\idxi\models \tlf_2$;
    \item $\str,\idxi\models \neg \tlf$ $\iff$ $\str,\idxi \not\models \tlf$;
    \item $\str,\idxi\models \past \tlf$ $\iff$ $\exists \idxj < \idxi: \str, \idxj\models \tlf$;
    \item $\str,\idxi\models \future \tlf$ $\iff$ $\exists \idxj > \idxi: \str, \idxj\models \tlf$;
    \item $\str,\idxi\models \tlf_1 \since \tlf_2$ $\iff$ $\exists \idxj<\idxi: \str,\idxj\models \tlf_2$ and $\str,k\models \tlf_1$ for all $k$ with $\idxj<k<\idxi$;
    \item $\str,\idxi\models \tlf_1 \until \tlf_2$ $\iff$ $\exists \idxj>\idxi: \str,\idxj\models \tlf_2$ and $\str,k\models \tlf_1$ for all $k$ with $\idxi<k<\idxj$;
    \item $\str,\idxi\models \yesterday \tlf$ $\iff$ $\str,\idxi-1\models \tlf$ for $\idxi > 1$;
\end{itemize}
To define string acceptance, we denote by $\strlen+1$ a position outside of the string and define
\begin{equation}
    \label{eg:tl_acc}
    \str\models\tlf \Leftrightarrow \str,\strlen+1\models \tlf.
\end{equation} 
For a set of temporal operators $\operators$, we denote by $\ltlAcr[\operators]$ the corresponding fragment of \ltlAcr{}. 
We say a formula $\tlf$ is a $\ltlAcr[\operators]$ formula if it can be written in \ltlAcr{} with only the Boolean connectives and operators from $\operators$.
$\since$ subsumes the left-context-dependent operators $\past$ and $\yesterday$ \citep{gabbay}. 

\begin{testexample}
    The language $\syma\kleene{\alphabet}$ of strings starting with a specific symbol $\syma$ can be described by a $\ptl$ formula:
    \begin{equation}
        \past \left(\atom_\syma \land \neg \past \true\right)
    \end{equation}

    The language $\kleene{\alphabet}\syma$ of strings ending with a specific symbol $\syma$ can be described by a $\ltlAcr[\yesterday]$ formula:
    \begin{equation}
        \yesterday \atom_\syma
    \end{equation}
    The language $\syma\kleene{\alphabet}\syma$ can be described by a $\ltlAcr[\since]$ definable formula because $\since$ can express both $\past$ and $\yesterday$.
\end{testexample}

\ltlAcr{} has become a key tool in characterizing the expressivity of fixed-precision transformers.
Softmax transformers, leftmost hard-attention transformers and average hard-attention transformers are equivalent to $\ptl$ \citep{li2025characterizingexpressivityfixedprecisiontransformer, jerad-etal-2025-unique}, while rightmost hard-attention transformers are equivalent to full $\stl$ \citep{yang2024maskedhardattentiontransformersrecognize}.

As formalized in \cref{sec:prelim}, $\ltlAcr{}$ can be extended with modular predicates, denoted by $\ltlAcr[\operators,\modpred]$ for a set of operators $\operators$.

\subsection{First-Order Logic}\label{app:first-order-logic}

First-order logic ($\fo$) can also be used to characterize subregular classes.
$\fo$'s building blocks are \defn{atomic} formulas, denoted by the unary predicate $\atom_\syma$ for all $\syma \in \alphabet$ and the binary predicate $<$ which determines the order between string positions.
We have $\atom_\syma(\idxi) = \true$ if and only if $\sym_\idxi = \syma$.
Building on these atomic formulas, $\fo$ operates on \defn{variables}, which denote positions of the input string. 
In $\fo$, a variable is either \defn{free} or \defn{bounded} by a quantifier. 
For example, in the existential quantification $\exists \idxi < \idxj$, $\idxi$ is a bounded variable while $\idxj$ is a free variable.
$\fo$ formulas are then constructed inductively as follows.
\begin{enumerate}[topsep=0pt, noitemsep,label=(\arabic*)]
    \item Every atomic formula is an $\fo$ formula;
    \item A Boolean combination of $\fo$ formulas is an $\fo$ formula;
    \item If $\aformula(\idxi \ldots)$ is an $\fo$ formula, then so is $\exists \idxi \colon \phi(\idxi \ldots)$\footnote{The universal quantifier $\forall$ is constructed by applying negation to the existential quantifier.}.
\end{enumerate}
Let $\aformula$ be a formula without any free variables. 
We say $\str \models \aformula$ if $\aformula$ is satisfied when evaluating it on $\str$.
$\aformula$ induces a language, which we denote by $\lang(\aformula)\defeq\set{\str\in\kleene{\alphabet} \mid \str \models \aformula}$.

\paragraph{Fragments of $\fo$:}\label{app:fo-fragments} We denote by $\fotwo$ the fragment of $\fo$ restricted to using at most two distinct variables in any sub-formula.
We denote by $\pfo$ the past fragment of $\fotwo$ in which formulas are restricted so that whenever a free variable $\idxi$ is present, every quantifier must be of the form $\exists \idxj \colon \idxj < \idxi$ or $\forall \idxj \colon \idxj < \idxi$.
In other words, only past positions relative to the free variable are accessed.
$\pfo$ formulas can be inductively written as follows.
\begin{enumerate}[topsep=0pt, noitemsep,label=(\arabic*)]
    \item Every atomic formula is a $\pfo$ formula;
    \item A Boolean combination of $\pfo$ formulas is a $\pfo$ formula if it does not contain more than two variables;
    \item If $\aformula(\idxi,\idxj)$ is a $\pfo$ formula with two free variables $\idxi, \idxj$, $\exists \idxi < \idxj \colon \aformula(\idxi,\idxj)$ and $\exists \idxj < \idxi \colon \aformula(\idxi,\idxj)$ are $\pfo$ formulas;
    \item If $\aformula(\idxi)$ is a $\pfo$ formula with one free variable $\idxi$, $\exists \idxi < \idxj \colon \aformula(\idxi)$ is a $\pfo$ formula;
    \item If $\aformula(\idxi)$ is a $\pfo$ formula with one free variable $\idxi$, $\exists \idxi \colon \aformula(\idxi)$ is a $\pfo$ formula.
\end{enumerate}

Over finite strings, $\pfo$ and $\ptl$ define the same languages \citep{li2025characterizingexpressivityfixedprecisiontransformer}, while $\fo$ is equivalent to the full star-free class, which is $\stl$ \citep{Kamp1968-KAMTLA, mcnaughton1971counter, gabbay}. 
We therefore use $\ptl$ in the main narrative and retain $\pfo$ when invoking first-order, automata-theoretic, or algebraic characterizations.

\paragraph{Modular predicates.}
Akin to $\ltlAcr{}$, $\fo$ can be extended with unary modular predicates, albeit with a slight modification to its \ltlAcr{} counterpart.
In \ltlAcr{}, formulas are evaluated at position $\strlen+1$, meaning $\ltlAcr{}$ can robustly make statements about the residue class of $\strlen$, the length of the string.
Because $\fo$ formulas are evaluated by ranging variables over the entire string, such an operator cannot be directly implemented with unary predicates.
Therefore, 0-ary predicates $\modpred^r_m$ are needed, where $\modpred^r_m$ is true if and only if the length of the input string $\strlen$ is congruent to $r$ modulo $m$.

Adding these modular predicates to first-order fragments gives $\pfo[\modpred]$, $\fotwo[\modpred]$, and $\fo[\modpred]$; as shown in \cref{subsec:ptl-mod}, $\ptl[\modpred]$ and $\pfo[\modpred]$ define the same class of languages.

\subsection{Regular Expressions} \label{subsec:regex}

\defn{Regular expressions} are a declarative formalism for describing languages, defined recursively:
\begin{enumerate*}[label=(\arabic*)]
    \item $\eps$ and each $\sym \in \alphabet$ is a regular expression;
    \item If $\alpha$ and $\beta$ are regular expressions, so are the union $\alpha + \beta$, concatenation $\alpha \beta$, complement $\reComplement{\alpha}$ and the Kleene closure $\kleene{\alpha}$.
\end{enumerate*}
A language is \defn{regular} if it can be described by a regular expression.
A regular language is \defn{star-free} if it can be described by a regular expression without the Kleene operator.

\begin{definition}
    A \defn{monomial} over an alphabet $\alphabet$ is a language described via a regular expression of the form:
    \begin{equation}
        \kleene{\alphabet_0} \sym_1 \kleene{\alphabet_1} \ldots \sym_\strlen \kleene{\alphabet_\strlen}
    \end{equation}
    Where $\alphabet_0, \alphabet_1 \ldots \subseteq \alphabet$ and $\sym_1, \sym_2 \ldots \in \alphabet$.
\end{definition}

\begin{definition}
    A \defn{left-deterministic} monomial is a monomial where for every $0 < \idxi \leq \strlen$, $\sym_\idxi \not \in \alphabet_{\idxi-1}$.
    A \defn{right-deterministic} monomial is a monomial where for every $0 < \idxi \leq \strlen$, $\sym_\idxi \not \in \alphabet_{\idxi}$.
    An \defn{unambiguous} monomial is a monomial that is either left-deterministic or right-deterministic.
\end{definition}
\begin{testexample}
    The language $\syma\kleene{\alphabet}$ of strings starting with the symbol $\syma$ is a left-deterministic monomial, since the marker $\syma$ does not appear in the preceding empty block ($\syma \notin \alphabet_0$).

    Symmetrically, the language $\kleene{\alphabet}\syma$ of strings ending with the symbol $\syma$ is a right-deterministic monomial.
    Both $\syma\kleene{\alphabet}$ and $\kleene{\alphabet}\syma$ are therefore unambiguous monomials.

    Their intersection $\syma\kleene{\alphabet}\cap\kleene{\alphabet}\syma=\syma\kleene{\alphabet}\syma$ is an unambiguous monomial.
\end{testexample}

A \defn{polynomial} is a finite union of monomials.
A left-deterministic (resp., right-deterministic, unambiguous) polynomial is a finite union of left-deterministic (resp., right-deterministic, unambiguous) monomials.

To adapt our analysis to $\rope$ transformers, we now introduce \emph{modular} polynomials \citep{dartois_et_al:LIPIcs.STACS.2013.329}.

\begin{definition}\label{def:modular-monomial}
    A monomial over $\alphabet$ is \defn{modular} with respect to some integer $\idxd$ if it can be written as:
    \begin{equation}
        \kleene{(\alphabet_{0,0} \alphabet_{0,1} \ldots \alphabet_{0,d-1})} \sym_1 \kleene{(\alphabet_{1,0} \alphabet_{1,1} \ldots \alphabet_{1,d-1})} \ldots \sym_\strlen \kleene{(\alphabet_{\strlen,0} \alphabet_{\strlen,1} \ldots \alphabet_{\strlen,d-1})}
    \end{equation}
    Where:
    \begin{itemize}
        \item Each marker $\sym_\idxi$ occurs at a position congruent to some $r_\idxi$ modulo $\idxd$;
        \item Each $\alphabet_{\idxi,\idxj}$ is some subset of $\alphabet$;
        \item For every $\idxi$ and every position $p$ lying in the $\idxi$'th block, the symbol at position $p$ belongs to $\alphabet_{\idxi,p \text{ mod } \idxd}$
    \end{itemize}
\end{definition}
Note that the standard monomial is recovered with $\idxd = 1$.
\begin{definition}\label{def:left-deterministic-modular}
    A \defn{left-deterministic modular} monomial is a modular monomial such that for every $0 < \idxi \leq \strlen$, $\sym_\idxi \notin \alphabet_{\idxi-1, r_\idxi}$.
\end{definition}

As before, we similarly define unambiguous modular monomials, unambiguous modular polynomials and left-deterministic modular polynomials.

\subsection{Automata}\label{app:automata}
The standard formalism for (sub)regular languages are finite automata---machines that transition within finitely many states.
\begin{definition}  \label{def:semiautomaton}
    A \defn{semiautomaton} $\wfsa$ is a 3-tuple $\satuple$ where $\alphabet$ is an alphabet, $\states$ is a finite set of \defn{states} and $\trans \colon \states \times \alphabet \rightarrow \states$ is a \defn{transition function}.
    We further define an \defn{initialized semiautomaton} as a semiautomaton with an initial state.
\end{definition}

\begin{definition}  \label{def:dfa}
    A \defn{deterministic finite automaton (DFA)} $\wfsa$ is a 5-tuple $\fsatuple$ where $\satuple$ is a semiautomaton, $\qinit \in \states$ is an initial state, and $\final \subseteq \states$ is a set of final states.
\end{definition}

We now introduce the automata that characterize the class of languages described by $\pfo$.
Intuitively, because quantifiers in $\pfo$ can only look \emph{backward}, they can only accumulate information about the past.
We now introduce a class of automata that mirror this behavior: Their states
are partially ordered and transitions may only move upward in this order.
Consequently, reading a symbol can only refine what the automaton knows about
the prefix read so far---no information can be discarded, and no state can be
revisited.

\begin{definition}
    \label{def:p.o}
    Let $\transclosure \colon \states \times \kleene{\alphabet} \rightarrow \states$ be the \defn{
    transitive closure} of $\trans$, defined as
    \begin{subequations}
        \begin{align}
            \transclosure(\stateq, \sym) &= \trans(\stateq, \sym), \; \text{for } \sym \in \alphabet\\
            \transclosure(\stateq, \sym_{1}\cdots\sym_{\strlen}) &= \trans(\transclosure(\stateq,\sym_1\cdots\sym_{\strlen-1}), \sym_{\strlen})
        \end{align}
    \end{subequations}
    with $\transclosure(\stateq, \eps) = \stateq$ for any $\stateq \in \states$.
    A \defn{partially ordered DFA (\podfaAcr{})} is a DFA $\wfsa = \dfatuple$ where there is a partial order relation $\relation$ on $\states$ defined as $\stateq$ $\relation$ $\statep$ if and only if $\transclosure(\stateq, \str) = \statep$ for some string $\str \in \kleene{\alphabet}$.
\end{definition}

As an example, the language of strings that contain the symbol $\syma$ over any alphabet can be modeled by the following $\podfaAcr$.
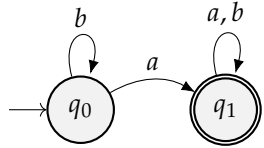
\begin{figure}[h]
    \centering
    \begin{tikzpicture}
        \node[state, initial] (q0) { $\stateq_0$ };
        \node[state, accepting] (q1) [right=of q0] { $\stateq_1$ };
        \draw[transition]
        (q0) edge[auto, loop above] node{$\symb$} (q0)
        (q0) edge[auto, bend left] node{$\syma$} (q1)
        (q1) edge[auto, loop above] node{$\syma,\symb$} (q1);
    \end{tikzpicture}
    \caption{$\podfaAcr$ for $\kleene{\alphabet}\syma\kleene{\alphabet}$ over $\alphabet = \set{\syma,\symb}$. The partial order $\stateq_0 \relation \stateq_1$ holds since $\stateq_1$ is reachable from $\stateq_0$ but not vice versa.}
    \label{fig:podfa-contains-a}
\end{figure}

To link automata with modular predicates, we now introduce the $\idxk$-automaton of some given automaton.
In first-order logic with modular predicates, strings can be evaluated to the same value \emph{up to a modulus}.
For some modulus $\idxk$, we therefore consider $\idxk$-automata where strings are read as blocks of $\idxk$ symbols.
\begin{definition}  \label{def:k-automaton}
    Let $\automaton = \dfatuple$ be a DFA.
    Let $\idxk$ be an integer.
    We define $\automaton_\idxk \defeq (\alphabet^\idxk, \states, \qinit, \final, \trans^\idxk)$ as the \defn{$\idxk$-automaton} of $\automaton$ where:
    \begin{itemize}
        \item $\alphabet^{\idxk} \defeq \set{\str \mid |\str|=\idxk \text{ and } \str \in \kleene{\alphabet}}$
        \item $\trans^\idxk(\stateq, \sym_1 \sym_2 \ldots \sym_\idxk) \defeq \trans(\ldots \trans(\trans(\stateq, \sym_1), \sym_2), \sym_\idxk)$.
    \end{itemize}
\end{definition}
We denote by \podfaAcr$_\idxk$ the set of automata such that their $\idxk$-automaton are \podfaAcr{}s.
As an example, $\kleene{(\syma \symb)}$ is the language of an automaton $\automaton$ in $\podfaAcr_2$, as seen in \cref{fig:ab_star_automata}.
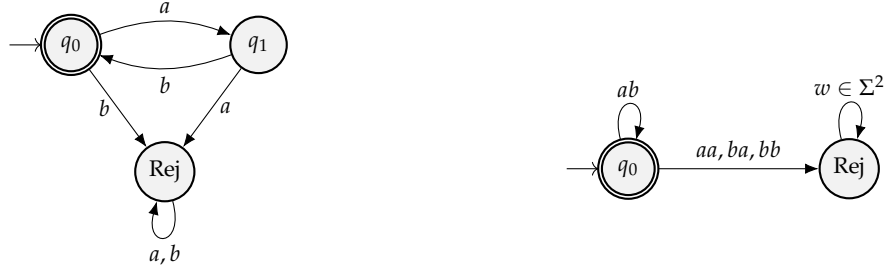
\begin{figure}[h]
    \centering
    \begin{minipage}{0.45\textwidth}
        \centering
        \begin{tikzpicture}[node distance=2cm, scale=0.85, transform shape]
            \node[state, initial left, accepting] (q0) {$\stateq_0$};
            \node[state] (q1) [right=of q0] {$\stateq_1$};
            \node[state] (rej) [below=1.5cm of $(q0)!0.5!(q1)$] {Rej};

            \draw[transition]
            (q0) edge[auto, bend left=20] node{$\syma$} (q1)
            (q1) edge[auto, bend left=20] node{$\symb$} (q0)
            (q0) edge[auto] node[left]{$\symb$} (rej)
            (q1) edge[auto, swap] node[right]{$\syma$} (rej)
            (rej) edge[loop below, looseness=8] node{$\syma,\symb$} (rej);
        \end{tikzpicture}
    \end{minipage}
    \hfill
    \begin{minipage}{0.45\textwidth}
        \centering
        \begin{tikzpicture}[node distance=2.5cm, scale=0.85, transform shape]
            \node[state, initial left, accepting] (q0) {$\stateq_0$};
            \node[state] (rej) [right=of q0] {Rej};

            \draw[transition]
            (q0) edge[loop above, looseness=8] node{$\syma\symb$} (q0)
            (q0) edge[auto] node{$\syma\syma, \symb\syma, \symb\symb$} (rej)
            (rej) edge[loop above, looseness=8] node{$\sym\in\alphabet^2$} (rej);
        \end{tikzpicture}
    \end{minipage}
    \caption{Minimal automata for $\kleene{(\syma\symb)}$ and its corresponding $2$-automaton.}
    \label{fig:ab_star_automata}
\end{figure}

\subsection{The Syntactic Monoid}\label{app:syntactic-monoid}
The \defn{syntactic monoid} is a tool that partitions $\kleene{\alphabet}$ into classes of strings, where each class contains strings that all contain the same \emph{syntactic} information with respect to a given regular language.
We now formalize the syntactic monoid of a language.

\begin{definition}\label{def:monoid}
    A $\defn{monoid}$ $\monoid$ is a set equipped with a binary operation and an identity element.
\end{definition}
For instance, the \defn{free monoid} is the set $\kleene{\alphabet}$ equipped with the concatenation operation and the empty string $\eps$ as identity.

\begin{definition}\label{def:rtrivial}
    A monoid $\monoid$ is in $\RMonoid$ if and only if for all $\sym_1, \sym_2, \sym_3 \in \monoid$, $\sym_1\sym_2\sym_3=\sym_1$ implies $\sym_1\sym_2=\sym_1$.
    A monoid $\monoid$ is in $\LMonoid$ if and only if for all $\sym_1, \sym_2, \sym_3 \in \monoid$, $\sym_3\sym_2\sym_1=\sym_1$ implies $\sym_2\sym_1=\sym_1$.
\end{definition}
\begin{definition}\label{def:congruence}
    The \defn{syntactic congruence} $\relation_\lang$ is the equivalence relation on $\kleene{\alphabet}$ given the language $\lang$ such that for all $\strx$, $\stry \in \kleene{\alphabet}$, we have $\strx \relation_\lang \stry$ if and only if:
    \begin{equation}
        \strs \strx \strz \in \lang \iff \strs \stry \strz \in \lang \ \forall \strs, \strz \in \kleene{\alphabet}
    \end{equation}
\end{definition}
Given a language, the syntactic congruence partitions strings of $\kleene{\alphabet}$ into classes such that they are syntactically equivalent with respect to the language's automaton.
\begin{testexample}
    Consider the language \textsc{Parity} (introduced in \cref{subsec:ptl-mod}) and its syntactic congruence $\relation_\lang$.
    $\relation_\lang$ then has exactly two congruence classes, one for strings with an even numbers of $1$s (which we denote by $\congruenceclass_{\text{even}}$) and another for strings with an odd number of $1$s (which we denote by $\congruenceclass_{\text{odd}}$). Example strings in each class are:
    \begin{align*}
        011, 1111, 0101 & \in \congruenceclass_{\text{even}} \\
        1, 10011, 1101 & \in \congruenceclass_{\text{odd}}
    \end{align*}
\end{testexample}
\begin{definition}\label{def:syntactic-monoid}
    Let $\monoid$ be the free monoid, $\lang$ be a regular language and $\relation_\lang$ its syntactic congruence.
    The \defn{syntactic monoid} of $\lang$ is the quotient monoid $\monoid / \relation_\lang$.
\end{definition}
Interestingly, the syntactic monoid of a regular language is isomorphic to the transition monoid of the minimal automaton accepting the language.

Finally, we now extend the monoid class $\RMonoid$ to account for modularity.

\begin{definition}\label{def:syntacticmonoid}
    Let $\monoid$ be the syntactic monoid of a regular language $\lang$ with associated automaton $\automaton$.
    Let $\monoid_\idxk$ be the syntactic monoid of the language associated with $\automaton_\idxk$, the $\idxk$-automaton.
    Then, $\monoid$ is in $\QRMonoid$ if and only if there exists an integer $\idxk$ such that $\monoid_\idxk$ is in $\RMonoid$.
\end{definition}

\subsection{Equivalences between Formalisms}\label{app:formalism-equivalences}
To conclude this section, we summarize in \cref{tab:fltequivalence} the known characterizations of $\fotwo[\modpred]$ and $\pfo$, and introduce the analogous characterization of $\pfo[\modpred]$ which we prove in \cref{app:pfo2mod}.
\begin{table}[H]
    \footnotesize
    \centering
    \renewcommand{\arraystretch}{1.75} 
    \begin{tabular}{
        >{\centering\arraybackslash}p{2.2cm}
        >{\centering\arraybackslash}p{2.2cm}
        >{\centering\arraybackslash}p{2.5cm}
        >{\centering\arraybackslash}p{1.2cm}
        >{\centering\arraybackslash}p{1.5cm}
        >{\centering\arraybackslash}p{2.2cm}
        }
        \toprule
        First-order logic & \textbf{LTL} & Regular expression & Syntactic Monoid & Automata & Note \\
        \midrule
        $\fotwo[\modpred]$ & $\tl[\modpred]$ & Unambiguous modular & $\QDAMonoid$ & 2-PODFA$_{\idxk}$ & \citet{dartois_et_al:LIPIcs.STACS.2013.329} \\
        $\pfo$ & $\ptl$ & Left-deterministic & $\RMonoid$ & PODFA & \citet{BRZOZOWSKI198032} and \citet{li2025characterizingexpressivityfixedprecisiontransformer} \\
        $\pfo[\modpred]$ & $\ptl[\modpred]$ & Left-deterministic modular & $\QRMonoid$ & PODFA$_\idxk$ & \cref{app:pfo2mod} \\
        \bottomrule
    \end{tabular}
    \caption{Characterizations of fragments of first-order logic. The equivalences in the last row are novel.}
    \label{tab:fltequivalence}
\end{table}

\section{Proofs: Characterizations of $\ptl[\modpred]$}\label{app:pfo2mod}
In this section, we provide a rich characterization of $\ptl[\modpred]$ via the following theorem.
\equivalence*

Most of the following lemmata are adaptations from equivalences proven in prior work \citep{dartois_et_al:LIPIcs.STACS.2013.329, dartois2015addingmodularpredicatesfirstorder, li2025characterizingexpressivityfixedprecisiontransformer} applied on different fragments of first-order logic.

\begin{lemma}\label{lem:DfaEquivMonoid}
    A language $\lang$ has its syntactic monoid in $\QRMonoid$ if and only if there exists a $\idxk$ such that $\lang$'s automaton belongs to $\podfaAcr_\idxk$.
\end{lemma}
\begin{proof}
    By definition, a language $\lang$ has its syntactic monoid in $\QRMonoid$ if and only if there exists a $\idxk$ such that its $\idxk$-automaton $\automaton_\idxk$ has its syntactic monoid in $\RMonoid$.
    The equivalence between the transition monoid of $\automaton_\idxk$ being in $\RMonoid$ and $\automaton_\idxk$ being a \podfaAcr{} follows from \citet{BRZOZOWSKI198032}.
    Therefore, $\lang$'s automaton belongs to $\podfaAcr_\idxk$ for some $\idxk$.
\end{proof}

\begin{lemma}\label{lem:PfoEquivQR}
    If $\phi$ is a $\pfo[\modpred]$ formula, then the syntactic monoid of $\phi$'s language is in $\QRMonoid$.
    If a language's syntactic monoid is in $\QRMonoid$, it is the language of some $\pfo[\modpred]$ formula.
\end{lemma}
\begin{proof}
    The regular languages with a syntactic monoid in $\RMonoid$ are exactly the $\pfo$ languages \citep{li2025characterizingexpressivityfixedprecisiontransformer}.
    Moreover, for any logic characterized by a local variety \textbf{V}, augmenting the logic with modular predicates yields a logic characterized by \textbf{QV} \citep{dartois2015addingmodularpredicatesfirstorder}.
    Since $\pfo$ is characterized by $\RMonoid$ and $\RMonoid$ is local \citep{Almeida_2025}, augmenting $\pfo$ with modular predicates gives the class characterized by $\QRMonoid$.
\end{proof}

\begin{lemma}\label{lem:FolEquivLtl}
    $\pfo[\modpred]$ and $\ptl[\modpred]$ characterize the same class of languages.
\end{lemma}
\begin{proof}
    The equivalence without modular predicates has already been proven via structural induction \citep{li2025characterizingexpressivityfixedprecisiontransformer}. 
    Adding the same unary predicate $\modpred$ to both logics therefore preserves the equivalence.
    It remains to show that unary modular predicates in $\ptl[\modpred]$ can express 0-ary predicates in $\pfo[\modpred]$.

    Formulas in $\ptl$ are evaluated at the string position $\strlen+1$. 
    Moreover, $\strlen$ is congruent to $r$ modulo $m$ if and only if $\strlen+1$ is congruent to $r+1$ modulo $m$.
    Therefore, evaluating in $\ptl[\modpred]$ the unary modular predicate $\modpred_m^{r+1}$ at $\strlen+1$ is equivalent to evaluating the 0-ary predicate $\modpred_m^{r}$ in $\pfo[\modpred]$.    
\end{proof}

To prove the equivalence between $\pfo[\modpred]$ and left-deterministic modular polynomials, we consider languages over \emph{enriched alphabets} \citep{dartois_et_al:LIPIcs.STACS.2013.329}.

We denote by $\pfo[\modpred_\idxd]$ the restriction of $\pfo[\modpred]$ with congruences modulo $\idxd$.
For any $\pfo[\modpred]$ formula, there exists an integer $\idxd$ such that the formula is also $\pfo[\modpred_\idxd]$-definable---by for example setting $\idxd$ as the least common multiplier of all the moduli used.

\begin{definition}[Enriched alphabets; \citealt{dartois_et_al:LIPIcs.STACS.2013.329}]
    The set $\alphabet_\idxd \defeq \alphabet \times (\Z/\idxd\Z)$ is the \defn{enriched alphabet} of $\alphabet$, where $\Z/\idxd\Z \defeq \set{0,1,\ldots,\idxd-1}$.
    The projection $\projfunc \colon \kleene{\alphabet_\idxd} \to \kleene{\alphabet}$ removes the $(\Z/\idxd\Z)$ component of every symbol $(\syma, \idxi) \in \alphabet_\idxd$ of strings in $\kleene{\alphabet_\idxd}$.
\end{definition}

\begin{definition}[Well-formed words; \citealt{dartois_et_al:LIPIcs.STACS.2013.329}]
    Given some alphabet $\alphabet$, the enriched word $(\sym_0, \idxi_0)\ldots(\sym_\strlen, \idxi_\strlen)$ is \defn{well-formed} if $\idxi_\idxj \equiv \idxj \pmod \idxd$ for all $0 \leq \idxj \leq \strlen$.
    We denote by $\WFSet \subset \kleene{\alphabet_\idxd}$ the set of well-formed words.
\end{definition}
Note that on $\WFSet$, the projection operator $\projfunc$ is a bijective application.
Indeed, every word $\str = \sym_0\cdots\sym_\strlen \in \kleene{\alphabet}$ has a unique preimage in $\WFSet$: the word $(\sym_0,\, 0 \bmod \idxd)\cdots(\sym_\strlen,\, \strlen \bmod \idxd)$ obtained by annotating each position $j$ with $j \bmod \idxd$. Surjectivity is immediate, and injectivity holds because distinct well-formed words over $\alphabet_\idxd$ differ in at least one symbol and therefore project to distinct words.

We can now describe $\pfo[\modpred]$ with enriched alphabets.\footnote{We may write $\pfo(\kleene{\alphabet})$ to specify that the logic is being used over the alphabet $\alphabet$.}
\begin{proposition}\label{prop:pfo2enriched}
$\pfo[\modpred_\idxd](\kleene{\alphabet}) = \projfunc(\pfo(\kleene{\alphabet_\idxd})\cap \WFSet)$
\end{proposition}
\begin{proof}
    The statement has been shown to be true for $\fotwo[\modpred]$ \citep{dartois_et_al:LIPIcs.STACS.2013.329}.
    The result does not depend on how the binary predicate $<$ is used, and therefore it also holds for $\pfo[\modpred]$.
    Modular predicates of the form $\modpred^r_\idxd$ can be computed by disjuncting over all atomic formulas of the form $\atom_{(\sym, r)}$ in $\pfo(\kleene{\alphabet_\idxd})\cap \WFSet$ for all $\sym \in \alphabet$.
    For the converse, formulas over exclusively well-formed words can be equivalently described by formulas over words in $\kleene{\alphabet}$ with modular predicates.
    For instance, $\atom_{(\sym, r)}$ in $\pfo(\kleene{\alphabet_\idxd})\cap \WFSet$ can be written as $\atom_{\sym}\land\modpred^r_\idxd$ in $\pfo[\modpred_\idxd](\kleene{\alphabet})$.
\end{proof}
We can now prove the following lemma.

\begin{lemma}\label{lem:PFOModPredEquivPoly}
    A language is a left-deterministic modular polynomial if and only if it can be described by a $\pfo[\modpred]$ formula.
\end{lemma}
\begin{proof}
    Recall that for any $\pfo[\modpred]$ formula describing a language $\lang$, there exists a $\pfo[\modpred_\idxd]$ formula, for some integer $\idxd$, that also describes $\lang$.
    By \cref{prop:pfo2enriched}, $\lang$ is therefore the projection of the well-formed words of a language $\lang'$ definable by a $\pfo(\kleene{\alphabet_\idxd})$ formula.
    It remains to be shown via this projection that $\lang$ can be equivalently described with a left-deterministic polynomial.

    Because $\lang'$ is definable by a $\pfo(\kleene{\alphabet_\idxd})$ formula, it is a disjoint union of left-deterministic monomials over $\alphabet_\idxd$ \citep{li2025characterizingexpressivityfixedprecisiontransformer}.
    Recall that over well-formed strings, $\projfunc$ is a bijection.
    Therefore, $\projfunc$ preserves disjoint unions.
    It thus suffices to show that $\projfunc$ maps each left-deterministic monomial over $\alphabet_\idxd$ to a left-deterministic \emph{modular} monomial over $\alphabet$, and conversely.

    Fix an enriched monomial $\monomial = \kleene{\Xi_0}\,\overbar{\sym}_{1}\,\kleene{\Xi_1} \ldots \overbar{\sym}_{\strlen}\,\kleene{\Xi_\strlen}$ over $\alphabet_\idxd$, where $\Xi_\idxi \subseteq \alphabet \times (\Z/\idxd\Z)$ and $\overbar{\sym}_{\idxi} = (\sym_\idxi, r_\idxi)$ for all $0 \leq \idxi \leq \strlen$, and set
    \begin{equation*}
        \alphabet_{\idxi,\idxj} \defeq \set{\sym \mid (\sym, \idxj) \in \Xi_\idxi}.
    \end{equation*}
    We claim that the projection of the well-formed words of $\monomial$ is exactly the modular monomial determined, in the sense of \cref{def:modular-monomial}, by the markers $\sym_\idxi$, the residues $r_\idxi$, and the subsets $\alphabet_{\idxi,\idxj}$.

    \textbf{Forward} ($\Rightarrow$). Let $\str = \projfunc(\overbar{\str})$ for a well-formed $\overbar{\str} \in M$, factored as $\overbar{\str} = \overbar{u}_0\,\overbar{\sym_1}\,\overbar{u}_1 \ldots \overbar{\sym_\strlen}\,\overbar{u}_\strlen$ with $\overbar{u}_\idxi \in \kleene{\Xi_\idxi}$.
    Projecting yields $u = u_0\,\sym_1\,u_1 \ldots \sym_\strlen\,u_\strlen$.
    Well-formedness places $\overbar{\sym_\idxi} = (\sym_\idxi, r_\idxi)$ at a position $\equiv r_\idxi \pmod{\idxd}$, so each $\sym_\idxi$ stands at residue $r_\idxi$ in $u$.
    For any position $p$ inside the block $u_\idxi$, its enriched letter is $(\sym, p \bmod \idxd) \in \Xi_\idxi$, hence $\sym \in \alphabet_{\idxi,\,p \bmod \idxd}$.
    Thus $u$ satisfies both conditions of \cref{def:modular-monomial} and lies in the modular monomial.

    \textbf{Reverse} ($\Leftarrow$). Conversely, any $u$ in that modular monomial admits a factorization meeting the same two conditions; tagging each position $p$ of $u$ with $p \bmod \idxd$ produces the unique well-formed preimage $\overbar{u}$ with $\projfunc(\overbar{u}) = u$, and those conditions place $\overbar{u}$ in $M$.
    No rotation or partial-block bookkeeping arises, since \cref{def:modular-monomial} is stated over global positions and already absorbs each block's residue offset.

    Finally, left-determinism transfers in both directions: for every $0 < \idxi \leq \strlen$,
    \begin{equation*}
        \overbar{\sym_\idxi} \notin \Xi_{\idxi-1}
        \iff (\sym_\idxi, r_\idxi) \notin \Xi_{\idxi-1}
        \iff \sym_\idxi \notin \alphabet_{\idxi-1,\,r_\idxi},
    \end{equation*}
    the right-hand side being exactly the condition of \cref{def:left-deterministic-modular}.
    Hence $M$ is left-deterministic if and only if its projected modular monomial is left-deterministic, and $\projfunc$ restricts to a bijection between left-deterministic monomials over $\alphabet_\idxd$ and left-deterministic modular monomials over $\alphabet$.
    Applying this correspondence to each term of the disjoint union describing $\lang'$, together with \cref{prop:pfo2enriched}, shows that $\lang$ is in $\pfo[\modpred]$ if and only if it is a left-deterministic modular polynomial.
\end{proof}

\subsection{Inexpressibility Results}\label{app:inexpressibility}

\begin{restatable}{proposition}{propLT}\label{prop:LTNotPFOMOD}
    The locally testable language $\kleene{\alphabet} \syma \symb \kleene{\alphabet}$ with $|\alphabet|>2$ is not in $\ptl[\modpred]=\pfo[\modpred]=\smat[\ropeperiodic]$\footnote{Note that over the two-letter alphabet $\alphabet = \set{\syma,\symb}$, $\kleene{\alphabet} \syma \symb \kleene{\alphabet}$ is in $\pfo$ and thus also in $\pfo[\modpred]$.}.
\end{restatable}
\begin{proof}\label{app:proofLT}
    Consider the minimal DFA $\automaton$ for $\kleene{\alphabet}
    \syma \symb \kleene{\alphabet}$ with the alphabet $\alphabet=\set{\syma,\symb,\symc}$, illustrated in \cref{fig:ab-substring-3letter}.
    \\
    \begin{figure}[H]
        \centering
        \begin{tikzpicture}
            \node[state, initial] (q0) { $\stateq_0$ };
            \node[state] (q1) [right=of q0] { $\stateq_1$ };
            \node[state, accepting] (q2) [right=of q1] { $\stateq_2$ };

            \draw[transition]
            (q0) edge[auto, loop above] node{$\symb,\symc$} (q0)
            (q1) edge[auto, loop above] node{$\syma$} (q1)
            (q2) edge[auto, loop above] node{$\syma,\symb,\symc$} (q2)
            (q0) edge[auto, bend left] node{$\syma$} (q1)
            (q1) edge[auto, bend left] node{$\symb$} (q2)
            (q1) edge[auto, bend left] node{$\symc$} (q0);
        \end{tikzpicture}
        \caption{Minimal DFA for $\kleene{\alphabet}
        \syma \symb \kleene{\alphabet}$ over
        $\alphabet = \set{\syma,\symb,\symc}$}
        \label{fig:ab-substring-3letter}
    \end{figure}
    $\automaton_1 = \automaton$ can be seen to not be partially ordered.
    Consider some integer $\idxk > 1$ and the $\idxk$-automaton $\automaton_\idxk$ of $\automaton$.
    Let $\str_\syma\in\alphabet^\idxk$ be a string of $\idxk$ times the symbol $\syma$, and $\str_\symc\in\alphabet^\idxk$ a string of $\idxk$ times the symbol $\symc$.
    We have that $\trans^\idxk(\stateq_0,\str_\syma)=\stateq_1$ and $\trans^\idxk(\stateq_1,\str_\symc)=\stateq_0$.
    Therefore, there is no partial order on the states of $\automaton_\idxk$ as $\stateq_0$ and $\stateq_1$ are mutually reachable.
    Thus, for any integer $\idxk$, $\automaton_\idxk$ can never be partially ordered, and $\kleene{\alphabet} \syma \symb \kleene{\alphabet} \not \in \pfo[\modpred]$ for $|\alphabet|>2$.
\end{proof}

\begin{restatable}{proposition}{propParity}\label{prop:ParityNotPFOMOD}
    \textsc{Parity} is not in $\ptl[\modpred]=\pfo[\modpred]=\smat[\ropeperiodic]$.
\end{restatable}
\begin{proof}\label{app:proofParity}
    Consider the minimal DFA $\automaton$ for \textsc{Parity}, illustrated in \cref{fig:parity-automaton}.
    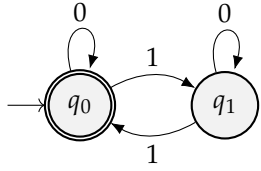
\begin{figure}[H]
        \centering
        \begin{tikzpicture}
            \node[state, initial, accepting] (q0) { $\stateq_0$ };
            \node[state] (q1) [right=of q0] { $\stateq_1$ };

            \draw[transition]
            (q0) edge[auto, loop above] node{$0$} (q0)
            (q1) edge[auto, loop above] node{$0$} (q1)
            (q0) edge[auto, bend left] node{$1$} (q1)
            (q1) edge[auto, bend left] node{$1$} (q0);
        \end{tikzpicture}
        \caption{Minimal DFA for \textsc{Parity} over $\alphabet = \set{0,1}$}
        \label{fig:parity-automaton}
    \end{figure}
    $\automaton_1 = \automaton$ can be seen to not be partially ordered.
    Consider some integer $\idxk > 1$ and the $\idxk$-automaton $\automaton_\idxk$ of $\automaton$.
    Let $\str$ be a string of length $\idxk$ with an odd number of $1$s.
    We have that $\trans^\idxk(\stateq_0,\str)=\stateq_1$ and $\trans^\idxk(\stateq_1,\str)=\stateq_0$.
    Therefore, there is no partial order on the states of $\automaton_\idxk$ as $\stateq_0$ and $\stateq_1$ are mutually reachable.
    Thus, for any integer $\idxk$, $\automaton_\idxk$ can never be partially ordered, and \textsc{Parity}$\not \in \pfo[\modpred]$.
\end{proof}

\section{Proofs: Periodic RoPE}\label{app:periodic-proofs}
This section supplies the full proofs for periodic $\rope$ and for both inclusions summarized in \cref{sec:smatrope}: first $\smat[\ropeperiodic]\subseteq\pfo[\modpred]$, and then $\ptl[\modpred]\subseteq\smat[\ropeperiodic]$.
In all the following proofs, we follow the integer indexing convention for arrays:
For $\vv \in \finiteset^\hiddDim$ and a finite set of integers $\set{\idxi_1, \cdots \idxi_\idxn}$, $\vv[\idxi_1, \cdots \idxi_\idxn]\in\finiteset^\idxn$ is the vector of elements in $\vv$ indexed by $\idxi_1, \cdots \idxi_\idxn$.

\subsection{Component-Periodic Schedules}\label{app:schedules-proofs}

\periodicinputs*
\begin{proof}
    Let $m_1,\ldots,$$m_{\hiddDim/2}$ be component periods and $M\defeq\operatorname{lcm}\left(m_1,\ldots,m_{\hiddDim/2}\right)$.
    Then $\widehat{\mR}_{\idxi+M}=\widehat{\mR}_{\idxi}$, so the image is contained in $\left\{\widehat{\mR}_0,\ldots,\widehat{\mR}_{M-1}\right\}$. For a matrix coordinate $c$ and realized value $v$, define $A_{c,v}\defeq\left\{a\in\left\{0,\ldots,M-1\right\}:\widehat{\mR}_a[c]=v\right\}$. Then
    \begin{equation}\label{eq:periodic-with-mod}
        \widehat{\mR}_{\idxi}[c]=v\quad\Longleftrightarrow\quad\bigvee_{a\in A_{c,v}}\modpred_M^a\left(\idxi\right).
    \end{equation}
    Thus every entry-value relation is definable with unary modular predicates.
\end{proof}

\rationalangles*
\begin{proof}
    We consider the mapping for cosine, as the sine analysis is the same.
    For some $\idxi \in \N$, we have:
    \begin{subequations}
        \begin{align}
            \cos\left(\theta_\idxd \idxi\right)
            &= \cos\left(\frac{2\pi \theta_\idxd}{2\pi} \idxi\right) \\
            &= \cos\left(2\pi \frac{a_\idxd}{m_\idxd} \idxi\right) \\
            &= \cos\left(2\pi \frac{a_\idxd}{m_\idxd} \idxi+2a_\idxd\pi\right) \\
            &= \cos\left(2\pi \frac{a_\idxd}{m_\idxd} (\idxi+m_\idxd)\right) \\
            & = \cos\left(\theta_\idxd (\idxi+m_\idxd)\right)
        \end{align}
    \end{subequations}
    Therefore, $\cos\left(\theta_\idxd \idxi\right)$ is periodic with integer period $m_\idxd$.
\end{proof}

\subsection{Upper Bounding Periodic RoPE}\label{app:upperperiodic-proofs}
\ropepupperbound*
\begin{proof}
    We formalize what it means for $\pfo[\modpred]$ to \emph{simulate} a transformer.
    \begin{restatable}[\text{\citealt{li2025characterizingexpressivityfixedprecisiontransformer}}, Definition B.1.]{definition}{deflogicsimulatestf}\label{def:logicsimulatestf}
        A function $\tffunc \colon \kleene{(\eosalphabet)} \to \kleene{(\finiteset^{\hiddDim})}$ is said to be \defn{simulated} by $\pfo[\modpred]$ if for every dimension $\idxd \in \NTo{\hiddDim}$ and every floating-point value $\floatvalue \in \finiteset$, there exists a single-variable formula $\aformula$ in $\pfo[\modpred]$ such that for every position $\idxi \in \NTo{\strlen}$:
        \begin{equation}
            \tffunc(\str)_{\idxd, \idxi} = \floatvalue \iff \str, \idxi \models \aformula
        \end{equation}
        Similarly, we can define the simulation of a matrix-valued function via a two-variable $\pfo[\modpred]$ formula.
    \end{restatable}
    We prove that $\pfo[\modpred]$ can simulate any $\ropeperiodic$ transformer.
    First, each realized matrix entry can be simulated by the formulas from \cref{prop:periodicinputs}.

    A realized rotation matrix $\widehat{\mR}_{\idxi}$ can be seen as a function mapping positions $\idxi$ to elements in $\finiteset^{\hiddDim\times\hiddDim}$.
    For every matrix coordinate $c$ $\in \NTo{\hiddDim\times\hiddDim}$ and floating-point value $\floatvalue \in \finiteset$, we define a $\pfo[\modpred]$ formula $\aformula_{c, \floatvalue}$ such that:
    \begin{equation}
        \widehat{\mR}_{\idxi}[c] = \floatvalue \iff \str, \idxi \models \aformula_{c, \floatvalue}
    \end{equation}

    \textbf{0-elements of the matrix:} For the 0-elements of $\widehat{\mR}_{\idxi}$, we create constant formulas $\aformula_\true \defeq \true$ and $\aformula_\false \defeq \false$.
    If the floating-point value $\floatvalue$ is 0, we assign the constant formula $\aformula_\true$.
    Otherwise, we assign $\aformula_\false$.

    \textbf{Position-dependent elements of the matrix.}
    For every coordinate $c$ and realized value $v$, the formula $\bigvee_{a\in A_{c,v}}\modpred_M^a\left(\idxi\right)$ is true exactly where $\widehat{\mR}_{\idxi}[c]=v$ (\cref{prop:periodicinputs}). This covers every block entry without recovering an ideal angle from its finite-precision realization.

    Therefore, we can simulate the rotation matrices with $\pfo[\modpred]$.
    Omitting the periodic rotations, all other components of a fixed precision soft-attention transformer can be simulated by $\pfo$ \citep{li2025characterizingexpressivityfixedprecisiontransformer}.
    Finally, as the elements of the rotation matrices belong to a finite set (due to \cref{prop:periodicinputs}), the rotation matrices belong themselves to some finite set.
    Therefore, the rotation of the keys and queries is a mapping between two finite sets.
    Any mapping between finite sets can be simulated via a disjunction over the finitely-many inputs in the domain that lead to some output in the image \citep{chiang2023tighterboundsexpressivitytransformer, li2025characterizingexpressivityfixedprecisiontransformer}.
    This concludes the proof.
\end{proof}

\binarymodpred*
\begin{proof}
    We will show that $\modpred^r_m(\idxi, \idxj)$ can be written as the following finite disjunction.
    \begin{equation}
        \modpred^r_m(\idxi, \idxj) = \bigvee_{a=0}^{m-1}(\modpred_m^a(\idxi)\land\modpred^{(a+r) \textnormal{ mod } m}_m(\idxj))
    \end{equation}
    We denote by $\idxi_0$ and $\idxj_0$ the residues such that $\idxi \equiv \idxi_0\pmod m$ and $\idxj \equiv \idxj_0 \pmod m$.

    $\modpred^r_m(\idxi, \idxj) = \true$ if and only if $(\idxj - \idxi) \equiv r \pmod m$.
    Equivalently, $(\idxj_0 - \idxi_0) \equiv r \pmod m$, i.e. $\idxj_0 \equiv r+\idxi_0 \pmod m$.
    Therefore, $\modpred^r_m(\idxi, \idxj) = \true$ if and only if there exist residues $\idxi_0$ and $\idxj_0$ such that 1) $\idxi \equiv \idxi_0 \pmod m$ 2) $\idxj \equiv \idxj_0 \pmod m$ with $\idxj_0 \equiv \idxi_0+r \pmod m$.
    This statement is exactly the right hand side of the given equation.
\end{proof}

\subsection{Lower Bounding Periodic RoPE}\label{app:lowerperiodic-proofs}
\textbf{On finite precision.}
Before proving the lower bound $\ptl[\modpred] \subseteq \smat[\ropeperiodic]$, we recall that our result holds for \emph{some} finite-precision regime $\finiteset$.
In other words, we design a tailored finite set of values $\finiteset$ that works for our constructions.
This assumption is implicitly made in related work \citep{yang2024maskedhardattentiontransformersrecognize}.
This allows us to assume that the finitely many values our constructions manipulate are exactly representable, i.e. are present in $\finiteset$.
For each modulus $m$, we specify the required roots of unity directly as an $m$-entry rotation table and include its finitely many entries in $\finiteset$. 
A cyclic counter or modular index return the corresponding table entry. This establishes an exact result in the abstract expressivity model; it does not assert that standard floating-point hardware represents $2\pi/m$, the corresponding roots of unity, or the unbounded product $2\pi\idxi/m$ exactly. The experiments in \cref{sec:experiments-periodic} test a practical finite-precision implementation empirically.

\ropeplowerbound*
\begin{proof}
    We formalize what it means for $\smat[\ropeperiodic]$ to \emph{simulate} a $\ptl[\modpred]$ formula.
    \begin{restatable}[\text{\citealt{li2025characterizingexpressivityfixedprecisiontransformer}}, Definition B.9.]{definition}{deftfsimulateslogic}\label{def:tfsimulateslogic}
        A $\ptl[\modpred]$ formula $\tlf$ is simulated by a function $\tffunc \colon \kleene{(\eosalphabet)} \to \kleene{(\finiteset^\hiddDim)}$ if there exists a coordinate $\idxd \in \NTo{\hiddDim}$ such that for every input string $\str \in \kleene{(\eosalphabet)}$ and string position  $\idxi \in \NTo{\strlen}$:
        \begin{equation}
            \tffunc(\str)_{\idxd, \idxi} = \begin{cases}
                1 \text{ if } \str, \idxi \models \tlf \\ 
                0 \text{ otherwise }
            \end{cases}
        \end{equation}
    \end{restatable}

    We proceed by structural induction on the $\ptl[\modpred]$ formula.

    \paragraph{Induction on $\ptl$ without $\modpred$.}
    Any $\ptl$ formula can be simulated by a transformer (without PEs) via structural induction on its constituents \citep{li2025characterizingexpressivityfixedprecisiontransformer}.
    To lift this to $\smat[\ropeperiodic]$, we allocate dedicated dimensions for the Boolean connectives and $\past$---the sole temporal operator of $\ptl$---and set their rotation angles to $0$ (identity matrices), so that $\rope$ has no effect on those dimensions.
    We then apply \citeposs{li2025characterizingexpressivityfixedprecisiontransformer} simulation directly, yielding $\ptl \subseteq \smat[\ropeperiodic]$.

    \paragraph{Simulating modular predicates.}   
    It remains to realize each modular atom. 
    Consider the formula $\modpred_m^r\left(\idxi\right)$.
    We isolate the case $\idxi=1$, where $\modpred^r_m(1)$ is then a \emph{compile-time constant}. 
    The first position can be flagged by performing uniform attention over the strict left context and checking the count is exactly 1, which we assume is in $\finiteset$.
    We can then store $\modpred^r_m(1)$ at that position.
    Moreover, the case $m=1$ is the constant-true predicate, so we proceed to the $m\geq2$, $\idxi \geq 2$ case. 
    
     We first allocate one rotary pair at dimensions $d, d+1$ with the exact period-$m$ lookup table (see \cref{subsec:componentperiodic})
    \begin{equation}
        \begin{aligned}
        T_m\left(s\right)&\defeq\rotationmat{2\pi/m}{s}\quad\left(0\leq s<m\right),
        &\widehat{\mR}^{(\idxd)}_\idxi&\defeq T_m\left(\idxi\bmod m\right),\\
        \delta_m&\defeq\min_{1\leq s<m}\left(1-\cos\left(2\pi s/m\right)\right)>0,
        &b_m&\defeq1-\delta_m/2.
        \end{aligned}
    \end{equation}
    We also allocate one unrotated pair of dimensions $\idxd', \idxd'+1$.
    All other dimensions are unused (i.e., manually set to $0$ in keys and queries) in our construction.

    \underline{\textit{Query}.}
    Before the rotations, we set the elements of the query projection to position-independent constants. 
    These constants are implemented by setting the query weight matrix $\mQ$ to a zero-matrix and encoding the desired values in the query bias vector $\biasVector_q$, so that $\mQ \vx^{\layerIdx}_\idxi + \biasVector_q = \biasVector_q$ for all inputs $\vx^{\layerIdx}_\idxi$.
    \begin{align*}
        (\mQ \vx^{\layerIdx}_\idxi)[\idxd,\idxd+1] & = \begin{pmatrix}
            \cos\left(2\pi r/m\right)\\
            -\sin\left(2\pi r/m\right)
        \end{pmatrix}\\
        (\mQ \vx^{\layerIdx}_\idxi)[\idxd',\idxd'+1] & = 
        \begin{pmatrix}
            1\\
            0
        \end{pmatrix}
    \end{align*}
    The post-rotation query dimensions then become:
    \begin{align*}
        \vq^{\layerIdx}_\idxi[\idxd,\idxd+1] & = \begin{pmatrix}\cos\left(\frac{2\pi \idxi}{m}\right)\cos\left(\frac{2\pi r}{m}\right)+\sin(\frac{2\pi \idxi}{m})\sin\left(\frac{2\pi r}{m}\right) \\
        \sin(\frac{2\pi \idxi}{m})\cos\left(\frac{2\pi r}{m}\right)-\cos\left(\frac{2\pi \idxi}{m}\right)\sin\left(\frac{2\pi r}{m}\right)\end{pmatrix}
        \\
        \vq^{\layerIdx}_\idxi[\idxd',\idxd'+1] & = \begin{pmatrix}1 \\
        0\end{pmatrix}
    \end{align*}
    \underline{\textit{Key.}}
    We set the pre-rotation dimensions of the key vector as follows.
    \begin{align*}
    (\mK \vx^{\layerIdx}_\idxj)[\idxd,\idxd+1] & = \begin{pmatrix}
    \ropeconstant\mathbbm{1}\{\sym_\idxj=\bos\} \\
    0
    \end{pmatrix} \\
    (\mK \vx^{\layerIdx}_\idxj)[\idxd',\idxd'+1] & = \begin{pmatrix}
    Cb_m (1-\mathbbm{1}\{\sym_\idxj=\bos\})\\
    0
    \end{pmatrix} 
\end{align*}
    Where $C$ is a positive constant in $\finiteset$ defined later, and $\mathbbm{1}\{\sym_\idxj=\bos\}$ can be computed via a feedforward network.

    Note that because $\bos$ correspond to position $0$, the dimensions $d,d+1$ are unaffected by rotation: At $\bos$ the rotation is an identity mapping, at a symbol position the vector is a $0$-vector.
    Moreover, the dimensions $d',d'+1$ are by construction unrotated.

    \underline{\textit{Score.}}
    The query--key computation score yields:
    \begin{equation}\label{eq:score-periodic-simple}        \transpose{(\vq^{\layerIdx}_\idxi)}\vk^{\layerIdx}_\idxj=\begin{cases}
        \ropeconstant\cos\left(2\pi\left(\idxi-r\right)/m\right), & \text{if } \sym_\idxj = \bos \\
        \ropeconstant b_m, & \text{otherwise }
      \end{cases}
    \end{equation}
    Notice that $\cos\left(2\pi\left(\idxi-r\right)/m\right)$ is maximized when $\idxi\equiv r\pmod m$.

    \underline{\textit{Modular predicate condition satisfied.}}
    If $\idxi\equiv r\pmod m$, the score with $\bos$ is $\ropeconstant$ and the scores with non-$\bos$ positions are all $\ropeconstant(1-\delta_m/2)$.
    Therefore, the largest score is $\ropeconstant$ by a margin of $\ropeconstant\delta_m/2$.
    
    \underline{\textit{Modular predicate condition not satisfied.}}
    Suppose $\idxi\not\equiv r\pmod m$.
    Then $\cos\left(2\pi(\idxi-r)/m\right)\leq1-\delta_m$ by definition of $\delta_m$, so the score with $\bos$ satisfies $\ropeconstant\cos\left(2\pi(\idxi-r)/m\right)\leq\ropeconstant(1-\delta_m)$.
    Every non-$\bos$ position still scores exactly $\ropeconstant b_m=\ropeconstant(1-\delta_m/2)$, independently of $\idxj$.
    Hence every non-$\bos$ position now beats $\bos$, by a margin of at least
    \begin{equation}
        \ropeconstant(1-\delta_m/2)-\ropeconstant(1-\delta_m)=\ropeconstant\delta_m/2.
    \end{equation}
    So, in either case ($\idxi\equiv r\pmod m$ or not), whichever side wins---$\bos$ or the (tied) non-$\bos$ positions---does so by a margin of at least $\ropeconstant\delta_m/2$.

    \underline{\textit{Choosing $\ropeconstant$ and extracting the predicate.}}
    We define the value vector to be $1$ at $\bos$ and $0$ at every other position, so that the attention output at position $\idxi$ is exactly $\valpha_{\idxi,\bos}$.
    Selecting $\ropeconstant\geq2\floatlarge/\delta_m$ makes the margin $\ropeconstant\delta_m/2$ established above at least $\floatlarge$ in both cases.
    By stable softmax (\cref{sec:prelim}), subtracting the maximum score from every score leaves an exponent of $0$ at every position attaining the maximum and an exponent at most $-\floatlarge$---hence rounding to $\exp(-\floatlarge)=0$---at every position on the losing side of the margin.
    \begin{itemize}
        \item If $\idxi\equiv r\pmod m$: $\bos$ uniquely attains the maximum score, so $\valpha_{\idxi,\bos}=1$ and $\valpha_{\idxi,\idxj}=0$ for every non-$\bos$ $\idxj$. The attention output is $1$.
        \item If $\idxi\not\equiv r\pmod m$: Every non-$\bos$ position attains the (tied) maximum score, so $\valpha_{\idxi,\bos}=0$, while the non-$\bos$ weights $\valpha_{\idxi,\idxj}$ absorb the remaining mass. The attention output is still $\valpha_{\idxi,\bos}=0$, regardless of how stable softmax splits weight among the tied non-$\bos$ positions.
    \end{itemize}
    The attention output then equals $\mathbbm{1}\{\modpred_m^r(\idxi)\}$: The rotary pair $\idxd,\idxd+1$ supplies the position-dependent score against $\bos$, while the unrotated pair $\idxd',\idxd'+1$ supplies the constant non-$\bos$ baseline $\ropeconstant b_m$ that this score is compared against, and together they correctly compute $\modpred^r_m(\idxi)$.

    \underline{\textit{Completing the induction.}}
    A formula contains only finitely many modular atoms, so it requires finitely many rotary pairs and constants.
    Choose $\finiteset$ to contain the entries of their finite lookup tables and all displayed intermediate values. 
    This completes the induction.
\end{proof}

\subsection{Characterizing Periodic SiPE}\label{app:proofsmatsin}

\smatsinpf*
\begin{proof}
We now show both directions of the equivalence.

\textbf{Forward} ($\Rightarrow$).
Every non-PE component of a $\smat[\sinperiodic]$ transformer is in $\pfo$ \citep{li2025characterizingexpressivityfixedprecisiontransformer}.
Periodic sinusoidal position encodings can also be expressed via unary modular predicates by \cref{prop:periodicinputs}, yielding the inclusion in $\pfo[\modpred]$.

\textbf{Reverse} ($\Leftarrow$).
$\nope$ transformers, which are subsumed by $\smat[\sinperiodic]$ transformers, can express $\pfo$.
Simulating modular predicates can be done by processing sinusoidal position encodings with feedforward networks \citep{yang2024maskedhardattentiontransformersrecognize} using an appropriate schedule and pre-computed lookup table (\cref{subsec:componentperiodic}), leading to the reverse direction.
Note that \citeposs{yang2024maskedhardattentiontransformersrecognize} modular predicate simulation does not rely on the attention type. 
\end{proof}

\section{Proofs: Non-Periodic RoPE}\label{app:nonperiodic-proofs}

Because \cref{corr:yesterday} certifies the simulation only up to a finite length, we first record the length-bounded variant of \cref{def:tfsimulateslogic} that it uses.
\begin{definition}[Bounded-length simulation]\label{def:boundedsimulation}
    Let $L\in\N$. A formula $\tlf$ is \defn{$L$-simulated} by a function $\tffunc\colon\kleene{(\eosalphabet)}\to\kleene{(\finiteset^\hiddDim)}$ if there exists a coordinate $\idxd\in\NTo{\hiddDim}$ such that for every input $\bos\str\eos$ with $\left|\str\right|\leq L$ and every position $\idxi\in\NTo{\left|\str\right|+1}$,
    \begin{equation}
        \tffunc\left(\bos\str\eos\right)_{\idxd,\idxi}=
        \begin{cases}
            1 & \text{if }\str,\idxi\models\tlf\\
            0 & \text{otherwise.}
        \end{cases}
    \end{equation}
    Taking $L\to\infty$ recovers \cref{def:tfsimulateslogic}.
\end{definition}
Throughout, positions are indexed so that $\bos$ occupies position $0$, the string occupies positions $1,\ldots,\left|\str\right|$, and $\eos$ occupies position $\left|\str\right|+1$, which is the readout position. Hence $\left|\str\right|\leq L$ makes every query position at most $L+1$, which is exactly the range quantified over in \cref{eq:yesterday-bound}.

\yesterdaypf*
\begin{proof}\label{proof:yesterday}
    Allocate unrotated dimensions to the standard structural-induction simulation of $\ptl$ \citep{li2025characterizingexpressivityfixedprecisiontransformer}. Fix $k\geq1$. To simulate a $\yesterday^k\tlf$ subformula directly, use one rotary pair $\idxd,\idxd+1$ in a dedicated attention head, with query bias $C\left(\cos\left(kg\right),-\sin\left(kg\right)\right)$ and key bias $\left(1,0\right)$ at those two coordinates.
    In this head we set $\mQ$ and $\mK$ to the zero matrix and set every coordinate of $\biasVector_q$ and $\biasVector_k$ outside the pair $\idxd,\idxd+1$ to $0$. The score is therefore exactly the contribution of that pair: The unrotated dimensions carrying the $\ptl$ simulation, and every other rotary pair, contribute $0$ to it. In particular the biases are the head's only source of query and key content, so the score is position-dependent but content-independent, as the fixed-offset semantics of $\yesterday^k$ requires.

    After rotation at query position $\idxi$ and key position $\idxj$, exact arithmetic gives
    \begin{equation}
        \mS^{g,C,k}_{\idxi,\idxj}
        =C\cos\left(g\left(\idxi-k-\idxj\right)\right),
    \end{equation}
    which is \cref{eq:ideal-previous-score}. Under the model's actual arithmetic, denote the realized score by $\widehat{\mS}^{g,C,k,\finiteset}_{\idxi,\idxj}$ as in \cref{eq:yesterday-bound}. Store the truth value of $\tlf$ in the value vector at each string position and store $0$ at $\bos$.
    
    We prove (i). Fix $k\leq L\leq\boundyesterday[(k)]\left(g,C,\finiteset\right)$. 
    At every $k+1\leq\idxi\leq L+1$, position $\idxi-k$ beats every other key by at least $\floatlarge$, by the definition of $\boundyesterday[(k)]$; note that the competing keys include $\bos$ at $\idxj=0$, whose rotation matrix is the identity, so its key is $\left(\cos 0,\sin 0\right)=\left(1,0\right)$. 
    Stable softmax therefore assigns weight $1$ to $\idxi-k$ and $0$ elsewhere, so the head returns the truth value of $\tlf$ exactly $k$ positions earlier. 
    Since $\idxi\geq k+1$, the attended position $\idxi-k$ is a string position, not $\bos$. 
    At the boundary $\idxi\leq k$, the semantics require $\yesterday^k\tlf$ to be false. 
    The unrotated construction simulates the gate $\past^k\true$, which holds exactly when $\idxi>k$---indeed $\past\true$ holds at $\idxi$ iff $\idxi>1$, and inductively $\past^k\true=\past\left(\past^{k-1}\true\right)$ holds at $\idxi$ iff some $\idxj<\idxi$ has $\idxj>k-1$, i.e.\ iff $\idxi>k$; a feedforward Boolean conjunction with this gate therefore suppresses the arbitrary attention output at those boundary positions, mapping it to $0$ regardless of its value. 
    This $L$-simulates $\yesterday^k$, proving (i).

    For (ii), every $\ltlAcr[\past,\yesterday]$ formula is built from atomic predicates, Boolean connectives, $\past$, and $\yesterday$, so it needs only the offset $k=1$: one layer per $\yesterday$ occurrence, applying (i) with $k=1$, which is available because $L\leq\boundyesterday[(1)]\left(g,C,\finiteset\right)$ by hypothesis. Together with the unrotated $\past$ and Boolean constructions this completes the structural induction. For a finite offset set $\offsetset$, give each $k\in\offsetset$ its own head and rotary pair; the hypothesis $L\leq\min_{k\in\offsetset}\boundyesterday[(k)]\left(g,C,\finiteset\right)$ makes (i) available for every $k\in\offsetset$ simultaneously, and $\max\offsetset\leq L$ makes each application well posed.

    For (iii), fix any finite $L$ and finite collection of offsets, and choose $g/\left(2\pi\right)\notin\Q$. Then $\delta_L\left(g\right)>0$. Choose the scales $C$ and a finite-precision regime jointly so that $C\delta_L\left(g\right)\geq\floatlarge$ for every corresponding head and all parameters and intermediate sine, cosine, product, and score values used by the constructions through positions $L+1$ are represented exactly. This is possible in the abstract finite-value model because only finitely many offsets and nonzero values are required, while the regime can retain a separate underflow cell around zero. On that finite range the ideal identities are evaluated exactly, so \cref{eq:yesterday-bound} gives $\boundyesterday[(k)]\left(g,C,\finiteset\right)\geq L$ for every selected $k$.
\end{proof}

\section{Experiments}\label{app:experiments}
We supplement the periodic and conventional non-periodic experiments in \cref{sec:experiments-periodic,sec:experiments-nonperiodic} with additional details.

\subsection{Languages}\label{app:experiment-languages}
We first introduce the different languages we train our transformers on.
\begin{itemize}
    \item $\kleene{(\syma \symb)}$: the language of strings in which the substring $\syma \symb$ is repeated arbitrarily many times. It is star-free, inside $\ptl[\modpred]=\pfo[\modpred]$ and outside of $\pfo$. It is also the Dyck language (well-balanced parentheses) with 1 pair of parentheses, limited to depth 1.
    \item $\kleene{(\syma \syma)}$: the language of strings of even length.
    It is a canonical example of a non-star-free, $\ptl[\modpred]$ language.
    \item $\syma\kleene{\alphabet}$: the language of strings starting with a designated symbol $\syma$ over some alphabet $\alphabet$ such that $|\alphabet|>1$.
    It is the simplest language of the class $\pfo = \ptl$.
    \item $\kleene{\alphabet}\syma$: the language of strings ending with a designated symbol $\syma$ over some alphabet $\alphabet$ such that $|\alphabet|>1$.
    It belongs in $\ltlAcr[\future]$ and $\ltlAcr[\yesterday]$.
    \item $\kleene{\alphabet}\syma\symb\kleene{\alphabet}$: the language of strings containing the contiguous substring $\syma\symb$. It is locally testable and definable in $\ltlAcr[\past,\yesterday]$.
\end{itemize}
\subsection{Experimental Setup}\label{app:experimental-setup}
Our experimental setup follows \citet{delétang2023neuralnetworkschomskyhierarchy, li2025characterizingexpressivityfixedprecisiontransformer}.
Our transformers use soft attention and strict future masking, with 5 layers, model dimension set to 64, and 8 attention heads.
Training strings are of length up to 40, while test strings range from length 41 to 500.
The model is trained for
1,000,000 steps with a batch size of 128. For evaluation, we generate 512 samples per test length. Each experiment is run with 5 different random seeds and 3 learning rates ($1 \times 10^{-4}
,3 \times 10^{-4}
,5 \times 10^{-4}$).

\subsection{Periodic-Construction Settings}\label{app:periodic-settings}
The periodic experiments in \cref{tab:exp-periodic} compare the period-targeting $\ropeperiodic$ and $\sinperiodic$ variants with $\nope$; the conventional fully rotated controls appear in \cref{fig:exp-modular-base-sweep}. 
The component-periodic scheduler is explicitly written out in the code implementation.

\subsection{Conventional Non-Periodic Settings}\label{app:nonperiodic-settings}
Recall in $\rope$'s initial presentation that the base value is set to $10^4$.
Prior extrapolation experiments find that this standard choice can be the worst tested base, with improvements on both sides of $10^4$ \citep{liu2024scalinglawsropebasedextrapolation}. To test whether our conclusions depend on this choice, we therefore sweep bases $\{10^{-12},10^{-8},10^{-4},1,10^4,10^8,10^{12}\}$. \Cref{fig:exp-nonperiodic-base-sweep} reports every fully rotated $\ropenonperiodic$ configuration and compares it with $\nope$. \Cref{fig:exp-modular-base-sweep} reports fully rotated $\ropenonperiodic$ and $\sinnonperiodic$ configurations across the same sweep for $\kleene{(\syma\syma)}$ and $\kleene{(\syma\symb)}$.

\subsection{Detailed Results}\label{app:detailed-results}
\Cref{tab:detailed-modular-sweep,tab:detailed-locality-sweep} report the complete aggregate statistics for the tasks and fully rotated settings shown in \cref{fig:exp-modular-base-sweep,fig:exp-nonperiodic-base-sweep}. 
Each row aggregates 5 seeds and 3 learning rates. 

\begingroup
\small
\setlength{\tabcolsep}{4.5pt}
\begin{longtable}{@{}llcrrrrrr@{}}
\caption{Detailed results for the fully rotated conventional base sweep on the modular languages in \cref{fig:exp-modular-base-sweep}.}\label{tab:detailed-modular-sweep}\\
\toprule
Language & PE & Base & \multicolumn{3}{c}{Accuracy} & \multicolumn{3}{c}{$N^*$} \\
\cmidrule(lr){4-6}\cmidrule(lr){7-9}
& & & Max & Mean & Std. & Max & Mean & Std. \\
\midrule
\endfirsthead
\multicolumn{9}{c}{\tablename\ \thetable{} continued}\\
\toprule
Language & PE & Base & \multicolumn{3}{c}{Accuracy} & \multicolumn{3}{c}{$N^*$} \\
\cmidrule(lr){4-6}\cmidrule(lr){7-9}
& & & Max & Mean & Std. & Max & Mean & Std. \\
\midrule
\endhead
\midrule
\multicolumn{9}{r}{Continued on next page}\\
\endfoot
\bottomrule
\endlastfoot
$\kleene{(\syma\syma)}$ & $\ropenonperiodic$ & $10^{-12}$ & 0.522 & 0.489 & 0.0204 & 53.0 & 41.3 & 11.6 \\
$\kleene{(\syma\syma)}$ & $\ropenonperiodic$ & $10^{-8}$ & 0.500 & 0.483 & 0.0069 & 65.0 & 50.1 & 14.9 \\
$\kleene{(\syma\syma)}$ & $\ropenonperiodic$ & $10^{-4}$ & 0.513 & 0.501 & 0.0061 & 57.0 & 49.5 & 3.1 \\
$\kleene{(\syma\syma)}$ & $\ropenonperiodic$ & $1$ & 0.509 & 0.494 & 0.0076 & 47.0 & 34.3 & 17.2 \\
$\kleene{(\syma\syma)}$ & $\ropenonperiodic$ & $10^{4}$ & 0.507 & 0.489 & 0.0117 & 48.0 & 26.3 & 21.5 \\
$\kleene{(\syma\syma)}$ & $\ropenonperiodic$ & $10^{8}$ & 0.515 & 0.490 & 0.0146 & 46.0 & 22.7 & 21.3 \\
$\kleene{(\syma\syma)}$ & $\ropenonperiodic$ & $10^{12}$ & 0.517 & 0.494 & 0.0095 & 49.0 & 23.3 & 21.9 \\
\addlinespace
$\kleene{(\syma\syma)}$ & $\sinnonperiodic$ & $10^{-12}$ & 0.520 & 0.504 & 0.0117 & 41.0 & 16.4 & 20.1 \\
$\kleene{(\syma\syma)}$ & $\sinnonperiodic$ & $10^{-8}$ & 0.654 & 0.619 & 0.0218 & 88.0 & 71.3 & 11.9 \\
$\kleene{(\syma\syma)}$ & $\sinnonperiodic$ & $10^{-4}$ & \textbf{0.657} & \textbf{0.623} & 0.0273 & \textbf{90.0} & \textbf{75.7} & 6.3 \\
$\kleene{(\syma\syma)}$ & $\sinnonperiodic$ & $1$ & 0.509 & 0.496 & 0.0089 & 43.0 & 8.5 & 16.9 \\
$\kleene{(\syma\syma)}$ & $\sinnonperiodic$ & $10^{4}$ & 0.522 & 0.503 & 0.0078 & 41.0 & 2.7 & 10.2 \\
$\kleene{(\syma\syma)}$ & $\sinnonperiodic$ & $10^{8}$ & 0.517 & 0.500 & 0.0049 & 0.0 & 0.0 & 0.0 \\
$\kleene{(\syma\syma)}$ & $\sinnonperiodic$ & $10^{12}$ & 0.515 & 0.497 & 0.0107 & 0.0 & 0.0 & 0.0 \\
\midrule
$\kleene{(\syma\symb)}$ & $\ropenonperiodic$ & $10^{-12}$ & 0.572 & 0.530 & 0.0168 & \textbf{98.0} & 65.1 & 13.0 \\
$\kleene{(\syma\symb)}$ & $\ropenonperiodic$ & $10^{-8}$ & 0.520 & 0.510 & 0.0057 & 58.0 & 48.9 & 5.2 \\
$\kleene{(\syma\symb)}$ & $\ropenonperiodic$ & $10^{-4}$ & 0.530 & 0.514 & 0.0060 & 68.0 & 52.8 & 5.6 \\
$\kleene{(\syma\symb)}$ & $\ropenonperiodic$ & $1$ & 0.522 & 0.509 & 0.0041 & 56.0 & 48.1 & 3.1 \\
$\kleene{(\syma\symb)}$ & $\ropenonperiodic$ & $10^{4}$ & 0.740 & 0.560 & 0.0602 & 96.0 & \textbf{66.0} & 15.7 \\
$\kleene{(\syma\symb)}$ & $\ropenonperiodic$ & $10^{8}$ & \textbf{0.817} & \textbf{0.566} & 0.0896 & \textbf{98.0} & 57.1 & 14.7 \\
$\kleene{(\syma\symb)}$ & $\ropenonperiodic$ & $10^{12}$ & 0.703 & 0.524 & 0.0484 & 72.0 & 50.1 & 6.5 \\
\addlinespace
$\kleene{(\syma\symb)}$ & $\sinnonperiodic$ & $10^{-12}$ & 0.515 & 0.504 & 0.0050 & 42.0 & 16.8 & 20.6 \\
$\kleene{(\syma\symb)}$ & $\sinnonperiodic$ & $10^{-8}$ & 0.648 & 0.536 & 0.0416 & 54.0 & 41.9 & 16.7 \\
$\kleene{(\syma\symb)}$ & $\sinnonperiodic$ & $10^{-4}$ & 0.547 & 0.523 & 0.0145 & 70.0 & 52.9 & 7.6 \\
$\kleene{(\syma\symb)}$ & $\sinnonperiodic$ & $1$ & 0.515 & 0.505 & 0.0045 & 48.0 & 23.1 & 21.6 \\
$\kleene{(\syma\symb)}$ & $\sinnonperiodic$ & $10^{4}$ & 0.515 & 0.503 & 0.0044 & 46.0 & 17.1 & 20.9 \\
$\kleene{(\syma\symb)}$ & $\sinnonperiodic$ & $10^{8}$ & 0.559 & 0.514 & 0.0142 & 52.0 & 26.3 & 21.6 \\
$\kleene{(\syma\symb)}$ & $\sinnonperiodic$ & $10^{12}$ & 0.586 & 0.514 & 0.0204 & 52.0 & 24.1 & 22.8 \\
\end{longtable}

\begin{longtable}{@{}llcrrrrrr@{}}
\caption{Detailed results for $\nope$ and the fully rotated conventional $\rope$ base sweep in \cref{fig:exp-nonperiodic-base-sweep}.}\label{tab:detailed-locality-sweep}\\
\toprule
Language & PE & Base & \multicolumn{3}{c}{Accuracy} & \multicolumn{3}{c}{$N^*$} \\
\cmidrule(lr){4-6}\cmidrule(lr){7-9}
& & & Max & Mean & Std. & Max & Mean & Std. \\
\midrule
\endfirsthead
\multicolumn{9}{c}{\tablename\ \thetable{} continued}\\
\toprule
Language & PE & Base & \multicolumn{3}{c}{Accuracy} & \multicolumn{3}{c}{$N^*$} \\
\cmidrule(lr){4-6}\cmidrule(lr){7-9}
& & & Max & Mean & Std. & Max & Mean & Std. \\
\midrule
\endhead
\midrule
\multicolumn{9}{r}{Continued on next page}\\
\endfoot
\bottomrule
\endlastfoot
$\kleene{\alphabet}\syma$ & $\nope$ & -- & 0.604 & 0.564 & 0.0237 & 60.0 & 50.8 & 3.9 \\
$\kleene{\alphabet}\syma$ & $\ropenonperiodic$ & $10^{-12}$ & 0.741 & 0.709 & 0.0181 & 149.0 & 95.0 & 26.6 \\
$\kleene{\alphabet}\syma$ & $\ropenonperiodic$ & $10^{-8}$ & \textbf{0.927} & \textbf{0.829} & 0.0898 & 116.0 & 81.3 & 16.8 \\
$\kleene{\alphabet}\syma$ & $\ropenonperiodic$ & $10^{-4}$ & 0.913 & 0.785 & 0.0538 & \textbf{224.0} & \textbf{99.8} & 46.8 \\
$\kleene{\alphabet}\syma$ & $\ropenonperiodic$ & $1$ & 0.753 & 0.649 & 0.0483 & 82.0 & 54.8 & 9.2 \\
$\kleene{\alphabet}\syma$ & $\ropenonperiodic$ & $10^{4}$ & 0.555 & 0.524 & 0.0152 & 51.0 & 45.5 & 2.3 \\
$\kleene{\alphabet}\syma$ & $\ropenonperiodic$ & $10^{8}$ & 0.890 & 0.733 & 0.1003 & 112.0 & 65.7 & 14.6 \\
$\kleene{\alphabet}\syma$ & $\ropenonperiodic$ & $10^{12}$ & 0.685 & 0.640 & 0.0305 & 70.0 & 58.9 & 4.8 \\
\midrule
$\syma\kleene{\alphabet}$ & $\nope$ & -- & \textbf{1.000} & \textbf{0.989} & 0.0413 & \textbf{500.0} & \textbf{470.5} & 82.0 \\
$\syma\kleene{\alphabet}$ & $\ropenonperiodic$ & $10^{-12}$ & \textbf{1.000} & 0.977 & 0.0253 & \textbf{500.0} & 199.5 & 99.1 \\
$\syma\kleene{\alphabet}$ & $\ropenonperiodic$ & $10^{-8}$ & \textbf{1.000} & 0.927 & 0.1033 & \textbf{500.0} & 166.7 & 101.4 \\
$\syma\kleene{\alphabet}$ & $\ropenonperiodic$ & $10^{-4}$ & \textbf{1.000} & 0.961 & 0.0515 & \textbf{500.0} & 199.1 & 95.4 \\
$\syma\kleene{\alphabet}$ & $\ropenonperiodic$ & $1$ & \textbf{1.000} & 0.953 & 0.1005 & \textbf{500.0} & 204.9 & 109.5 \\
$\syma\kleene{\alphabet}$ & $\ropenonperiodic$ & $10^{4}$ & \textbf{1.000} & 0.847 & 0.1449 & 184.0 & 146.9 & 30.2 \\
$\syma\kleene{\alphabet}$ & $\ropenonperiodic$ & $10^{8}$ & \textbf{1.000} & 0.928 & 0.0864 & \textbf{500.0} & 293.3 & 121.4 \\
$\syma\kleene{\alphabet}$ & $\ropenonperiodic$ & $10^{12}$ & \textbf{1.000} & 0.988 & 0.0254 & \textbf{500.0} & 402.9 & 125.1 \\
\midrule
$\kleene{\alphabet}\syma\symb\kleene{\alphabet}$ & $\nope$ & -- & \textbf{0.874} & \textbf{0.672} & 0.0829 & \textbf{200.0} & \textbf{92.9} & 32.6 \\
$\kleene{\alphabet}\syma\symb\kleene{\alphabet}$ & $\ropenonperiodic$ & $10^{-12}$ & 0.770 & 0.625 & 0.0697 & 103.0 & 61.9 & 16.8 \\
$\kleene{\alphabet}\syma\symb\kleene{\alphabet}$ & $\ropenonperiodic$ & $10^{-8}$ & 0.753 & 0.604 & 0.0812 & 76.0 & 58.7 & 9.8 \\
$\kleene{\alphabet}\syma\symb\kleene{\alphabet}$ & $\ropenonperiodic$ & $10^{-4}$ & 0.740 & 0.606 & 0.0556 & 73.0 & 55.2 & 7.4 \\
$\kleene{\alphabet}\syma\symb\kleene{\alphabet}$ & $\ropenonperiodic$ & $1$ & 0.647 & 0.624 & 0.0221 & 51.0 & 46.6 & 1.5 \\
$\kleene{\alphabet}\syma\symb\kleene{\alphabet}$ & $\ropenonperiodic$ & $10^{4}$ & 0.612 & 0.572 & 0.0250 & 53.0 & 49.8 & 3.5 \\
$\kleene{\alphabet}\syma\symb\kleene{\alphabet}$ & $\ropenonperiodic$ & $10^{8}$ & 0.624 & 0.594 & 0.0183 & 53.0 & 47.2 & 2.6 \\
$\kleene{\alphabet}\syma\symb\kleene{\alphabet}$ & $\ropenonperiodic$ & $10^{12}$ & 0.658 & 0.637 & 0.0102 & 50.0 & 46.1 & 2.2 \\
\end{longtable}
\endgroup

\end{document}